\documentclass[twoside,11pt]{article}
\usepackage{fix-cm}
\usepackage{microtype}

\usepackage{booktabs}
\usepackage{xcolor}
\usepackage{colortbl}
\usepackage[normalem]{ulem} 
\usepackage{longtable}
\usepackage{makecell}
\usepackage{nicefrac}
\usepackage{amsmath,amssymb,mathtools}
\usepackage{amsthm}
\usepackage{array} 
\usepackage{algorithm}
\usepackage{algorithmic} 
\usepackage{graphicx}  
\usepackage{tabularx}
\usepackage{multirow}
\usepackage{enumitem}
\usepackage{siunitx}
\usepackage{subcaption} 
\usepackage[T1]{fontenc}
\usepackage{lmodern}

\usepackage{tikz}
\usetikzlibrary{arrows.meta,positioning,calc}
\usepackage{pgfplots}
\pgfplotsset{compat=1.18}

\usepackage[preprint]{jmlr2e}
\usepackage{etoc}
\etocdepthtag.toc{mtchapter}
\etocsettagdepth{mtchapter}{subsection}
\etocsettagdepth{mtappendix}{none}
\usepackage[textsize=tiny]{todonotes}

\newcommand{\keybox}[1]{%
  \begin{center}
  \setlength{\fboxsep}{6pt}%
  \setlength{\fboxrule}{0.4pt}%
  \fbox{\parbox{0.95\linewidth}{\small\itshape #1}}
  \end{center}
}

\newtheorem{assumption}[theorem]{Assumption}
\newtheorem{corollary*}{Corollary*}

\usepackage{lastpage}
\jmlrheading{1}{2026}{1-\pageref{LastPage}}{5/26}{---}{25-0000}{Yi Xie, Zhanke Zhou, Chentao Cao, Bo Liu, Bo Han}
\ShortHeadings{DICE: Entropy-Regularized Equilibrium Selection}{Yi Xie, Zhanke Zhou, Chentao Cao, Bo Liu, Bo Han}
\firstpageno{1}

\begin{document}

\title{DICE: Entropy-Regularized Equilibrium Selection for Stable Multi-Agent LLM Coordination}

\author{\name Yi Xie \thanks{These authors contributed equally to this work.}
       \email yix@arizona.edu \\
       \addr Department of Electrical \& Computer Engineering\\
       University of Arizona\\
       Tucson, AZ 85721-0104, USA
       \AND
       \name Zhanke Zhou $^*$
       \email cszkzhou@comp.hkbu.edu.hk \\
       \addr Department of Computer Science \\
       Hong Kong Baptist University\\
       Kowloon Tong, Hong Kong SAR
        \AND
       \name Chentao Cao \email csctcao@comp.hkbu.edu.hk \\
       \addr Department of Computer Science \\
       Hong Kong Baptist University\\
       Kowloon Tong, Hong Kong SAR
        \AND
       \name Bo Liu \email boliu@arizona.edu \\
       \addr Department of Electrical \& Computer Engineering\\
       University of Arizona\\
       Tucson, AZ 85721-0104, USA
        \AND
       \name Bo Han \email bhanml@comp.hkbu.edu.hk \\
       \addr Department of Computer Science \\
       Hong Kong Baptist University\\
       Kowloon Tong, Hong Kong SAR
        }

\editor{My editor}

\maketitle

\vspace{-10pt}
\begin{abstract}
Multi-agent large language model (LLM) systems often fail to reliably outperform a single strong model equipped with best-of-$N$ sampling. We argue that a core source of this instability is \emph{ill-posed equilibrium selection}: current systems specify what information agents share, but not which coordination convention should be selected. We formalize a broad class of such systems as discounted incomplete-information Markov games and show that two common pathologies, oscillation between competing conventions and drift across them, can both induce unstable learning and linear Bayesian regret. To obtain a well-posed target, we introduce the Heterogeneous Quantal Response Equilibrium (HQRE), an entropy-regularized equilibrium concept with agent- and state-dependent temperatures. Under a monotonicity condition, HQRE is unique, admits linearly convergent mirror updates, and yields bounded Bayesian regret; the same condition yields rollout-measurable stability diagnostics. We instantiate this objective in two algorithms: DICE-PC, which coordinates frozen models through prompt-control actions, and DICE-FT, which performs parameter-efficient mirror fine-tuning. Across eleven benchmarks in four domains, DICE improves accuracy-cost trade-offs over strong within-class baselines; on reasoning and planning tasks, DICE-PC improves by 4.3 percentage points on average and DICE-FT by 8.5 points.
\end{abstract}


\section{Introduction}
\label{sec:intro}

\begin{figure}[t!]
\centering
\includegraphics[width=\textwidth]{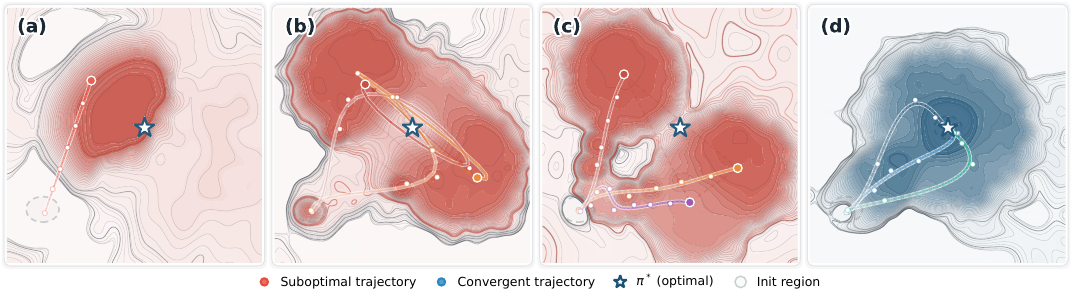}
\vspace{-14pt}
\caption{
\textbf{Loss landscapes under different coordination paradigms.}
More agents create more possible ways to coordinate.
Without an explicit selection rule, the system can saturate at a
single-agent ceiling (a), oscillate between competing conventions (b),
or drift across conventions during learning (c).
DICE turns the implicit choice of convention into an explicit one,
collapsing the landscape to a single basin (d).
Shading encodes loss (darker $=$ lower); the dashed ellipse marks the
shared initialization region and the star marks the reference policy
$\pi^\star$.
Landscapes are schematic; the underlying phenomena are formalized in
Section~\ref{sec:dice_framework}.
}
\label{fig:landscape}
\vspace{-18pt}
\end{figure}

Multi-agent LLM systems promise greater capability by distributing work across specialized roles such as planning, solving, critiquing, and verification~\citep{liu2024groupdebate,hong2024metagpt}. While adding more agents expands parallel exploration, it also increases system complexity, particularly in multi-agent coordination. In practice, small changes in prompts, sampling, or context truncation can push the same system into completely different joint behaviors. One agent may follow a divide-and-conquer protocol, while another acts under a debate-and-merge protocol. A coordination convention is a consistent method of coordination in which no agent benefits from unilaterally deviating. Each convention is reasonable on its own, but mixing them results in incoherent behavior and severe system instability.

This instability helps explain a recurring empirical finding:
multi-agent designs often fail to reliably outperform a single strong
model equipped with best-of-$N$ sampling and lightweight
verification~\citep{kim2025scalingagents,cemri2025fail}.
The issue is not only what information agents exchange, but whether
the system specifies a stable rule for choosing among multiple
self-consistent ways of coordinating\citep{liu2024statistical, shi2025fundamental}.
Current architectures usually address the first question through
shared transcripts, judges, summarizers, or routing heuristics.
They rarely address the second. 

Our central claim is that \emph{ill-posed equilibrium selection} is a
core source of instability in multi-agent LLM systems.
Existing methods typically address what information becomes public
among agents, but they do not address which coordination convention
should be selected, or why that convention should remain stable under
perturbations.
We therefore separate two concerns that are usually conflated:
\emph{information structure}, which governs what agents can observe,
and \emph{coordination mechanism}, which governs how a particular
convention is selected and maintained.
To study both within a single framework, we model a broad class of
multi-agent architectures as a discounted incomplete-information
Markov game (IIMG) with a shared public stream and private agent
histories.
Figure~\ref{fig:landscape} illustrates how this distinction plays out
across common coordination paradigms.

A single strong model with best-of-$N$ sampling (panel~a) induces an
effectively unimodal landscape: selection is stable, but performance
saturates at a ceiling set by the capacity of one agent.
Adding multiple agents through message passing or debate (panel~b)
raises this ceiling by pooling capacity, yet the resulting landscape
typically contains competing basins, so learning oscillates between
conventions rather than committing to either.
Letting multiple agents learn together without an explicit selection
rule (panel~c) introduces a further failure mode: the realized
convention can drift across training, so the system never settles.
Both oscillation and drift yield linear Bayesian regret.

These failures share a common root cause.
The underlying game admits several self-consistent joint strategies,
and nothing in the system selects which one should be realized.
HQRE regularization (panel~d) resolves this by implementing an
explicit equilibrium-selection mechanism, yielding a single dominant
basin and restoring convergence toward a selected convention.
Panels (a) through (c) thus illustrate three increasingly severe
coordination failures, while panel (d) shows that a single
regularization mechanism, which we introduce next, resolves all three
by providing a unique equilibrium target.

We call this mechanism the \emph{Heterogeneous Quantal Response
Equilibrium} (HQRE): an entropy-regularized equilibrium concept with
agent- and state-dependent temperatures.
The resulting method, DICE (\textbf{D}iscounted
\textbf{I}ncomplete-information \textbf{C}oordination via
\textbf{E}ntropy), performs entropy-regularized equilibrium selection
within the IIMG framework.
Unlike a generic exploration bonus, the regularization here plays a
structural role: it makes the coordination game well-posed, so that
equilibrium selection admits a unique solution rather than a set of
competing ones.

Under a monotonicity condition in which regularization curvature
dominates cross-agent coupling, HQRE is unique and explicit mirror
updates converge linearly, yielding Bayesian regret that is uniformly
bounded in the horizon; we defer the precise statements and the
supporting smoothness assumptions to Section~\ref{sec:hqre_theory}.
Because the monotonicity condition can be conservative at LLM scale,
we treat it primarily as a design principle.
It motivates the regularization structure and yields
rollout-measurable stability diagnostics that we track throughout our
evaluation.

We instantiate the HQRE objective in two complementary deployment
regimes.
DICE-PC learns distributions over bounded prompt-control actions to
coordinate frozen execution models, preserving pretrained
generalization while stabilizing equilibrium selection.
DICE-FT performs parameter-efficient mirror fine-tuning under
differentiated token-type regularization, providing additional gains
on coordination-heavy tasks at the cost of greater optimization
difficulty.
We evaluate both on eleven benchmarks spanning reasoning, planning,
active information gathering, and multi-agent coordination, with
stability diagnostics reported alongside accuracy and token cost.

The main contributions of this paper are summarized as follows:
\begin{itemize}[leftmargin=*]
\item \textbf{Formalization.}
  We formalize multi-agent LLM coordination as a discounted IIMG, explicitly
  separating information structure from coordination mechanism, and
  identify well-posed equilibrium selection as a critical but previously
  unaddressed design dimension.
\item \textbf{Failure-mode analysis.}
  Within the IIMG framework, we prove that debate-style defensive mixing
  and equilibrium multiplicity can each produce linear Bayesian regret,
  identifying ill-posed equilibrium selection as the common root cause
  and establishing that richer communication alone is neither necessary
  nor sufficient for stable coordination.
\item \textbf{Regularized equilibrium theory.}
  We introduce HQRE as an entropy-regularized equilibrium concept with
  the following guarantees under a monotonicity condition:
  (i)~provable uniqueness of the equilibrium,
  (ii)~linear convergence of explicit KL-mirror updates under standard
  field-Lipschitz and local-quadratic assumptions,
  (iii)~uniformly bounded Bayesian regret, and
  (iv)~rollout-measurable stability diagnostics derived from the theory.
\item \textbf{Algorithms and evaluation.}
  We instantiate the HQRE objective as DICE-PC and DICE-FT.
  Validation on eleven benchmarks across four domains demonstrates
  improved accuracy--cost trade-offs and reduced instability relative to
  single-model, debate-style, and ensembling baselines under matched
  deployment constraints.
 DICE-PC achieves +4.3~pp average improvement over prompt-control baselines and DICE-FT achieves +8.5~pp over fine-tuning baselines across six reasoning and planning benchmarks (Section~\ref{sec:experiments}).

\end{itemize}

Beyond the specific algorithms, this work introduces three
methodological contributions to multi-agent system that are reusable
beyond the LLM coordination setting.
First, the IIMG formalization provides a reusable modeling tool: any
multi-agent system with private information and a shared public stream
can be embedded in this framework (Lemma~\ref{lem:protocol_embedding}),
enabling diagnosis of coordination failures through the coupling
constant $L_c$ and the monotonicity margin.
Second, the inexact mirror recurrence
(Corollary*~\ref{cor:approx_convergence}) is a general-purpose bridge
theorem for connecting exact variational-inequality convergence to
practical deep-learning optimization with clipping, sampling, and
function approximation---applicable whenever mirror descent is used
with implementable surrogates.
Third, the four-component error decomposition
(Proposition~\ref{prop:delta_bound}) establishes a diagnostic
methodology: by isolating clipping, sampling, representation, and
optimization errors, it provides a template for error attribution in
any approximate policy-optimization pipeline.

\textbf{Paper organization.}
The remainder of the paper is organized as follows.
Section~\ref{sec:relatedwork} surveys related work.
Section~\ref{sec:dice_framework} formalizes the IIMG framework and identifies failure modes of unregularized coordination.
Section~\ref{sec:hqre_theory} introduces HQRE and establishes uniqueness, convergence, and regret guarantees.
Section~\ref{sec:algorithm} presents the DICE-PC and DICE-FT algorithms.
Section~\ref{sec:experiments} provides experimental evaluation, and Section~\ref{sec:conclusion} concludes.

\textbf{Relation to the conference version.}
This manuscript is a substantially extended version of our previous work
published at \textit{ICML 2025} \citep{yi2025debate}. The conference version
introduced ECON and used Bayesian Nash equilibrium (BNE) as a principled
target for belief-driven multi-agent LLM coordination. The present manuscript
identifies a limitation of that unregularized target: BNE existence alone
does not make coordination well posed when multiple self-consistent
equilibria coexist. 
The main extensions are summarized in Section~\ref{sec:further_discussion}, including (1) a discounted
incomplete-information Markov game formulation with explicit protocol
embedding and failure-mode analysis of ill-posed BNE selection, (2) a new
entropy-regularized HQRE theory with uniqueness, KL-mirror convergence,
zero-temperature connection to BNE, and bounded-regret guarantees, and
(3) two deployable algorithms, DICE-PC and DICE-FT, supported by an
inexact mirror recurrence, a field-error decomposition, and broader
experimental validation across 11 benchmarks in 4 domains.

\section{Related Work}\label{sec:relatedwork}

\paragraph{Single-model reasoning.}
Single-model reasoning methods improve exploration within a single policy but
cannot address instability arising from coexisting coordination conventions.
Chain-of-thought prompting~\citep{wei2022chain} and zero-shot reasoning
instructions~\citep{kojima2022large} elicit intermediate computations, while
self-consistency~\citep{wang2023self} and tree-based exploration
methods~\citep{yao2024tree} improve robustness by aggregating or searching over
multiple candidate traces.
A complementary line of work automates or schedules prompts to match the task
structure, including least-to-most prompting~\citep{zhou2022least},
complexity-aware prompting~\citep{fu2022complexity}, automatic chain-of-thought
discovery~\citep{zhang2023automatic}, contrastive prompting~\citep{chia2023contrastive}, inception prompting~\citep{li2023deepinception}, and tool-augmented prompting~\citep{zhou2025alphaapollo, lu2023dynamic}.
These approaches are limited to improving exploration or verification within a single model.

\paragraph{Multi-agent coordination protocols.}
Multi-agent LLM coordination protocols increase communication bandwidth but
typically lack a formal equilibrium-selection target that guarantees stability.
Debate and message-passing frameworks~\citep{du2023debate,li2023camel} aim to
improve solutions via critique and refinement.
System-oriented agent frameworks coordinate multiple LLMs with tools, memory,
and communication, such as MetaGPT~\citep{hong2024metagpt},
AutoGen~\citep{wu2024autogen}, and Lumos\citep{yin2024agent}.
Other protocols facilitate information exchange or consensus building across
agents~\citep{yin2023exchange,lan2024stance,yuan2024evoagent, xie2025acorn, xie2025heuristics}.
Recent evaluations, however, report that multi-agent systems do not reliably
outperform strong single-model baselines and can be sensitive to prompts and
sampling~\citep{kim2025scalingagents,cemri2025fail}.
Our work is aligned with these findings and focuses on a specific missing design
dimension: making equilibrium selection well-posed rather than increasing
communication volume~\citep{xie2026sat, xie2026teamtr}.
The scaling analysis of \citet{kim2025scalingagents} is particularly relevant: their finding that multi-agent systems often fail to outperform single-model baselines at larger scales aligns with our coordination-vs-capacity regime analysis (Section~\ref{subsec:ablations_sensitivity}), and our IIMG framework provides a formal explanation through the coupling constant $L_c$.
Similarly, the failure patterns cataloged by \citet{cemri2025fail}, including inconsistent outputs under prompt perturbation, sensitivity to agent ordering, and degradation under scaling, correspond to the equilibrium multiplicity and defensive mixing failures formalized in Sections~\ref{subsec:multiplicity}-\ref{subsec:mad}.

\paragraph{Aggregation and adjudication.}
Explicit aggregation and adjudication mechanisms specify how outputs are combined but do not provide convergence or uniqueness guarantees for the resulting joint policy.
For example, ChatEval uses referee agents for structured evaluations~\citep{chan2023chateval},
and RECONCILE aggregates a round-table discussion with confidence-weighted
consensus~\citep{chen2024reconcile}.
These systems specify \emph{what information is shared} and how outputs are
combined, but typically do not provide a formal selection target that guarantees
stability under perturbations.
In contrast, we formulate a unified incomplete-information Markov game
view and introduce an entropy-regularized equilibrium concept.
This concept turns equilibrium selection into a well-posed fixed-point
problem with uniqueness and convergence guarantees under verifiable
coupling conditions.

\paragraph{Multi-agent RL formalisms.}
Standard multi-agent RL formalisms typically address scalability and credit assignment but not equilibrium selection.
From a reinforcement learning perspective, decentralized partially observable
models such as DEC-POMDPs~\citep{bernstein2002complexity} and Markov games under
centralized training with decentralized execution (CTDE) are standard
formalisms for cooperative multi-agent control~\citep{lowe2017multi,lanctot2017unified,zhang2021multi, xiao2023insample}.
Value decomposition and monotone mixing networks~\citep{sunehag2017value,rashid2020monotonic,wang2020qplex, graves2023off}
provide scalable CTDE recipes by aligning individual action-values with a global
objective. Our algorithms use monotone critics and mixing consistent with these CTDE
principles, while our main contribution targets a different bottleneck that is
salient in open-ended LLM coordination: equilibrium multiplicity and selection
instability.

\paragraph{Regularized game theory.}
Our equilibrium concept builds on regularized game theory and first-order
methods for monotone variational inequalities.
Quantal Response Equilibrium (QRE) models bounded rationality via logit
responses~\citep{mckelvey1995quantal,mckelvey1998quantal}, and classical
conditions for existence and uniqueness under (strong) monotonicity trace back
to early game-theoretic analyses~\citep{rosen1965existence}.
Mirror-descent-style methods for monotone problems yield convergence guarantees
in appropriate Bregman geometries~\citep{nemirovski2009robust}, and entropy
regularization connects policy optimization to maximum-entropy control and
inference views~\citep{sutton2018reinforcement,levine2018reinforcement}.
We extend these principles to a heterogeneous entropy-regularized equilibrium
for discounted incomplete-information Markov games with shared public streams,
and leverage the resulting uniqueness and KL-mirror convergence to obtain
stability diagnostics and bounded Bayesian regret guarantees.

\paragraph{Distinction from single-agent entropy regularization.}
Before proceeding to the formal framework, we clarify a potential
source of confusion.
A natural question is how HQRE relates to single-agent maximum-entropy
methods, which also add entropy to the objective.
Single-agent maximum-entropy RL (e.g., SAC~\citep{haarnoja2018soft})
uses entropy regularization primarily for exploration and robustness within a
single decision-maker.
In contrast, the entropy term in HQRE serves an equilibrium-selection role:
it converts set-valued multi-agent best responses into a unique fixed point,
a function that has no analogue in the single-agent setting where the optimal policy is generically unique.

\begin{table}[t!]
\centering
\small
\caption{Qualitative comparison of coordination approaches.}
\label{tab:related_comparison}
\setlength{\tabcolsep}{5pt}
\begin{tabular}{@{} l c c c @{}}
\toprule
\textbf{Approach} & \textbf{Equilibrium guarantee} & \textbf{Convergence} & \textbf{Communication} \\
\midrule
Debate / MAD         & None             & Oscillation risk   & Full transcript    \\
Ensemble / SC        & None (aggregation only) & N/A (no learning)  & Output voting \\
CTDE (QMIX, etc.)   & Nash (set-valued)& Empirical          & Latent mixing  \\
\textbf{DICE (ours)} & \textbf{Unique HQRE}    & \textbf{Linear}    & Public stream  \\
\bottomrule
\end{tabular}
\end{table}

\section{Coordination Framework and Bayesian Regret under Incomplete Information}
\label{sec:dice_framework}

We treat multi-agent LLM coordination as an equilibrium phenomenon in an
incomplete-information game.
We use ``equilibrium'' in formal and theoretical passages
(Sections~\ref{sec:dice_framework}--\ref{sec:algorithm}) and
``coordination convention'' in informal and motivational passages
(introduction, experiment discussion, conclusion); both refer to a
self-consistent joint policy from which no agent benefits by
unilaterally deviating.
The public stream~$y^t$ captures all publicly committed
coordination-relevant information, ranging from rich shared chat and
judge-feedback protocols to communication-light designs.
This viewpoint separates the \emph{information structure} of a protocol
from the \emph{mechanism that yields stability}: explicit information
exchange may be present or absent; stable coordination requires a
well-posed equilibrium selection mechanism.

\begin{table}[t!]
\centering
\caption{Summary of results in Section~\ref{sec:dice_framework}.}
\label{tab:dice_sec2_results}
\small
\begin{tabular}{p{4.1cm}p{6.0cm}p{3.3cm}}
\toprule
\textbf{Result} & \textbf{Guarantee} & \textbf{Representative condition} \\
\midrule
IIMG formulation (\S\ref{subsec:notation})
& Models coordination as private histories + shared public stream
& (A1)--(A5) \\[3pt]

Explicit-communication protocols (\S\ref{subsec:notation})
& Covers shared-chat, blackboard, coordinator-mediated, and no-exchange
  regimes
& Public commitment of coordination-relevant information \\[3pt]

BNE existence (\S\ref{subsec:bne})
& Existence of a behavioral BNE in the discounted IIMG
& (A1)--(A4) \\[3pt]

Rate-dependent Bayesian regret (\S\ref{subsec:bne})
& If $\mathbb{E}[\bar\Delta_k]\le C_\Delta\, k^{-\alpha}$, then
$\mathrm{Regret}(T)\le \frac{2 N R_{\max} C_\Delta}{(1-\gamma)^2}
\sum_{k=1}^T k^{-\alpha}$
& Performance difference + TV bound \\[3pt]

Multiplicity $\Rightarrow$ linear regret (\S\ref{subsec:multiplicity})
& Unstable selection can yield
  $\mathrm{Regret}(T)=\Omega(T)$
& Two equilibria + persistent mis-selection \\[3pt]

Debate-style dynamics $\Rightarrow$ linear regret (\S\ref{subsec:mad})
& Persistent defensive mixing implies
  $\mathrm{Regret}_{\mathrm{MAD}}(T)=\Omega(T)$
& Defensive mixing + nondegenerate gap \\

\bottomrule
\end{tabular}
\end{table}

\subsection{IIMG model and explicit-communication protocols}
\label{subsec:notation}

There are $N$ agents indexed by
$\mathcal{N}=\{1,\dots,N\}$.
Before the formal definition, the key modeling idea is as follows:
agents share a public stream $y^t$ (e.g., shared chat or judge
feedback) and maintain private histories $h_i^t$ (e.g., scratchpads
or retrieved documents).
A discounted IIMG is specified by the following components (we explain
each in turn):
\[
\big\langle \mathcal N,\ \mathcal S,\ \mathcal A,\ \mathcal O,\
\mathcal Y,\ \mathcal P,\ \Omega,\ G,\ (\phi_i)_{i\in\mathcal N},\
\gamma,\ \mu_0 \big\rangle.
\]
Here, $\mathcal S$, $\mathcal A$, $\gamma$, and $\mu_0$ are
standard Markov-game components familiar from stochastic games and
Dec-POMDPs\citep{amato2019modeling, konidaris2009skill, konidaris2016constructing}.
The key additions for incomplete-information coordination are the
public-stream space $\mathcal Y$, the public-stream map $G$, the
private-observation kernel $\Omega$, and the reward decomposition
$\phi_i$.
In the table below, standard components are listed first, followed by
the IIMG-specific ones.

\noindent
\begin{tabular*}{\textwidth}{@{\extracolsep{\fill}}llll@{}}
\toprule
\multicolumn{4}{@{}l}{\textit{Standard Markov-game components}} \\
$\mathcal N$ & Agent set $\{1,\dots,N\}$ & $\mathcal S$ & Latent state space \\
$\mathcal A=\prod_i\mathcal A_i$ & Joint action space & $\mathcal O=\prod_i\mathcal O_i$ & Joint private-obs space \\
$\mathcal P$ & Transition kernel $\mathcal S\times\mathcal A\to\Delta(\mathcal S)$ & $\gamma$ & Discount factor $\in(0,1)$ \\
$\mu_0$ & Initial state distribution & & \\
\midrule
\multicolumn{4}{@{}l}{\textit{Novel IIMG-specific components}} \\
$\mathcal Y$ & Public-stream space & $G$ & Public-stream map $\mathcal A\to\mathcal Y$ \\
$\Omega$ & Observation kernel $\mathcal S\times\mathcal A\to\Delta(\mathcal O)$ & $\phi_i$ & Agent $i$'s reward component \\
\bottomrule
\end{tabular*}

The function $G$ determines what becomes common knowledge among agents;
varying $G$ captures a spectrum from no communication ($G$ constant) to
full transparency ($G$ bijective).
Agent~$i$ receives instantaneous reward
\[
R_i(s,\mathbf a)\ :=\ \phi_i\!\big(s,\,G(\mathbf a)\big),
\]
with discount factor $\gamma\in(0,1)$ and initial state
$s^0\sim \mu_0$.
At each time $t\ge 0$, agents simultaneously choose a
(possibly structured) macro-action $\mathbf a^t$, producing the public
stream token $y^t=G(\mathbf a^t)$.
The state then transitions to
$s^{t+1}\sim\mathcal P(\cdot\mid s^t,\mathbf a^t)$, and each agent
receives a private observation
$o_i^{t+1}\sim \Omega(\cdot\mid s^{t+1},\mathbf a^t)$.
This action-only $G$ is a convenient normal form; state-dependent or
within-round sequencing effects can be represented by augmenting
$\mathcal S$ and refining the time index.

\paragraph{Interpretation in MA-LLM coordination.}
In MA-LLM systems, the latent state $s^t$ aggregates the task instance,
environment/tool state, and any hidden ground truth needed to define rewards.
An action $a_i^t$ corresponds to the agent's round-$t$ committed act under the
protocol (e.g., a publicly posted message, a one-shot candidate, a tool call,
or a bounded decoding/prompt-control choice).
The public stream token $y^t$ records the protocol's publicly committed record
that becomes common knowledge (e.g., shared transcript segments, judge feedback,
or a coordinator's aggregated outcome), while the private observation $o_i^t$
captures agent-local information not revealed in $y^t$ (private scratchpad,
retrieved documents, tool traces, etc.).
The discount $\gamma$ also models attenuation of far-future influence under
context limits and truncation by downweighting long-range effects.

\begin{definition}[Explicit-communication protocol]
\label{def:explicit_comm}
A multi-agent coordination architecture is an
\emph{explicit-communication protocol} if there exists a finite
public-stream space $\mathcal Y$ and a publicly observed process
$\{y^t\}_{t\ge 0}$ such that each agent~$i$ can be modeled by a
behavioral policy $\pi_i(\cdot\mid h_i^t,t)$ with information state
\[
h_i^t=(o_i^0,\ a_i^0,\ y^0,\ r_i^0,\ \ldots,\ a_i^{t-1},\
y^{t-1},\ r_i^{t-1},\ o_i^t),
\qquad h_i^0=(o_i^0),
\]
and any private message available only to agent~$i$ is treated as part
of $o_i^t$.
The public stream is the only coordination-relevant information shared
across agents; no agent observes another agent's private history beyond
what is revealed in $\{y^u\}_{u\le t}$.
\end{definition}
\noindent In words, agent~$i$'s information state at time~$t$ groups
four types of data: private observations $(o_i^0, \ldots, o_i^t)$,
own actions $(a_i^0, \ldots, a_i^{t-1})$, public-stream tokens
$(y^0, \ldots, y^{t-1})$, and received rewards
$(r_i^0, \ldots, r_i^{t-1})$.

\begin{lemma}[Protocol embedding]\label{lem:protocol_embedding}
Any explicit communication protocol in
Definition~\ref{def:explicit_comm} can be represented as a discounted
IIMG by augmenting the state and observation kernels so that, for any
joint behavioral policy $\boldsymbol\pi$, the induced trajectory
distribution over $(s^t,y^t,\mathbf o^t,\mathbf a^t)_{t\ge 0}$ and
each discounted return $U_i(\boldsymbol\pi)$ are preserved.
Consequently, equilibrium notions and Bayesian regret computed from
Eq.~\eqref{eq:unreg_return}--Eq.~\eqref{eq:bayes_regret_def} are invariant
under this embedding.
\end{lemma}
\noindent A constructive embedding is provided in the Appendix.
Intuitively, the embedding augments the state to track the protocol's
internal bookkeeping, so any trajectory distribution under the
protocol is faithfully reproduced in the IIMG.

\paragraph{DICE within the explicit-communication framework.}
DICE instantiates Definition~\ref{def:explicit_comm} with a minimal
public stream that records the coordinator's context, together with the
aggregated outcome.
DICE's equilibrium selection mechanism, not communication volume, is
what provides stability.

\paragraph{Behavioral policies.}
A (possibly nonstationary) behavioral policy $\pi_i$ maps information
states to action distributions,
$\pi_i(\cdot\mid h_i^t,t)\in\Delta(\mathcal A_i)$.
We write $\boldsymbol\pi=(\pi_1,\ldots,\pi_N)$ and
$\Pi=\prod_{i=1}^N \Pi_i$.

\paragraph{Returns and Bayesian regret.}
Given $\boldsymbol\pi$, agent~$i$'s discounted return and social welfare
(the sum of all agents' returns) are
\begin{equation}
U_i(\boldsymbol\pi)
= \mathbb{E}_{\boldsymbol\pi}\Big[\sum_{t=0}^\infty \gamma^t
  R_i(s^t,\mathbf a^t)\Big],
\qquad
W(\boldsymbol\pi)=\sum_{i=1}^N U_i(\boldsymbol\pi).
\label{eq:unreg_return}
\end{equation}
We measure performance over \emph{learning iterations}
$k=1,2,\ldots$ (distinct from environment time~$t$) via Bayesian regret
relative to a reference joint policy $\boldsymbol\pi^\star$:
\begin{equation}
\mathrm{Regret}(T)
:=\mathbb{E}\Big[\sum_{k=1}^T
  \big(W(\boldsymbol\pi^\star)-W(\boldsymbol\pi^k)\big)\Big].
\label{eq:bayes_regret_def}
\end{equation}

\paragraph{Occupancies and policy geometry.}
To translate convergence in policy space into welfare bounds, we need a
sensitivity lemma that relates policy changes to occupancy changes.
We now introduce the necessary occupancy measures.
The discounted state occupancy is
\[
d^{\boldsymbol\pi}(s)\ :=\ (1-\gamma)\sum_{t\ge 0}\gamma^t\,
\Pr_{\boldsymbol\pi}(s^t=s).
\]
Let $h^t=(h_1^t,\ldots,h_N^t)$ denote the joint history. The
discounted joint history occupancy is
\[
\mathsf H_{\mathrm{disc}}^{\boldsymbol\pi}(h^t,t)\ :=\
(1-\gamma)\gamma^t\ \mathsf{Law}(h^t\mid \boldsymbol\pi).
\]
We measure worst-case deviations of joint action distributions using the
joint-policy TV metric, denoted $\rho_\Pi$ to distinguish it from both
the occupancy measure $d^{\boldsymbol\pi}$ and the reward components
$\phi_i$:
\begin{equation}
\rho_{\Pi}(\boldsymbol\pi,\boldsymbol\pi')
:=\sup_{(h^t,t)} \mathrm{TV}\Big(
  \textstyle\bigotimes_{i=1}^N \pi_i(\cdot\mid h_i^t,t),\
  \textstyle\bigotimes_{i=1}^N \pi_i'(\cdot\mid h_i^t,t)\Big).
\label{eq:policy_metric}
\end{equation}
Occupancy sensitivity holds with
$C_{\mathrm{occ}}=\frac{2\gamma}{1-\gamma}$
(by a standard simulation lemma; a derivation is included in the Appendix):
\[
\|d^{\boldsymbol\pi}-d^{\tilde{\boldsymbol\pi}}\|_1
\le C_{\mathrm{occ}}\,
\rho_{\Pi}(\boldsymbol\pi,\tilde{\boldsymbol\pi}).
\]

\paragraph{Standing assumptions.}
We invoke the following standard conditions only when required.
Each theorem statement specifies which assumptions are needed.
\begin{itemize}[leftmargin=1.25em,topsep=2pt,itemsep=1pt]
\item[(A1)] Finite spaces: $\mathcal S$, each $\mathcal A_i$, and each
  $\mathcal O_i$ are finite.
\item[(A2)] Bounded rewards: $|\phi_i|\le R_{\max}$ for all~$i$, and a
  policy-independent $\mu_0$.
\item[(A3)] Measurable behavioral policies (measurable in countable
  coordinates, i.e., the policy is a well-defined function of the
  history).
\item[(A4)] Perfect recall (behavioral strategies are w.l.o.g.\ under
  truncation/discounting).
\item[(A5)] Limited feedback: at time~$t$, agent~$i$ observes only
  $(o_i^t,y^t,r_i^t)$ and not other agents' actions.
\end{itemize}
\noindent Assumptions (A1)--(A4) are standard in the Dec-POMDP and
stochastic-game literature; (A5) is specific to this work and captures
the information asymmetry inherent in MA-LLM coordination, where agents
share information only through the public stream.
(A1) ensures that the best-response map is well-defined on a compact
domain; (A2) guarantees bounded returns; (A3)--(A4) provide the
measurability needed for behavioral strategies.

\subsection{Unregularized equilibria and regret control}
\label{subsec:bne}

With the IIMG framework and standing assumptions in place, we now
characterize the regret consequences of operating without an
equilibrium-selection mechanism.
Before introducing regularization, we establish that unregularized
games admit equilibria (Theorem~\ref{thm:bne_existence}) but provide no
mechanism to select among them: the game does not specify \emph{which}
equilibrium open-ended learning will select.
This gap is the root cause of the instability identified in the
introduction, and it alone is sufficient for regret to accumulate, even
when the underlying game admits equilibria with good welfare.

\begin{theorem}[Compactness and continuity of the policy space]
\label{thm:topology}
Under (A1)--(A2), the joint behavioral policy space $\Pi$ is nonempty,
convex, compact, and metrizable under the product topology.
Moreover, each payoff $U_i:\Pi\to\mathbb{R}$ defined in
Eq.~\eqref{eq:unreg_return} is continuous.
\end{theorem}
\noindent Proof is deferred to Appendix~\ref{app:topology_proof}.
Compactness ensures that best-response sequences have convergent
subsequences; continuity ensures that limits of approximate equilibria are exact equilibria.
Both properties are reused in the existence proof for the regularized equilibrium.

\paragraph{Bayesian Nash equilibrium (BNE).}
A joint policy $\boldsymbol\pi^\star$ is a BNE if no agent can improve its return by unilateral deviation:
\[
\pi_i^\star \in \arg\max_{\pi_i\in\Pi_i}
U_i(\pi_i,\boldsymbol\pi_{-i}^\star),
\qquad \forall i\in\mathcal N.
\]

\begin{theorem}[Existence of BNE]\label{thm:bne_existence}
Under assumptions (A1)--(A4), the discounted IIMG admits a behavioral
BNE.
\end{theorem}
\noindent The proof uses finite-horizon truncations and compactness; details are deferred to
Appendix~\ref{app:bne_existence}.
Existence alone does not tell us which equilibrium learning will
select; the next subsections show that multiplicity of BNE is the
core obstacle.

\paragraph{Regret transfer from convergence to a fixed equilibrium.}
Fix a reference BNE $\boldsymbol\pi^\star$ and define the
occupancy-averaged deviation
\[
\bar\Delta_k
:=\mathbb{E}_{(h^t,t)\sim
  \mathsf H_{\mathrm{disc}}^{\boldsymbol\pi^\star}}
\Big[\mathrm{TV}\big(\boldsymbol\pi^k(\cdot\mid h^t,t),
  \boldsymbol\pi^\star(\cdot\mid h^t,t)\big)\Big].
\]
Intuitively, $\bar\Delta_k$ measures how far the current policy is from
the reference equilibrium, averaged over visited states.
The following corollary shows that the rate at which $\bar\Delta_k$
decays directly controls regret.
\begin{corollary}[Rate-dependent Bayesian regret]\label{cor:regret}
If $\mathbb{E}[\bar\Delta_k]\le C_\Delta\,k^{-\alpha}$ for some
$C_\Delta>0$ and $\alpha>0$, then for all $T\ge 1$,
\[
\mathrm{Regret}(T)
\le \frac{2 N R_{\max} C_\Delta}{(1-\gamma)^2}\sum_{k=1}^T k^{-\alpha}.
\]
In particular, for $\alpha=\tfrac12$ this yields
$\mathrm{Regret}(T)=O\!\big(\frac{N\sqrt{T}}{(1-\gamma)^2}\big)$.
\end{corollary}
\noindent The bound follows from a performance-difference argument
combined with TV control; details are deferred to Appendix~\ref{app:welfare_sensitivity}.
Corollary~\ref{cor:regret} is informative only when learning converges
to a \emph{single} equilibrium so that $\bar\Delta_k\to 0$ for some
fixed $\boldsymbol\pi^\star$.
The next two subsections show that this condition can fail through two
distinct mechanisms, selection instability and persistent defensive
mixing, regardless of how rich the public stream is.

\subsection{Linear regret under equilibrium multiplicity}
\label{subsec:multiplicity}

When a game admits multiple equilibria with different welfare,
learning that does not settle on a single mode incurs a welfare gap
whenever it selects a lower-welfare equilibrium.
If this mis-selection persists with non-vanishing probability, regret
can accumulate linearly, even when every iterate is itself an equilibrium.

\begin{proposition}[Equilibrium multiplicity with persistent mis-selection can induce linear
Bayesian regret]\label{prop:multi_linear}

Assume the unregularized game admits two equilibria
$\boldsymbol\pi^{(1)}$ and $\boldsymbol\pi^{(2)}$ such that
$W(\boldsymbol\pi^{(1)})-W(\boldsymbol\pi^{(2)})=\Delta_W>0$.
Let $\{\boldsymbol\pi^k\}_{k\ge 1}$ be the iterates produced by a learning
procedure and suppose $\boldsymbol\pi^k\in\{\boldsymbol\pi^{(1)},\boldsymbol\pi^{(2)}\}$
for all $k$.
If there exist $p_0>0$ and $k_0$ such that for all $k\ge k_0$,
\[
\Pr\!\big(\boldsymbol\pi^k=\boldsymbol\pi^{(2)}\big)\ \ge\ p_0,
\]
then the Bayesian regret relative to $\boldsymbol\pi^{(1)}$ is linear:
\[
\mathrm{Regret}(T)\ \ge\ p_0\,\Delta_W\, (T-k_0+1)\ =\ \Omega(\Delta_W\,T).
\]
More generally, the same conclusion (up to constants) holds if
$\Pr\!\big(W(\boldsymbol\pi^k)\le W(\boldsymbol\pi^{(1)})-\Delta_W/2\big)\ge p_0$
for all $k\ge k_0$.

\end{proposition}
\noindent Proof is deferred to Appendix~\ref{app:multi_linear}.
The assumption that iterates lie exactly in
$\{\boldsymbol\pi^{(1)},\boldsymbol\pi^{(2)}\}$ is illustrative to
the reader; the general statement (final two lines of the proposition)
is the practically relevant claim, as the general condition applies
whenever learning persistently visits a low-welfare region with
bounded-away probability.
Proposition~\ref{prop:multi_linear} is diagnostic:
it formalizes how the empirically observed persistent mis-selection
in multi-agent LLM systems (Section~\ref{subsec:text_hanabi_margin})
translates to linear regret, motivating explicit equilibrium
selection.

\subsection{Debate-style dynamics and persistent regret}
\label{subsec:mad}

Even rich explicit communication cannot prevent persistent welfare gaps
when strategic uncertainty forces agents into defensive randomization,
and regret accumulates linearly.

\paragraph{Mechanism and link to selection.}
Debate-style protocols often induce an adversarial evaluation loop (e.g., through
judge selection or counter-argument pressure), so agents hedge against being
exploited by maintaining randomized defenses.
As a concrete illustration, consider a two-agent debate where agent~1
proposes a solution and agent~2 critiques it.
If agent~1 commits to a deterministic strategy, agent~2 can craft a
targeted counter-argument; anticipating this, agent~1 randomizes over
alternative proposals, and agent~2 likewise randomizes its critique style.
The result is a mixed equilibrium where both agents hedge rather than
committing to the welfare-maximizing coordinated strategy.
In the IIMG view, best-response behavior remains effectively set-valued near saddle-like regions.
Because both agents hedge, learning can stabilize in defensive mixing rather than committing to a high-welfare equilibrium, yielding a non-vanishing per-iteration welfare gap.

\begin{lemma}[Persistent suboptimality under defensive mixing]
\label{lem:debate_suboptimality_main}
Consider a debate-style multi-agent setting and fix an agent~$i$.
Let $s_k$ denote the (random) environment state encountered at learning
iteration $k$ (e.g., the initial state of the $k$-th episode).
Under persistent mixing and nondegenerate action-value gaps
(Appendix~\ref{appendix:comparison_debate}), there exist
$\delta_{\min}>0$ and $k_0$ such that for all $k\ge k_0$,
\[
\mathbb{E}\!\left[\max_{a_i}
  Q_i^{\star}(s_k,a_i,a_{-i}^k)
  \ -\ Q_i^{\star}(s_k,a_i^k,a_{-i}^k)\right]
\ \ge\ \delta_{\min}.
\]
Here $Q_i^{\star}(s, a_i, a_{-i})$ is agent~$i$'s action-value
function under the opponents' current policy.
\end{lemma}
\noindent Proof is deferred to
Appendix~\ref{appendix:comparison_debate}.

\begin{corollary}[Linear Bayesian regret for debate-style dynamics]
\label{cor:mad_linear}

Under the conditions of Lemma~\ref{lem:debate_suboptimality_main} and
standard $O(1/\sqrt{k})$ estimation errors, debate-style dynamics incur
linear Bayesian regret:
\[
\mathrm{Regret}_{\mathrm{MAD}}(T)=\Omega(T).
\]
\end{corollary}
\noindent Proof is deferred to Appendix~\ref{appendix:comparison_debate}.
Like Proposition~\ref{prop:multi_linear}, this result is diagnostic: it formalizes the regret consequence of empirically observed defensive mixing (Figure~\ref{fig:panel_2x2}a).
Sections~\ref{subsec:multiplicity} and \ref{subsec:mad} establish
that ill-posed equilibrium selection is the common root cause of
both failure modes.
Section~\ref{sec:hqre_theory} addresses this through entropy
regularization.

\keybox{
\textbf{Implication.}
Sublinear Bayesian regret requires convergence to a single equilibrium.
Equilibrium multiplicity and debate-induced defensive mixing each prevent convergence, leading to $\Omega(T)$ regret.}

\section{Regularized Equilibrium Theory: HQRE Uniqueness and Convergence}
\label{sec:hqre_theory}

\begin{table}[t!]
\centering
\caption{Summary of HQRE guarantees in Sec.~\ref{sec:hqre_theory}.}
\label{tab:core_results_hqre}
\small
\begin{tabular}{p{4.0cm}p{6.2cm}p{3.2cm}}
\toprule
\textbf{Result} & \textbf{Guarantee}
& \textbf{Representative condition} \\
\midrule
HQRE existence (\S\ref{subsec:hqre})
& Existence of HQRE in discounted IIMG
& $\alpha_{\min}>0$ \\[3pt]

Regularization gap bound (\S\ref{subsec:hqre})
& $J_i^{\boldsymbol\alpha}(\boldsymbol\pi)-U_i(\boldsymbol\pi)$ is
  uniformly bounded
& Entropy bound on $\Delta(\mathcal A_i)$ \\[3pt]

HQRE uniqueness (\S\ref{subsec:hqre})
& Unique HQRE in the strongly monotone regime
& $\alpha_{\min} > L_c$ \\[3pt]

Explicit KL-mirror linear convergence (\S\ref{subsec:hqre})
& Linear convergence of explicit mirror steps in KL geometry
& $\alpha_{\min} > L_c$ plus field Lipschitzness and local quadratic KL geometry \\[3pt]

Zero-temperature limit (\S\ref{subsec:hqre})
& HQRE $\rightarrow$ unregularized equilibria as $\alpha\downarrow 0$
& $\alpha_{\max}\downarrow 0$ \\[3pt]

Bounded Bayesian regret (\S\ref{subsec:hqre})
& Linear KL convergence implies Bayesian regret bounded uniformly
  in~$T$
& $\alpha_{\min} > L_c$ \\[3pt]

Hierarchical scalability (\S\ref{subsec:hier_theory})
& Unique Local--Global HQRE with structure-dependent stability margin
& $\mu_{\mathrm{hier}} > 0$ \\
\bottomrule
\end{tabular}
\end{table}

The essential results in this section are:
Theorem~\ref{thm:hqre_unique_mono} (HQRE uniqueness under strong
monotonicity, plus linear convergence of explicit KL-mirror steps under
additional regularity), and Corollary~\ref{cor:hqre_regret_bounded}
(bounded Bayesian regret from the linear-rate guarantee).
Lemma~\ref{lem:zero_temp_full} (zero-temperature limit connecting HQRE
to unregularized BNE) is useful context but not required for the
algorithm sections.
The hierarchical extension
(Section~\ref{subsec:hier_theory}) can be consulted selectively.
The key methodological advance in this section is the monotonicity
criterion itself: it provides a single verifiable condition that
simultaneously guarantees uniqueness and convergence, replacing the
ad hoc stability checks prevalent in multi-agent LLM practice.

\subsection{HQRE: existence, uniqueness, and convergence}
\label{subsec:hqre}

We introduce an entropy-regularized equilibrium concept that
converts set-valued best responses into
single-valued maps, yielding uniqueness and convergence guarantees.
We regularize each agent's objective by adding an entropy bonus
weighted by a temperature parameter $\alpha_{i,t}$.
Let $\mathcal H(p):=-\sum_a p(a)\log p(a)$ denote Shannon entropy and
fix a heterogeneous temperature schedule
$\boldsymbol\alpha=\{\alpha_{i,t}(h_i^t)\}$.
The entropy-regularized objective is
\begin{equation}
J_i^{\boldsymbol\alpha}(\boldsymbol\pi)
:=\mathbb{E}_{\boldsymbol\pi}\Big[\sum_{t\ge 0}\gamma^t
  \big(R_i(s^t,\mathbf a^t)
  +\alpha_{i,t}(h_i^t)\,\mathcal H(\pi_i(\cdot\mid h_i^t,t))
  \big)\Big].
\label{eq:J_i_def}
\end{equation}
Write $\alpha_{\min}:=\inf_{i,t,h_i^t}\alpha_{i,t}(h_i^t)$ and
$\alpha_{\max}:=\sup_{i,t,h_i^t}\alpha_{i,t}(h_i^t)$.
These extremes bound the range of regularization strength across all
agents, states, and times.

\begin{lemma}[Pointwise regularization gap]
\label{lem:reg_gap_pointwise}
For any joint policy $\boldsymbol\pi\in\Pi$ and any agent~$i$,
\[
0 \ \le\ J_i^{\boldsymbol\alpha}(\boldsymbol\pi)-U_i(\boldsymbol\pi)
\ \le\ \frac{\alpha_{\max}\log|\mathcal A_i|}{1-\gamma}.
\]
Consequently,
\[
0 \ \le\ \sum_{i=1}^N J_i^{\boldsymbol\alpha}(\boldsymbol\pi)
  -W(\boldsymbol\pi)
\ \le\ \frac{\alpha_{\max}}{1-\gamma}\sum_{i=1}^N\log|\mathcal A_i|.
\]
\end{lemma}
\noindent The bound follows directly from
$\mathcal H(\pi_i(\cdot\mid h_i^t,t))\le \log|\mathcal A_i|$; a proof is deferred to the Appendix.
Given $\boldsymbol\pi$, let
$Q_i^{\boldsymbol\pi}(h_i^t,a_i,t)
:=\mathbb{E}_{\boldsymbol\pi}\big[\sum_{s\ge t}\gamma^{s-t}
R_i(s^s,\mathbf a^s)\mid h_i^t,\,a_i^t=a_i\big]$
denote agent~$i$'s action-value function under $\boldsymbol\pi$.
The logit best response at each information state is defined by
\begin{equation}
\big(\mathcal B_i^{\boldsymbol\alpha}(\boldsymbol\pi)\big)
  (a_i\mid h_i^t,t)
:=\frac{\exp\!\big(Q_i^{\boldsymbol\pi}(h_i^t,a_i,t)
  /\alpha_{i,t}(h_i^t)\big)}
{\sum_{a_i'}\exp\!\big(Q_i^{\boldsymbol\pi}(h_i^t,a_i',t)
  /\alpha_{i,t}(h_i^t)\big)}.
\label{eq:logit_def}
\end{equation}
Let $\mathcal B^{\boldsymbol\alpha}(\boldsymbol\pi)
:=(\mathcal B_1^{\boldsymbol\alpha}(\boldsymbol\pi),\ldots,
\mathcal B_N^{\boldsymbol\alpha}(\boldsymbol\pi))$.
A \emph{heterogeneous quantal response equilibrium (HQRE)} is a fixed
point $\boldsymbol\pi^\star
=\mathcal B^{\boldsymbol\alpha}(\boldsymbol\pi^\star)$.
At an HQRE, every agent plays a softmax distribution over actions
weighted by their expected value, and all agents' value estimates are
mutually consistent.
The temperature controls how sharply each agent concentrates on
high-value actions: lower temperatures yield near-deterministic
play, while higher temperatures maintain exploratory mixing.
The key property is, unlike unregularized best responses, the
logit map is single-valued and continuous, enabling the
uniqueness and convergence results as follows.

\begin{theorem}[Existence of HQRE]\label{thm:hqre_existence}
Under (A1)--(A4) and $\alpha_{\min}>0$, an HQRE exists.
\end{theorem}
\noindent The proof applies Schauder's fixed-point theorem using
compactness of $\Pi$ and continuity of
$\mathcal B^{\boldsymbol\alpha}$; details are deferred to
Appendix~\ref{app:hqre_existence}.
Existence guarantees that HQRE is a meaningful solution concept; the
next theorem addresses uniqueness.
The following result shows that if the temperature is large enough to
dominate the worst-case cross-agent Q-value sensitivity $C_Q$, the
logit map is a contraction and the HQRE is unique.
The threshold scales with~$N$ because this criterion treats all
agent interactions uniformly. Formally:

\begin{theorem}[Contraction criterion]\label{thm:unique_hqre_joint}
For any $\boldsymbol\pi,\tilde{\boldsymbol\pi}\in\Pi$,
\[
\rho_{\Pi}\big(\mathcal B^{\boldsymbol\alpha}(\boldsymbol\pi),
  \mathcal B^{\boldsymbol\alpha}(\tilde{\boldsymbol\pi})\big)
\le \frac{N C_Q}{2\alpha_{\min}}\,
\rho_{\Pi}(\boldsymbol\pi,\tilde{\boldsymbol\pi}),
\]
where $C_Q\le\frac{(2-\gamma)R_{\max}}{(1-\gamma)^2}$ bounds the
sensitivity of agent~$i$'s Q-values to opponents' policy changes in
TV distance (see Lemma~\ref{lem:q_cont} for the formal statement).
Hence, if $\alpha_{\min}>\frac{N C_Q}{2}$, the HQRE is unique and the
synchronous iteration
$\boldsymbol\pi^{k+1}=\mathcal B^{\boldsymbol\alpha}(\boldsymbol\pi^k)$
converges geometrically.
\end{theorem}
\noindent Proof is deferred to Appendix~\ref{app:hqre_contraction}.
Theorem~\ref{thm:hqre_unique_mono} below provides a sharper criterion
that is used in all subsequent results.
Theorem~\ref{thm:unique_hqre_joint} provides a clean sufficient
condition but can be loose, because the contraction condition treats all
cross-agent interactions uniformly (hence the factor of $N$).
We now present a sharper monotonicity criterion that accounts for the
actual strength of coupling between agents.
In loosely coupled games, monotonicity imposes a much weaker requirement
on the temperature.

\paragraph{Weighted KL geometry.}
Intuitively, we weight each information state by how often it is visited
under the target policy.
This means that convergence guarantees focus on states that matter for
performance, rather than worst-case states that may never be reached.
Let $\{w_i(h_i^t,t)\}$ be positive weights with
$\sum_{(h_i^t,t)} w_i(h_i^t,t)=1$ and define the weighted block norm
\[
\|\boldsymbol\pi-\boldsymbol\pi'\|_{1,2;w}^2
:=\sum_{i=1}^N\sum_{(h_i^t,t)} w_i(h_i^t,t)\,
\big\|\pi_i(\cdot\mid h_i^t,t)-\pi_i'(\cdot\mid h_i^t,t)\big\|_1^2,
\]
and the weighted KL--Bregman divergence $D_\Psi$ as follows:
\begin{align}
\Psi(\boldsymbol\pi)
&:= \sum_{i=1}^N\sum_{(h_i^t,t)} w_i(h_i^t,t)\,\alpha_{i,t}(h_i^t)
\sum_{a_i\in\mathcal A_i}\pi_i(a_i\mid h_i^t,t)
  \log \pi_i(a_i\mid h_i^t,t),
\label{eq:psi_def}\\
D_\Psi(\boldsymbol\pi,\boldsymbol\pi')
&:= \Psi(\boldsymbol\pi)-\Psi(\boldsymbol\pi')
-\Big\langle \nabla\Psi(\boldsymbol\pi'),\,
  \boldsymbol\pi-\boldsymbol\pi' \Big\rangle \nonumber\\
&= \sum_{i=1}^N\sum_{(h_i^t,t)} w_i(h_i^t,t)\,\alpha_{i,t}(h_i^t)\,
\mathrm{KL}\!\Big(\pi_i(\cdot\mid h_i^t,t)
  \ \Big\|\ \pi_i'(\cdot\mid h_i^t,t)\Big).
\label{eq:weighted_bregman}
\end{align}
In words, $D_\Psi$ is a temperature- and occupancy-weighted KL divergence between joint policies.
We define $L_c$ as the tightest constant such that the
unregularized operator is $L_c$-hypomonotone in this geometry, measuring the worst-case cross-agent coupling:
\begin{equation}
\label{eq:Lc_def_main}
-\sum_{i=1}^N\big\langle
  \nabla_{\pi_i}U_i(\boldsymbol\pi)
  -\nabla_{\pi_i}U_i(\boldsymbol\pi'),\,
  \pi_i-\pi_i'\big\rangle
\ \ge\ -L_c\,\|\boldsymbol\pi-\boldsymbol\pi'\|_{1,2;w}^2,
\qquad \forall \boldsymbol\pi,\boldsymbol\pi'\in\Pi.
\end{equation}
When $L_c = 0$ the game is fully cooperative; larger $L_c$ indicates stronger strategic coupling.
A computable upper bound is in Appendix~\ref{app:hqre_mono};
Section~\ref{subsec:toy_2x2} illustrates the computation for a
concrete $2\times 2$ game.
Then, we have
\begin{equation}
\label{eq:F_reg_def}
F_{\mathrm{reg}}(\boldsymbol\pi)
:=\big(-\nabla_{\pi_1}J_1^{\boldsymbol\alpha}(\boldsymbol\pi),\ldots,
-\nabla_{\pi_N}J_N^{\boldsymbol\alpha}(\boldsymbol\pi)\big).
\end{equation}

\begin{theorem}[Monotonicity criterion: uniqueness and linear convergence]
\label{thm:hqre_unique_mono}

Let $\mu := \alpha_{\min}-L_c$.
For all $\boldsymbol\pi,\boldsymbol\pi'\in\Pi$,
\[
\big\langle F_{\mathrm{reg}}(\boldsymbol\pi)
  -F_{\mathrm{reg}}(\boldsymbol\pi'),\,
  \boldsymbol\pi-\boldsymbol\pi' \big\rangle
\ge \mu\,
\|\boldsymbol\pi-\boldsymbol\pi'\|_{1,2;w}^2.
\]
If $\mu>0$, the HQRE is unique.
Assume in addition that:
\begin{enumerate}[leftmargin=2em,topsep=2pt,itemsep=1pt]
\item[(i)] $F_{\mathrm{reg}}$ is $L_F$-Lipschitz in
$\|\cdot\|_{1,2;w}$;
\item[(ii)] the weighted KL geometry admits a local quadratic upper bound
on a neighborhood containing $\boldsymbol\pi^\star$ and the iterates,
namely
\[
\frac{\alpha_{\min}}{2}\|\boldsymbol\pi-\boldsymbol\pi'\|_{1,2;w}^2
\le D_\Psi(\boldsymbol\pi,\boldsymbol\pi')
\le \frac{L_\Psi}{2}\|\boldsymbol\pi-\boldsymbol\pi'\|_{1,2;w}^2.
\]
\end{enumerate}
A sufficient condition for (ii) is that the iterate region stays in the
$\nu_{\min}$-interior of each local simplex, in which case one may take
$L_\Psi = \alpha_{\max}/\nu_{\min}$; see
Appendix~\ref{app:linear_rate_proof}.

Consider the explicit mirror step
\begin{equation}
\label{eq:hqre_explicit_mirror_step}
\boldsymbol\pi^{k+1}
=\arg\min_{\tilde{\boldsymbol\pi}\in\Pi}
\Big\{\langle F_{\mathrm{reg}}(\boldsymbol\pi^k),
\tilde{\boldsymbol\pi}-\boldsymbol\pi^k\rangle
+\tfrac{1}{\eta}D_\Psi(\tilde{\boldsymbol\pi},\boldsymbol\pi^k)\Big\}.
\end{equation}
If
\[
0<\eta\le
\min\Big\{\frac{\mu\,\alpha_{\min}^2}{L_F^2L_\Psi},\,
\frac{\alpha_{\min}}{L_F}\Big\},
\]
then
\begin{equation}
\label{eq:kl_mirror_rate}
D_\Psi(\boldsymbol\pi^\star,\boldsymbol\pi^{k+1})
\le \rho\,D_\Psi(\boldsymbol\pi^\star,\boldsymbol\pi^k),
\qquad
\rho:=1-\eta\,\frac{\mu}{L_\Psi}\in(0,1).
\end{equation}
\end{theorem}
\noindent The uniqueness claim uses only strong monotonicity of
$F_{\mathrm{reg}}$. The linear-rate claim adds the standard
field-Lipschitz and local-quadratic assumptions needed for explicit
mirror descent in KL geometry. The proof is deferred to
Appendix~\ref{app:hqre_mono} and Appendix~\ref{app:linear_rate_proof}.

\noindent\emph{Remark.}
The additional assumptions for the linear-rate part are not implied by
monotonicity alone. In particular, the upper quadratic control of
$D_\Psi$ is local for entropy / KL mirror maps and requires an iterate
region that stays away from the boundary of the simplex. Practically,
our DICE-FT implementation enforces this regime through positive
temperatures, KL trust regions, and step rejection; see
Remark~\ref{rem:kl_quadratic}.

\begin{lemma}[Zero-temperature limit]\label{lem:zero_temp_full}
Let $\boldsymbol\alpha^{(n)}$ be temperature schedules\footnote{The superscript $(n)$ indexes the sequence of temperature schedules, not exponentiation.} with
$\alpha_{\max}^{(n)}
:=\sup_{i,t,h_i^t}\alpha_{i,t}^{(n)}(h_i^t)\downarrow 0$, and let
$\boldsymbol\pi^{(n)}$ be an HQRE under $\boldsymbol\alpha^{(n)}$.
Any accumulation point of $\{\boldsymbol\pi^{(n)}\}$ is a BNE of the
unregularized game.
\end{lemma}
\noindent Proof is deferred to Appendix~\ref{app:zero_temp}.

\begin{corollary}[Bounded Bayesian regret under linear KL convergence]
\label{cor:hqre_regret_bounded}

Fix a unique HQRE $\boldsymbol\pi^\star$ under $\boldsymbol\alpha$ in
the regime $\alpha_{\min}>L_c$, and define (note that
$\bar\Delta_k$ is now measured relative to the unique HQRE rather than
an arbitrary BNE as in Corollary~\ref{cor:regret})
\[
\bar\Delta_k
:=\mathbb{E}_{(h^t,t)\sim
  \mathsf H_{\mathrm{disc}}^{\boldsymbol\pi^\star}}
\Big[\mathrm{TV}\big(\boldsymbol\pi^k(\cdot\mid h^t,t),
  \boldsymbol\pi^\star(\cdot\mid h^t,t)\big)\Big].
\]
Suppose an algorithm produces iterates
$\{\boldsymbol\pi^k\}_{k\ge 1}$ such that
$\mathbb{E}[D_\Psi(\boldsymbol\pi^\star,\boldsymbol\pi^k)]
\le C_0\rho^k
$ for some $C_0>0$ and $\rho\in(0,1)$.
Then there exists a constant $C_{\mathrm{Pins}}$ (depending only on the
block weights and action spaces) such that
$\mathbb{E}[\bar\Delta_k]
\le C_{\mathrm{Pins}}\sqrt{C_0}\,\rho^{k/2}$, and the Bayesian regret
relative to $\boldsymbol\pi^\star$ is bounded uniformly in~$T$:
\[
\mathrm{Regret}(T)
\le \frac{2 N R_{\max} C_{\mathrm{Pins}}\sqrt{C_0}}{(1-\gamma)^2}
  \sum_{k=1}^T \rho^{k/2}
\le \frac{2 N R_{\max} C_{\mathrm{Pins}}\sqrt{C_0}}{(1-\gamma)^2}
  \cdot \frac{\sqrt{\rho}}{1-\sqrt{\rho}}.
\]
\end{corollary}
\noindent The Pinsker-type control and the bound follow from
Corollary~\ref{cor:regret} and are deferred to the Appendix.
\footnote{The constant $C_{\mathrm{Pins}}$ scales as
$O\!\big(\sqrt{\sum_i |\mathcal A_i|/w_{\min}}\big)$, where
$w_{\min}:=\min_{i,(h_i^t,t)}w_i(h_i^t,t)$; the bound remains
non-vacuous when action spaces are moderate and occupancy weights are
bounded away from zero.}
Together, Theorem~\ref{thm:hqre_unique_mono} and
Corollary~\ref{cor:hqre_regret_bounded} close the regret gap identified
in Section~\ref{sec:dice_framework}: regularization provides a unique target, fast convergence toward it, and bounded regret along the way.
\keybox{\textbf{Main result.}
Entropy regularization converts set-valued best responses into a single-valued logit map, yielding an explicit equilibrium-selection mechanism.
In the strongly monotone regime $\alpha_{\min}>L_c$, the selected HQRE is unique (Theorem~\ref{thm:hqre_unique_mono}); under the additional field-Lipschitz and local-quadratic assumptions stated there, explicit KL-mirror updates converge linearly (Eq.~\ref{eq:kl_mirror_rate}), and Bayesian regret is bounded uniformly in the horizon $T$ (Corollary~\ref{cor:hqre_regret_bounded}).}

\subsection{Hierarchical coordination via Local--Global HQRE}
\label{subsec:hier_theory}

The contraction threshold in Theorem~\ref{thm:unique_hqre_joint} scales
with $N$ and can be conservative for large agent populations.
A Local--Global hierarchy groups agents into clusters, each internally coordinated
via a local HQRE, with a global controller selecting between clusters, yielding stability margins governed by structural coupling rather than the raw agent count.
Let $\mathcal C=\{C_1,\ldots,C_K\}$ be a partition of agents into
$K$~clusters.
Let $z^t$ denote a finite public context variable that is measurable
with respect to the public stream $\{y^u\}_{u\le t}$, and write
$\mathcal Z$ for its range.
Within each cluster, a local coordinator aggregates
context-conditional values
\begin{equation}
\bar Q_{\mathrm{loc},c}(z,\mathbf a_{C_c},t)
:= \sum_{i\in C_c} w_i^{(c)}(z)\,
\bar Q_i^{\boldsymbol\pi}(z,a_i,t)\ +\ b_c(z),
\label{eq:hier_local}
\end{equation}
where
$\bar Q_i^{\boldsymbol\pi}(z,a_i,t)
:=\mathbb E[Q_i^{\boldsymbol\pi}(h_i^t,a_i,t)\mid z^t=z]$.
A global coordinator combines cluster values through a gating policy
$g:\mathcal Z\to\Delta^K$:
\begin{equation}
Q_{\mathrm{tot}}(z,\mathbf a,t)
:= \sum_{c=1}^K g_c(z)\,
\bar Q_{\mathrm{loc},c}(z,\mathbf a_{C_c},t)\ +\ b_0(z).
\label{eq:hier_global}
\end{equation}
We regularize execution agents with $\boldsymbol\alpha$ and the global
coordinator with temperature $\tau_g>0$.
Formal definitions and the composite KL geometry are provided in
Appendix~\ref{app:hier}.

\begin{theorem}[Hierarchical existence]\label{thm:hier_existence}
Under (A1)--(A2), $\alpha_{\min}>0$, and $\tau_g>0$, a Local--Global
HQRE exists.
\end{theorem}
\noindent Proof is deferred to Appendix~\ref{app:hier_exist}.

\begin{proposition}[Structural coupling bound]
\label{prop:hier_coupling_bound}
The hierarchical coupling constant can be bounded as
\[
L_{\mathrm{hier}}
\le W_{\max}^{\mathrm{loc}}\,C_Q
+ C_{\mathrm{occ}}\big(W_{\max}^{\mathrm{loc}}
  L_g^{\mathrm{dist}} R_{\max}
  + L_w^{\mathrm{dist}} R_{\max}\big),
\]
where $W_{\max}^{\mathrm{loc}}$ bounds local mixing weights and
$C_{\mathrm{occ}},C_Q$ are as in Sec.~\ref{subsec:notation}.
\end{proposition}
\noindent Proof is deferred to Appendix~\ref{app:hier_coupling}.

Let $L_c^{(c)}$ denote the within-cluster coupling constant and define
local margins $\mu_c:=\alpha_{\min}-L_c^{(c)}$ and gate margin
$\mu_g:=\tau_g$.
The hierarchical stability margin is
\begin{equation}
\mu_{\mathrm{hier}}
:= \min_{c\in\{1,\ldots,K\}} \mu_c\ +\ \mu_g\ -\ L_{\mathrm{hier}}.
\label{eq:mu_hier}
\end{equation}

\begin{theorem}[Hierarchical uniqueness and linear convergence]
\label{thm:hier_unique}
If $\mu_{\mathrm{hier}}>0$, then the Local--Global HQRE is unique.
Moreover, assume the hierarchical field is Lipschitz in the composite
block norm and that the composite KL geometry admits a local quadratic
upper bound on the iterate region, with constant
$L_\Psi^{\mathrm{hier}}$.
Then blockwise explicit mirror steps converge linearly in
$D_{\Psi}^{\mathrm{hier}}$ with rate
\[
\rho_{\mathrm{hier}} = 1-\eta_{\mathrm{hier}}
\frac{\mu_{\mathrm{hier}}}{L_\Psi^{\mathrm{hier}}}
\in(0,1)
\]
for sufficiently small $\eta_{\mathrm{hier}}$.
\end{theorem}
\noindent The composite geometry and proof are deferred to
Appendix~\ref{app:hier_mono_linear} and
Appendix~\ref{app:hier_linear_proof}. As in the flat case, uniqueness
uses strong monotonicity alone, while the linear-rate claim requires the
hierarchical analogues of field Lipschitzness and local quadratic KL
geometry. Section~\ref{sec:algorithm} instantiates these results as DICE-PC and DICE-FT.

\section{Algorithms: DICE-PC and DICE-FT}
\label{sec:algorithm}

This section presents two algorithms that bridge the gap between the exact KL-mirror theory above and practical LLM-scale optimization, targeting the
entropy-regularized HQRE objective of Sec.~\ref{sec:hqre_theory} under
different deployment constraints.
Throughout, $\boldsymbol\pi = (\pi_1, \ldots, \pi_N)$ denotes the
joint policy and $\pi_i$ denotes agent~$i$'s individual policy.
\emph{DICE-PC} coordinates frozen execution LLMs by learning
distributions over prompt-control actions, while \emph{DICE-FT}
fine-tunes token-level policies via clipped KL-prox mirror updates that
approximate the policy-space KL-mirror dynamics of
Theorem~\ref{thm:hqre_unique_mono}.
Both are implemented as coordinator-mediated instances of
Definition~\ref{def:explicit_comm}.
Readers primarily interested in experimental results may proceed directly to Section~\ref{sec:experiments}.

The essential results in this section are:
Corollary*~\ref{cor:approx_convergence} (inexact KL-mirror recurrence for DICE-FT),
Proposition~\ref{prop:delta_bound} (one-step field-error budget), and the two algorithms (Algorithms~\ref{alg:dice_pc}--\ref{alg:dice_ft}).
The error decomposition, empirical diagnostics, and sensitivity discussion provide supporting evidence and can be consulted selectively.

At each environment step~$t$, a Coordinator publishes a context
message~$m^t$ and aggregates one-shot candidates into the public
outcome~$y^t$.
We take the public stream token to be the augmented record
$\widetilde y^t=(m^t,y^t)$, so execution agents share only
$\{\widetilde y^u\}_{u\le t}$ and never observe each other's private
trajectories.
DICE-PC applies the finite-action HQRE response induced by
Eq.~\eqref{eq:logit_def} on learned candidate sets, whereas DICE-FT
approximates policy-space KL-prox mirror steps in the entropic geometry
induced by the HQRE objective Eq.~\eqref{eq:J_i_def}.
The HQRE temperature enters as an explicit entropy weight, while a KL
trust region controls the approximation error from parameterization
and finite-sample estimation.

\paragraph{Remark (Theory--algorithm gap for DICE-FT).}
The exact convergence guarantee of Theorem~\ref{thm:hqre_unique_mono}
applies to explicit mirror steps driven by the exact regularized field
$F_{\mathrm{reg}}(\boldsymbol\pi^k)$.
DICE-FT replaces that field with an implementable surrogate
\[
\widehat F^k = F_{\mathrm{reg}}(\boldsymbol\pi^k) + e^k,
\qquad
 e^k = e^k_{\mathrm{clip}} + e^k_{\mathrm{samp}} + e^k_{\mathrm{repr}} + e^k_{\mathrm{opt}},
\]
where the four terms correspond respectively to clipping, finite-sample
estimation, representation or critic approximation, and the residual gap
between the actual finite-step parameter update and the idealized
policy-space mirror step.
The correct object to bound is therefore not a KL distance to the
best-response map $\mathcal B^{\boldsymbol\alpha}$ but the dual-norm size
of the field perturbation $e^k$, which can then be propagated through an
inexact mirror recurrence.

\begin{corollary*}[Inexact KL-mirror recurrence for DICE-FT]
\label{cor:approx_convergence}
Assume the conditions of Theorem~\ref{thm:hqre_unique_mono} hold and let
$\boldsymbol\pi^\star$ denote the unique HQRE.
Model one DICE-FT update as inducing an iterate satisfying
\begin{equation}
\label{eq:inexact_mirror_main}
\boldsymbol\pi^{k+1}
=\arg\min_{\tilde{\boldsymbol\pi}\in\Pi}
\Big\{\langle F_{\mathrm{reg}}(\boldsymbol\pi^k)+e^k,
\tilde{\boldsymbol\pi}-\boldsymbol\pi^k\rangle
+\tfrac{1}{\eta}D_\Psi(\tilde{\boldsymbol\pi},\boldsymbol\pi^k)\Big\}.
\end{equation}
If
\[
0<\eta\le
\min\Big\{\frac{\mu\,\alpha_{\min}^2}{L_F^2L_\Psi},\,
\frac{\alpha_{\min}}{L_F},\,
\frac{\alpha_{\min}\mu}{4L_F^2}\Big\},
\qquad \mu:=\alpha_{\min}-L_c,
\]
then
\begin{equation}
\label{eq:inexact_main_recurrence}
D_\Psi(\boldsymbol\pi^\star,\boldsymbol\pi^{k+1})
\le \rho\,D_\Psi(\boldsymbol\pi^\star,\boldsymbol\pi^k)
+ C_{\mathrm{err}}\,\|e^k\|_\ast^2,
\end{equation}
where
\[
\rho := 1-\eta\,\frac{\mu}{L_\Psi}\in(0,1),
\qquad
C_{\mathrm{err}} := \frac{\eta}{\mu} + \frac{\eta^2}{\alpha_{\min}}.
\]
Consequently, if
$\mathbb E[\|e^k\|_\ast^2\mid \mathcal F_k]\le \bar{\mathcal E}_k$,
then
\[
\mathbb E[D_\Psi(\boldsymbol\pi^\star,\boldsymbol\pi^{k})]
\le \rho^k D_\Psi(\boldsymbol\pi^\star,\boldsymbol\pi^0)
+ C_{\mathrm{err}}\sum_{j=0}^{k-1}\rho^{k-1-j}
\mathbb E[\bar{\mathcal E}_j].
\]
If in addition $\sup_k\mathbb E[\bar{\mathcal E}_k]\le \bar{\mathcal E}$,
then
\[
\limsup_{k\to\infty}
\mathbb E[D_\Psi(\boldsymbol\pi^\star,\boldsymbol\pi^k)]
\le \frac{C_{\mathrm{err}}}{1-\rho}\,\bar{\mathcal E}.
\]
\end{corollary*}

\noindent The proof is deferred to
Appendix~\ref{app:linear_rate_proof}.
The key point is that the practical algorithm is analyzed through a
field-error recurrence in dual norm rather than through any triangle
inequality in KL space.

\begin{proposition}[One-step field-error budget for DICE-FT]
\label{prop:delta_bound}
Write the implementable field error in
Eq.~\eqref{eq:inexact_mirror_main} as
$e^k=e^k_{\mathrm{clip}}+e^k_{\mathrm{samp}}+e^k_{\mathrm{repr}}+e^k_{\mathrm{opt}}$.
Assume:
\begin{enumerate}[leftmargin=2em,topsep=2pt,itemsep=1pt]
\item (Bounded update features) the per-sample score / update feature
satisfies $\|s(\tau)\|_\ast\le G_{\max}$ almost surely.
\item (Clip-overshoot moment) defining
\[
\Delta_{\mathrm{clip}}^k
:= \mathbb E\Big[(\hat r(\tau)-\mathrm{clip}_{\varepsilon}(\hat r(\tau)))^2
\,\Big|\,\mathcal F_k\Big],
\]
we have
$\mathbb E[\|e_{\mathrm{clip}}^k\|_\ast^2\mid \mathcal F_k]
\le G_{\max}^2\Delta_{\mathrm{clip}}^k$.
\item (Finite-sample estimation)
$\mathbb E[\|e_{\mathrm{samp}}^k\|_\ast^2\mid \mathcal F_k]
\le \sigma_g^2/G$.
\item (Monotone-mixer transfer) the representation / critic
approximation error obeys
\[
\mathbb E[\|e_{\mathrm{repr}}^k\|_\ast^2\mid \mathcal F_k]
\le C_{\mathrm{mix}}^2\,(L_{\mathrm{repr}}^k)^2,
\]
where $L_{\mathrm{repr}}^k$ is a measurable per-agent value-approximation
proxy and $C_{\mathrm{mix}}$ is the local transfer constant induced by
the nonnegative mixer in Eq.~\eqref{eq:dice_mono_mixer}; see
Appendix~\ref{app:repr_bound}.
\item (Optimization / parameterization residual)
\[
\mathbb E[\|e_{\mathrm{opt}}^k\|_\ast^2\mid \mathcal F_k]
\le (L_{\mathrm{opt}}^k)^2,
\]
where $L_{\mathrm{opt}}^k$ is any measurable proxy for the residual gap
between the actual finite-step neural update and the idealized
policy-space mirror step.
\end{enumerate}
Then
\begin{equation}
\label{eq:prop2_field_budget}
\mathbb E[\|e^k\|_\ast^2\mid \mathcal F_k]
\le 4\Big(
G_{\max}^2\Delta_{\mathrm{clip}}^k
+\frac{\sigma_g^2}{G}
+C_{\mathrm{mix}}^2\,(L_{\mathrm{repr}}^k)^2
+(L_{\mathrm{opt}}^k)^2
\Big).
\end{equation}
If, in addition,
$\Delta_{\mathrm{clip}}^k\le p_{\mathrm{clip}}^k c_{\mathrm{ov}}^2\varepsilon^2$
for some overshoot constant $c_{\mathrm{ov}}$, then the clipping
contribution is $O(p_{\mathrm{clip}}^k\varepsilon^2)$.
\end{proposition}

\noindent\emph{Proof sketch.}
Decompose $e^k$ in dual norm using
$\|a+b+c+d\|_\ast^2 \le 4(\|a\|_\ast^2 + \cdots + \|d\|_\ast^2)$
and bound each component via the stated assumptions.
Full derivation in Appendix~\ref{app:prop2_details}. \qed

Combining Corollary*~\ref{cor:approx_convergence} and
Proposition~\ref{prop:delta_bound} yields the practical DICE-FT
recurrence with error budget $\mathcal{E}_k$ given by the
right-hand side of Eq.~\eqref{eq:prop2_field_budget}.

\paragraph{Empirical error diagnostics.}
We therefore track the measurable components entering
Eq.~\eqref{eq:prop2_field_budget}: the clip-overshoot moment
$\Delta_{\mathrm{clip}}^k$, the sampling term $\sigma_g^2/G$, and two
approximation proxies $L_{\mathrm{repr}}^k$ and $L_{\mathrm{opt}}^k$.
These diagnostics support the qualitative prediction of the inexact
recurrence without collapsing the theory to a single post-hoc scalar.
The field-error budget predicts that halving $G$ from 8 to 4 roughly doubles the sampling contribution, consistent with the observed performance degradation at $G{=}4$ in our ablations (Appendix~\ref{additional_experiment}).

\paragraph{Sensitivity to group size.}
Proposition~\ref{prop:delta_bound} isolates the only explicit
$G$-dependence in the sampling term $\sigma_g^2/G$.
This predicts diminishing returns from increasing $G$ once sampling
noise is no longer the dominant term in the error budget.
Our group-size ablations are consistent with this pattern and identify
$G{=}8$ as the best cost--accuracy trade-off in our setting; see
Appendix~\ref{additional_experiment}.

We control these gaps empirically through the KL diagnostics defined
below ($\widehat{\mathrm{KL}}_{\mathrm{old}}$ and
$\widehat{\mathrm{KL}}_{\mathrm{ref}}$), which serve as the
rollout-measurable stability proxies described in the introduction.
Section~\ref{sec:experiments} reports these diagnostics alongside
end-task performance.

\begin{table}[t!]
\centering
\small
\caption{Architectural Comparison: DICE-PC vs. DICE-FT.}
\label{tab:dice_comparison}
\setlength{\tabcolsep}{8pt}
\begin{tabularx}{\textwidth}{@{} l >{\raggedright\arraybackslash}X >{\raggedright\arraybackslash}X @{}}
\toprule
\textbf{Feature} & \textbf{DICE-PC (Prompt-Control)} & \textbf{DICE-FT (Fine-Tuning)} \\
\midrule
\textbf{Action Space} & Finite sets $\mathcal{U}$ (Prompt-based) & Full vocab $\mathcal{V}$ (Token-based) \\
\textbf{Learned Comps.} & Belief / Controller / Critic / Mixer & LoRA Policy / Critic / Mixer \\
\textbf{LLM Weights} & \textbf{Frozen} (No gradient) & \textbf{Fine-tuned} (Gradient updates) \\
\midrule
\textbf{Equilibrium} & \multicolumn{2}{c}{Targets HQRE (Sec.~\ref{subsec:hqre}) entropy-regularized fixed point} \\
\midrule
\textbf{Updates} & Logit best response on $\mathcal{U}_{i,t}$ & Clipped KL-prox mirror update \\
\textbf{Deployment} & Lightweight inference overlay & Standalone fine-tuned model \\
\bottomrule
\end{tabularx}
\end{table}

\subsection{Shared components: beliefs, critics, and value aggregation}
\label{subsec:dice_template}

Both algorithms share three components: (i)~a belief encoder that
compresses histories, (ii)~a monotone mixer that aggregates per-agent
values, and (iii)~a soft baseline for variance reduction.
We define each in turn.

\paragraph{Beliefs and public context.}
Agents in the IIMG of Sec.~\ref{subsec:notation} condition on
information states $h_i^t$ that include private observations and the
public stream.
Because $h_i^t$ grows with the interaction horizon, we compress it into
a fixed-dimensional belief embedding
$b_i^t = B_i(h_i^t;\theta^B)\in\mathbb{R}^{d_B}$
via a learned encoder that consumes the agent's private history
(including the public stream text $\widetilde y^{\le t}$).
Architectural details are given in Appendix~\ref{app:belief_encoder_detail}.

In addition, we compute a public context feature
$z^t = Z(\widetilde y^{\le t})$ that is measurable with respect to
$\sigma(\{\widetilde y^u\}_{u\le t})$.
The feature $z^t$ is used only by centralized training components
(critics and mixers); at deployment, agents require only $b_i^t$ and
$\widetilde y^{\le t}$.

\paragraph{Monotone mixing.}
To connect to the Local--Global theory in
Sec.~\ref{subsec:hier_theory}, we maintain per-agent action-values and
aggregate them through a monotone mixer,
\begin{equation}
\label{eq:dice_mono_mixer}
Q_{\mathrm{tot}}(z^t,\mathbf a^t,t)
= f_{\mathrm{mix}}\!\Big(z^t,\
  \{Q_i(b_i^t,a_i^t,t)\}_{i=1}^N;\ \phi_{\mathrm{mix}}\Big),
\qquad
\frac{\partial f_{\mathrm{mix}}}{\partial Q_i}\ \ge\ 0
\ \ \forall i,
\end{equation}
which is the flat ($K{=}1$) instance of the hierarchical aggregation in
Eq.~\eqref{eq:hier_global}.

\paragraph{HQRE-shaped policies and soft baselines.}
Given action-values and a temperature schedule $\alpha>0$, the HQRE
logit best response has the form

\[
\pi_i^{\mathrm{soft}}(a_i\mid b_i^t,t)\ \propto\
\exp\!\Big(Q_i(b_i^t,a_i,t)/\alpha_{i,t}\Big),
\]
matching Eq.~\eqref{eq:logit_def} with the information state $h_i^t$
replaced by its belief embedding $b_i^t$.

When the action set is finite, the corresponding soft baseline at a
belief state is
\[
V_i^{\mathrm{soft}}(b_i^t,t)
:= \alpha_{i,t}\log\sum_{a_i\in\mathcal A_i}
\exp\!\Big(\tfrac{1}{\alpha_{i,t}}Q_i(b_i^t,a_i,t)\Big),
\]
and the entropy-regularized advantage for a sampled action is
$Q_i - V_i^{\mathrm{soft}}$.
DICE-PC uses this baseline exactly on a finite candidate set, while
DICE-FT learns a soft baseline head as a tractable approximation of the
same log-sum-exp normalizer over the token vocabulary.

\paragraph{KL-prox mirror updates in practice (DICE-FT).}
For DICE-FT, the idealized policy-space iteration in
Theorem~\ref{thm:hqre_unique_mono} is the explicit mirror step
Eq.~\eqref{eq:hqre_explicit_mirror_step} driven by the exact field
$F_{\mathrm{reg}}(\boldsymbol\pi^k)$.
In practice we replace this field with a clipped stochastic surrogate
computed from token log-probabilities and enforce a KL trust region to a
frozen snapshot policy $\pi^{\mathrm{old}}$ (the current policy frozen at
the start of each update iteration).
This makes the update implementable with standard deep learning
primitives: rollouts store token log-probabilities under
$\pi^{\mathrm{old}}$, the current policy computes new log-probabilities on
the same tokens, importance ratios and KL are computed
directly, and early stopping enforces a per-update KL target.

\subsection{DICE-PC: prompt-control coordination}
\label{subsec:dice_pc}

DICE-PC coordinates frozen execution LLMs by learning distributions
over prompt-control actions rather than exchanging intermediate
deliberation trajectories.

At each environment step~$t$, the Coordinator publishes a
context~$m^t$ and each execution agent generates a one-shot candidate;
the Coordinator then aggregates candidates into the public
outcome~$y^t$ via a benchmark-dependent rule (best-of-$N$ for reasoning, majority vote for planning, critic scoring for active reasoning, concatenation for interactive tasks).
Each execution agent selects a bounded prompt-control action
$u_i^t\in\mathcal{U}$, parameterizing decode-time behavior
of a frozen LLM (temperature, nucleus threshold, repetition penalty, tool-access flag).
Since $\mathcal U$ is high-dimensional, DICE-PC performs policy
improvement on a small finite candidate set
$\mathcal U_{i,t}\subset\mathcal U$ induced by a controller
$\mu_i(b_i^t;\theta_i^u)$ with quantization and anchors,
preserving the logit best-response structure of
Eq.~\eqref{eq:logit_def}.

\paragraph{Sampling and rewards.}
Given $b_i^t$ and $\mathcal U_{i,t}$, the agent evaluates
$Q_i(b_i^t,u,t)$ for $u\in\mathcal U_{i,t}$, forms the HQRE logit
best response

\[
\pi_i(u\mid b_i^t,t)\ :=\
\frac{\exp\!\big(Q_i(b_i^t,u,t)/\alpha_{i,t}\big)}
     {\sum_{u'\in\mathcal U_{i,t}}
      \exp\!\big(Q_i(b_i^t,u',t)/\alpha_{i,t}\big)},
\qquad u\in\mathcal U_{i,t},
\]
samples $u_i^t\sim \pi_i(\cdot\mid b_i^t,t)$, and decodes a one-shot
candidate with the frozen LLM under~$u_i^t$.
Rewards are bounded per-agent scalars $r_i^t$ with
$|r_i^t|\le R_{\max}$, consistent with assumption~(A2); their
task-specific decomposition is described in
Sec.~\ref{sec:experiments}.


\paragraph{Critic updates and policy improvement.}
When learning is enabled, we update critics and the mixer by regression
on replayed transitions.
For one-turn tasks ($H{=}1$), we regress $Q_i(b_i^0,u_i^0,0)$ to the
observed return~$r_i^0$ (with optional normalization); for multi-step
tasks we use a discounted return-to-go or TD target.
Policy improvement is then implicit: at each visited belief state,
sampling uses the updated logit best response
$\pi_i(\cdot\mid b_i^t,t)$ induced by $Q_i$ and $\alpha_{i,t}$, which
directly instantiates the regularized best-response map
$\mathcal{B}_i^{\boldsymbol\alpha}$ of Eq.~\eqref{eq:logit_def}.

We keep $\alpha_{i,t}$ bounded away from zero to ensure well-posed
selection; the zero-temperature limit (Lemma~\ref{lem:zero_temp_full})
remains a theoretical connection to the unregularized game, not an
algorithmic target.
Algorithm~\ref{alg:dice_pc} summarizes the procedure.

\begin{algorithm}[t]
\caption{DICE-PC: prompt-control coordination via HQRE logit
best response}

\label{alg:dice_pc}
\begin{algorithmic}[1]
\REQUIRE Frozen execution LLMs $\{\mathrm{LLM}_i\}$; Coordinator;
replay buffer $\mathcal D$; encoder / controller / critic / mixer
parameters
\FOR{episode / query}
  \FOR{$t=0,1,\ldots,H-1$}
    \STATE Coordinator produces context $m^t$; broadcast to all agents
    \FOR{each execution agent $i$ (in parallel)}
      \STATE Update $h_i^t$; compute belief
        $b_i^t=B_i(h_i^t)$; form candidate set
        $\mathcal U_{i,t}$ around $\mu_i(b_i^t)$
      \STATE Evaluate $Q_i(b_i^t,u,t)$ for
        $u\in\mathcal U_{i,t}$; sample
        $u_i^t\sim \pi_i(\cdot\mid b_i^t,t)$
      \STATE Decode frozen $\mathrm{LLM}_i$ under control
        $u_i^t$ to obtain candidate $z_i^t$
    \ENDFOR
    \STATE Coordinator aggregates candidates into the outcome
      $y^t$; public stream token is
      $\widetilde y^t=(m^t,y^t)$
    \STATE Compute rewards $\{r_i^t\}$; store
      $(h^t,\mathbf u^t,\widetilde y^t,\mathbf r^t)$ in
      $\mathcal D$
    \IF{learning enabled}
      \STATE Update critics $\{Q_i\}$ and mixer
        $f_{\mathrm{mix}}$ by regression on minibatches
        from $\mathcal D$
      \STATE Update belief / controller encoders by
        backprop through the critic losses
    \ENDIF
  \ENDFOR
\ENDFOR
\end{algorithmic}
\end{algorithm}

\subsection{DICE-FT: HQRE-oriented mirror fine-tuning}
\label{subsec:dice_ft}

While DICE-PC coordinates frozen models through prompt control, DICE-FT
directly adjusts the models' token-level generation probabilities.
It does so using clipped KL-prox mirror updates toward the HQRE fixed
point.
The remainder of this subsection describes four components in order:
(1)~the heterogeneous temperature instantiation,
(2)~critic targets and advantage computation,
(3)~the clipped actor loss with entropy regularization, and
(4)~KL diagnostics for stability monitoring.
Readers primarily interested in experimental results may skip to
Section~\ref{sec:experiments} after reading the summary box
at the end of this section.

Each environment step~$t$ corresponds to one agent output (e.g., a
complete response).
Within that step, the agent makes $L_i$ sequential token decisions
indexed by $\ell=0,\ldots,L_i-1$.
The HQRE temperature $\alpha_{i,t,\ell}$, extending the
environment-step temperature $\alpha_{i,t}$ to a per-token granularity,
applies at each token position within each step, consistent with the
IIMG modeling in Sec.~\ref{subsec:notation}.

\paragraph{Trainable components and stored quantities.}
For each agent~$i$ we train an actor
$\pi_{\theta_i}(\cdot\mid b_i^t,\ell)$ (implemented with
parameter-efficient adapters such as LoRA), a token-action critic
$Q_{\psi_i}(b_i^t,x,\ell)$ evaluated only on sampled tokens, and a
soft baseline head
$V_{\nu_i}^{\mathrm{soft}}(z^t,b_i^t,\ell)$.
We also train a shared monotone mixer as a nonnegative linear map,
\begin{equation}
\label{eq:diceft_linear_mixer}
Q_{\mathrm{tot}}(z^t,\mathbf x_{\ell+1},\ell)
=\sum_{i=1}^N w_i^{\mathrm{mix}}(z^t)\,
Q_i(b_i^t,x_{i,\ell+1},\ell)+b^{\mathrm{mix}}(z^t),
\qquad w_i^{\mathrm{mix}}(z)\ge 0,
\end{equation}
which matches the flat instance of the hierarchical aggregation in
Eq.~\eqref{eq:hier_global}.

At each update iteration, we roll out a frozen snapshot policy
$\pi^{\mathrm{old}}$ and store, for every sampled token, its
log-probability under $\pi^{\mathrm{old}}$, under the current policy,
and under a fixed reference policy $\pi_{\mathrm{ref}}$.
These stored log-probabilities make the importance ratios, entropy
terms, and KL diagnostics directly computable.

\paragraph{Heterogeneous temperatures.}
We instantiate the agent- and state-dependent temperature $\alpha_{i,t}(h_i^t)$ of Eq.~\eqref{eq:J_i_def} as a two-level position-dependent schedule $\alpha_{i,t,\ell} \in \{\alpha_{\mathrm{lo}}, \alpha_{\mathrm{hi}}\}$.
The level switches based on whether the predictive entropy of a frozen reference model at the current prefix exceeds a threshold~$h_0$.
This schedule uses only information available at generation time; the minimum temperature $\alpha_{\min}=\alpha_{\mathrm{lo}}>0$ satisfies the theory's requirement.
Full specification of the threshold and two-level values is in Appendix~\ref{app:error_diagnostics_detail}.

\paragraph{Training objective overview.}
The training objective combines three components in a compositional pipeline:
(1)~critic regression losses that estimate per-token action values from rollout returns;
(2)~advantages that weight these estimates by the mixer coefficients and add a group-relative control variate for variance reduction; and
(3)~a clipped actor loss that performs the approximate KL-mirror step using the computed advantages and an explicit entropy bonus matching the HQRE temperature.
Each component feeds into the next, and we describe them in order below.

For a sampled rollout~$g$ with agent output length $L_i^{(g)}$, we
define a per-token scalar regression target by distributing the
terminal return across tokens,
\[
G_{i,\ell}^{(g)} := \frac{r_i^{(g)}}{\max\{L_i^{(g)},1\}},
\qquad \ell=0,\ldots,L_i^{(g)}-1,
\]
and mask tokens beyond a maximum generation budget to avoid unstable
gradients from truncated outputs.
This uniform allocation is intentionally simple: it avoids the need for
per-token reward models but may under-weight early tokens that determine
the solution structure.
We leave token-level credit assignment to future work.

\paragraph{Component 1: Critic regression losses.}
We train per-agent critics $Q_{\psi_i}$ and soft baselines $V_{\nu_i}^{\mathrm{soft}}$ by standard squared-error regression on the per-token return targets $G_{i,\ell}^{(g)}$; the mixer is trained jointly through its dependence on the same targets.
The explicit loss equations are given in Appendix~\ref{app:critic_loss_detail}.
\paragraph{Component 2: Advantage computation.}
The weighted entropy-regularized advantage for each sampled token is

\begin{equation}
\label{eq:diceft_weighted_adv}
A_i^{(w)}(z^t,b_i^t,x_{i,\ell+1},\ell)
:= w_i^{\mathrm{mix}}(z^t)\Big(
  Q_{\psi_i}(b_i^t,x_{i,\ell+1},\ell)
  -V_{\nu_i}^{\mathrm{soft}}(z^t,b_i^t,\ell)\Big).
\end{equation}

\paragraph{Group-relative control variates.}
We augment advantages with a group-relative control variate ($\lambda_{\mathrm{grp}}{=}0.5$) computed from $G$ rollouts via leave-one-out standardization:
\begin{equation}
\label{eq:diceft_adv_combined}
\widetilde A_{i}^{(g,\ell)} \ :=\
A_i^{(w)}(z^{(g)},b_i^{(g)},x_{i,\ell+1}^{(g)},\ell)
\ +\ \lambda_{\mathrm{grp}}\,\widehat A_{\mathrm{grp}}^{(g)}.
\end{equation}
The critic-based advantage in Eq.~\eqref{eq:diceft_weighted_adv} remains the primary signal; the group-relative term serves only for variance reduction (full derivation in Appendix~\ref{app:error_diagnostics_detail}).

\paragraph{Component 3: Clipped actor loss.}
The actor update follows a PPO-style clipped surrogate, augmented with the
HQRE entropy term.
Let $\rho_{i}^{(g,\ell)}$ be the importance ratio between the current
policy and the frozen snapshot policy at the sampled token,
\[
\rho_{i}^{(g,\ell)}
=\exp\!\Big(\log \pi_{\theta_i}(x_{i,\ell+1}^{(g)}\mid b_i^{(g)},\ell)
-\log \pi^{\mathrm{old}}_{\theta_i}(x_{i,\ell+1}^{(g)}
  \mid b_i^{(g)},\ell)\Big).
\]
We update each actor by maximizing the clipped near-prox surrogate with
an explicit entropy weight matching the HQRE temperature,

\begin{align}
\mathcal{L}_{i}^{\mathrm{actor}}(\theta_i)
&= \mathbb{E}_{g,\ell}\Big[
\min\big(\rho_{i}^{(g,\ell)}\,\widetilde A_{i}^{(g,\ell)},\
\mathrm{clip}(\rho_{i}^{(g,\ell)},
  1\!-\!\varepsilon_\ell,1\!+\!\varepsilon_\ell)\,
  \widetilde A_{i}^{(g,\ell)}\big)
\Big]
\nonumber
\\
& + \mathbb{E}_{g,\ell}\Big[
  \alpha_{i,t,\ell}\,
  \mathcal H(\pi_{\theta_i}(\cdot\mid b_i^{(g)},\ell))\Big],
\label{eq:diceft_actor_loss}
\end{align}
where $\varepsilon_\ell$ can be position-dependent (in our experiments,
$\varepsilon_\ell = 0.2$ for all positions).
\paragraph{KL diagnostics as stability proxies.}
We control the approximation error through
two KL diagnostics:
$\widehat{\mathrm{KL}}_{\mathrm{old}}$ (per-update snapshot KL triggering early stopping) and
$\widehat{\mathrm{KL}}_{\mathrm{ref}}$ (cumulative drift from the pretrained model, with rejected updates exceeding budget $\beta_{\mathrm{KL}}$).
In experiments, $\widehat{\mathrm{KL}}_{\mathrm{old}} > 0.05$ reliably signals overshooting and $\widehat{\mathrm{KL}}_{\mathrm{ref}} > 0.3$ signals drift risking exit from the stable convergence regime; we recommend these as conservative thresholds.
These serve as empirical proxies for the monotonicity margin $\alpha_{\min}-L_c$ (Theorem~\ref{thm:hqre_unique_mono}).

\paragraph{Deployment.}
At deployment, the centralized critic and mixer are discarded, and each
agent runs its fine-tuned policy using only its private history and the
public stream.
DICE-PC can additionally be applied as a lightweight inference overlay
on top of DICE-FT to further stabilize selection without modifying
fine-tuned weights.

Algorithm~\ref{alg:dice_ft} summarizes the training procedure.

\begin{algorithm}[t]
\caption{DICE-FT: token-level KL-prox mirror fine-tuning with a
monotone critic}

\label{alg:dice_ft}
\begin{algorithmic}[1]
\REQUIRE Training prompts; reference policy $\pi_{\mathrm{ref}}$;
actor / belief / critic / mixer parameters; group size $G$;
schedules $\{\alpha_{i,t,\ell},\varepsilon_\ell\}$; KL targets and
budgets
\FOR{update iteration $k=1,2,\ldots$}
  \STATE Freeze current actors as $\pi^{\mathrm{old}}$
  \STATE Sample a minibatch of prompts; for each prompt sample $G$
    joint rollouts under $\pi^{\mathrm{old}}$
  \STATE For each rollout, store sampled tokens and
    log-probabilities under $\pi^{\mathrm{old}}$ and
    $\pi_{\mathrm{ref}}$; compute rewards and group returns
  \STATE Compute $\widehat A_{\mathrm{grp}}$ by leave-one-out
    standardization; build per-token targets $G_{i,\ell}$ with
    masking
  \STATE \textit{// Critic update:} Update critics
    $\{Q_{\psi_i},V_{\nu_i}^{\mathrm{soft}}\}$ and mixer by
    minimizing $\sum_i(\mathcal L_{Q_i}+\mathcal L_{V_i})$
  \STATE \textit{// Advantage computation:} Compute per-token advantages
    $\widetilde A_{i}^{(g,\ell)}$ via
    Eq.~\eqref{eq:diceft_adv_combined} and temperatures
    $\alpha_{i,t,\ell}$
  \STATE \textit{// Actor update with early stopping:} Update actors by maximizing
    Eq.~\eqref{eq:diceft_actor_loss} with early stopping when
    $\widehat{\mathrm{KL}}_{\mathrm{old}}$ exceeds target
  \STATE \textit{// KL safeguard:} If
    $\widehat{\mathrm{KL}}_{\mathrm{ref}}$ exceeds budget, reject
    the update and optionally increase $G$ for high-variance prompts
    in the next iteration
\ENDFOR
\end{algorithmic}
\end{algorithm}

\keybox{\textbf{Algorithm summary.}
DICE-PC learns distributions over bounded prompt-control actions to coordinate frozen LLMs, while DICE-FT fine-tunes token-level policies via clipped KL-prox mirror updates.
Both target the HQRE fixed point using monotone critics, heterogeneous temperatures, and KL diagnostics that serve as practical stability proxies.}

\section{Experiments}
\label{sec:experiments}

This section evaluates DICE-PC and DICE-FT against single-model, debate-style, and ensembling baselines.
We organize experiments to mirror the causal chain underlying DICE:
mechanism validation of HQRE-based equilibrium selection and its deployable diagnostics (Sections~\ref{subsec:mechanism_validation}--\ref{subsec:convergence_diagnostics}),
end-task evaluation across eleven benchmarks in four domains under bounded public-stream budgets (Sections~\ref{subsec:endtask_reasoning_planning}--\ref{subsec:endtask_active_coord}),
and communication efficiency and sensitivity analysis (Sections~\ref{subsec:efficiency_summary}--\ref{subsec:ablations_sensitivity}).
Table~\ref{tab:exp_claim_map} maps each claim to its supporting experiment.

\subsection{Overview and Experimental Protocol}
\label{subsec:exp_setup}

All experiments share a common protocol to ensure fair comparison.
Unless explicitly noted, each DICE instance uses one Coordinator and
three Executors under the coordinator-mediated protocol described in
Section~\ref{sec:algorithm}.
The Coordinator broadcasts a single global strategy message per round
(capped at 70 tokens), Executors generate one-shot candidates
independently, and the Coordinator aggregates candidates into the public
outcome.
This design fixes the public-stream footprint per step.

We evaluate two deployment variants throughout.
\textit{DICE-PC} keeps Executor LLM weights frozen and learns only a lightweight decision layer comprising belief encoders, value networks, mixers, and prompt-control heads using HQRE-shaped mirror updates.
\textit{DICE-FT} additionally fine-tunes Executors via parameter-efficient adapters under the same HQRE mirror objective.
Both variants employ the heterogeneous temperature schedule and differentiated training constraints described in Section~\ref{subsec:dice_ft}.

We evaluate on four domains:
(1)~reasoning (AIME24/25, MATH-500~\citep{hendrycks2021math}, ZebraLogic~\citep{lin2025zebralogic}, AutoLogic~\citep{zhu2025autologi});
(2)~constrained planning (PlanBench~\citep{valmeekam2023planbench}, TravelPlanner);
(3)~active reasoning requiring iterative information gathering (ARBench)~\citep{zhou2025active}; and
(4)~multi-agent coordination (\textsc{LLM-Coordination-Bench}~\citep{agashe2023evaluating}).
This selection covers both one-step episodes and multi-turn interactive settings where discounting and incomplete information play central roles.
The per-agent reward $r_i^t$ is computed from benchmark outputs as
follows: AIME and MATH-500 use binary correctness (1 if the final
answer matches the ground truth, 0 otherwise); ZebraLogic and AutoLogic
use binary correctness of the logical assignment; PlanBench uses binary
plan validity; TravelPlanner uses the fraction of hard constraints
satisfied; ARBench uses the benchmark's native scoring function;
Overcooked-AI uses the environment's shaped reward; and Hanabi uses the
game score normalized to $[0,1]$.
Our primary configurations use Qwen3-4B and Qwen3-8B as Executors,
with Qwen3-32B and GPT-4-Turbo as scale references.%
\footnote{We append ``-Inst'' for instruct-tuned variants and ``-Base''
for base (unaligned) checkpoints. ``Qwen3-A3B'' denotes the 30B MoE
variant with 3B active parameters.}
We use open-source model families (Qwen, LLaMA) as primary backbones and include closed-source references where public numbers are available, marking them clearly in all tables.
Because closed-source models use different evaluation pipelines (prompt formats, sampling strategies, tool access), their entries serve as contextual reference points rather than controlled comparisons.

Baselines fall into three categories: prompt-only inference methods (zero-/few-shot CoT, Self-Consistency, ToT, debate variants), single-agent fine-tuning (PPO/GRPO/DAPO-style updates), and multi-agent fine-tuning under matched parameter budgets.
We group results by method class in all tables to enable fair within-class comparisons that isolate coordination effects from backbone differences.

\subsubsection{Training Details}
DICE-PC trains belief encoders, value networks, mixers, and
prompt-control heads using Adam (learning rate $3\times10^{-4}$,
$\beta_1{=}0.9$, $\beta_2{=}0.999$) with a replay buffer of 10K
transitions and minibatch size~256.
DICE-FT uses the same optimizer with learning rate $1\times10^{-5}$
for LoRA adapters (rank~16, $\alpha_{\mathrm{LoRA}}{=}32$) and
$3\times10^{-4}$ for critic and mixer heads.
The default group size is $G{=}8$ rollouts per prompt, the clip
parameter is $\varepsilon_\ell{=}0.2$, and the heterogeneous
temperature schedule uses $\alpha_{\mathrm{hi}}{=}0.1$,
$\alpha_{\mathrm{lo}}{=}0.01$, with entropy threshold $h_0{=}1.5$
nats.
Training runs for 200 iterations on reasoning benchmarks and 500
iterations on interactive benchmarks, with ABR-based early stopping
(the Average Best-Response gap, defined in
Section~\ref{subsec:abr_stopping}; threshold~0.1, used uniformly across all benchmarks without per-task tuning).
All experiments were conducted on 4$\times$A100-80GB GPUs.
Approximate wall-clock training times (GPU-hours) by benchmark class:
\textsc{Text-Hanabi} (decision layer only), ${\sim}2$;
reasoning benchmarks (AIME, MATH-500, ZebraLogic, AutoLogic), ${\sim}6$ for DICE-PC and ${\sim}18$ for DICE-FT;
planning (PlanBench, TravelPlanner), ${\sim}8$ for DICE-PC and ${\sim}22$ for DICE-FT;
active reasoning (ARBench), ${\sim}10$ for DICE-PC;
interactive (Overcooked-AI, Hanabi, CollabEscape/Capture), ${\sim}12$ for DICE-PC and ${\sim}24$ for DICE-FT.

\subsubsection{Monotonicity Verification at Scale}
The theoretical guarantees of Theorem~\ref{thm:hqre_unique_mono} require
the monotonicity condition $\alpha_{\min}>L_c$.
We verify this condition directly only in the \textsc{Text-Hanabi}
experiments (Section~\ref{subsec:text_hanabi_margin}), where the
lightweight decision layer makes online estimation of $L_c$ tractable.
For the large-scale reasoning and planning benchmarks (Tables~\ref{tab:overall}--\ref{tab:travel-planner-results}), estimating $L_c$ would require
computing second-order cross-agent gradient information through the full
LLM, which is computationally prohibitive at current scale.
Instead, we rely on the KL diagnostics
$\widehat{\mathrm{KL}}_{\mathrm{old}}$ and
$\widehat{\mathrm{KL}}_{\mathrm{ref}}$ as indirect stability proxies:
controlled KL drift is necessary but not sufficient for the
monotonicity condition to hold, so the diagnostics provide evidence of
stability without formally certifying uniqueness.
Developing scalable, data-driven estimators for $L_c$ at LLM scale
is an important open problem (Section~\ref{sec:conclusion}).

For reproducibility, we report both performance and communication cost across all tasks.
Token accounting includes total generated tokens from all agents, message rounds per episode, and per-turn token rates where applicable.
Training uses standard data sets with contamination audits to ensure no evaluation instances appear in training data.
We report pass@$K$ and avg@$K$ metrics with benchmark-specific sampling counts, averaging over five random seeds unless stated otherwise.

For interactive coordination environments with native multi-agent interfaces, we follow the benchmark's turn-taking or simultaneous-action protocol and treat the environment-visible trajectory log as the public stream.
\textsc{Text-Hanabi} is used for mechanism diagnostics under a two-agent interface; its details are specified in Section~\ref{subsec:text_hanabi_setup}.

\subsubsection{Reproducibility}
All hyperparameters are specified above (optimizer, learning rates, LoRA
rank and $\alpha$, group size $G$, clip parameter $\varepsilon$,
temperature schedule, and early-stopping threshold).
We report means over 5 random seeds (seeds 0--4) for all DICE entries and include
standard deviations where they exceed 0.5\%.
Each seed controls the random number generator for decision-layer weight initialization (or LoRA adapter initialization for DICE-FT), training data shuffling, and rollout sampling.
Anonymous code, configuration files, and evaluation scripts are included in the supplementary material for review.

\begin{table}[t!]
\centering
\caption{Experiment-to-claim map.  Each theoretical claim is paired with the section and key evidence that supports it.}
\label{tab:exp_claim_map}
\small
\setlength\tabcolsep{4pt}
\begin{tabular}{@{}p{4.4cm}p{2.6cm}p{6.0cm}@{}}
\toprule
\textbf{Claim} & \textbf{Section} & \textbf{Key evidence} \\
\midrule
HQRE temperature selects among multiple equilibria
  & \ref{subsec:toy_2x2}
  & Fig.~\ref{fig:panel_2x2} \\
\addlinespace[2pt]
Uniqueness margins predict stability under incomplete information
  & \ref{subsec:text_hanabi_margin}
  & Fig.~\ref{fig:uniqueness_convergence} \\
\addlinespace[2pt]
Debate locks into defensive mixing with linear regret
  & \ref{subsec:text_hanabi_margin}
  & Fig.~\ref{fig:defensive_mixing}(a,\,c) \\
\addlinespace[2pt]
HQRE achieves bounded Bayesian regret
  & \ref{subsec:text_hanabi_margin}
  & Fig.~\ref{fig:defensive_mixing}(b,\,c) \\
\addlinespace[2pt]
Practical updates follow KL--mirror geometry
  & \ref{subsec:convergence_diagnostics}
  & Fig.~\ref{fig:mirror_geometry} \\
\addlinespace[2pt]
ABR provides a reliable stopping criterion
  & \ref{subsec:convergence_diagnostics}
  & Fig.~\ref{fig:abr_bound} \\
\addlinespace[2pt]
Hierarchical decomposition improves scalability margins
  & \ref{subsec:hier_text_hanabi}
  & Fig.~\ref{fig:hierarchical_scaling} \\
\addlinespace[2pt]
End-task gains under bounded public streams
  & \ref{subsec:endtask_reasoning_planning}--\ref{subsec:endtask_active_coord}
  & Tables~\ref{tab:overall},\,\ref{tab:travel-planner-results}; Fig.~\ref{fig:arbench}; Tables~\ref{tab:overcooked_summary}--\ref{tab:hanabi_summary} \\
\addlinespace[2pt]
Communication efficiency under deployment budgets
  & \ref{subsec:efficiency_summary}
  & Tables~\ref{tab:overcooked_summary},\,\ref{tab:collab_summary} \\
\bottomrule
\end{tabular}
\end{table}

\subsection{Mechanism Validation: Equilibrium Selection}
\label{subsec:mechanism_validation}

We validate HQRE-based equilibrium selection in a closed-form coordination game and then test whether theory-derived uniqueness margins remain predictive under partial observability, long horizons, and function approximation.

\subsubsection{Closed-form verification in a minimal coordination game}
\label{subsec:toy_2x2}

Consider a symmetric $2\times2$ coordination game parameterized by
$\varepsilon\in(0,1)$, with payoffs $u(L,L)=1$,
$u(R,R)=1-\varepsilon$, and $u(L,R)=u(R,L)=0$.
We set $\varepsilon=0.3$ (so $u(R,R)=0.7$).
The game admits two pure Nash equilibria and one unstable mixed equilibrium.
Under symmetric HQRE, the logit response becomes a contraction when $\alpha$ exceeds a conservative sufficient bound of $(2-\varepsilon)/4 \approx 0.425$, obtained from Theorem~\ref{thm:unique_hqre_joint}'s contraction criterion applied to this game.
We numerically compute all fixed points across temperatures to characterize the actual phase transition.

\begin{figure}[t!]
\centering
\includegraphics[width=\textwidth]{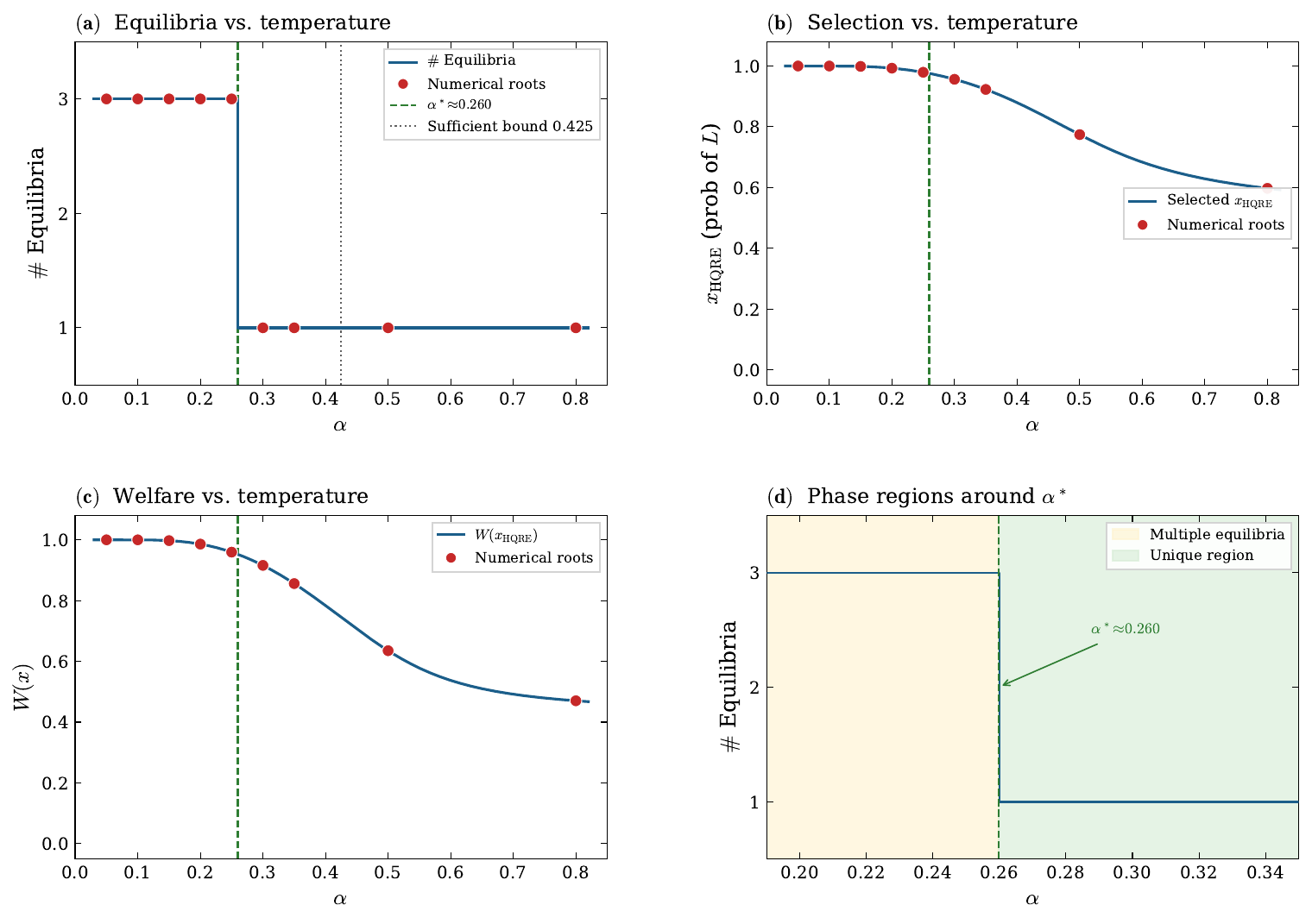}
\caption{HQRE in a $2\times 2$ coordination game ($\varepsilon = 0.3$).
(a)~Number of equilibria versus temperature~$\alpha$ shows a phase change from three equilibria to a unique equilibrium; the empirical boundary occurs at $\alpha^* \approx 0.261$ (green dashed), tighter than the sufficient bound $0.425$ (gray dotted).
(b)~Selected symmetric equilibrium $x_{\text{HQRE}}$ decreases with $\alpha$ due to increased mixing.
(c)~Welfare $W(x)$ declines correspondingly as entropy rises.
(d)~Phase diagram marking multiple-equilibria (yellow) and unique (green) regions around $\alpha^*$.}
\label{fig:panel_2x2}
\vspace{-14pt}
\end{figure}

Figure~\ref{fig:panel_2x2} confirms a sharp transition from three equilibria to a unique equilibrium at $\alpha^*\approx 0.261$, well below the conservative bound.
As $\alpha$ increases, the selected HQRE moves toward greater mixing and welfare decreases as entropy rises, illustrating the regularization--performance trade-off.
In the opposite limit $\alpha\to 0$, HQRE strategies approach pure Nash equilibria, consistent with Lemma~\ref{lem:zero_temp_full}.
The closed-form analysis confirms that HQRE temperature controls equilibrium multiplicity in principle; we next test whether these predictions persist under partial observability, long horizons, and function approximation.

\subsubsection{\textsc{Text-Hanabi} setup}
\label{subsec:text_hanabi_setup}

\textsc{Text-Hanabi} is a two-agent Dec-POMDP with $\gamma=0.99$ and horizon $T=200$.
Both agents are frozen \textsc{Qwen3-4B} language models.
To isolate equilibrium-selection effects from backbone changes, learning is restricted to a lightweight decision layer (approximately 12M parameters per agent) that selects high-level actions and bounded decoding controls.
The decision layer consists of a two-layer Transformer belief encoder
(4 heads, hidden dimension~256), an MLP value head (two hidden layers
of 512 units), and a linear action controller, together with the belief
and value components used by the HQRE updates.
We follow the benchmark's native interaction protocol and treat the environment-visible trajectory as the public stream for the IIMG abstraction.
Token accounting, temperature schedules, and stopping rules follow the default experimental setup described in Section~\ref{subsec:exp_setup}.

\subsubsection{Uniqueness margins in \textsc{Text-Hanabi}}
\label{subsec:text_hanabi_margin}

We test whether the uniqueness margin from Theorem~\ref{thm:hqre_unique_mono} remains predictive in a structured cooperative game with partial observability and long horizons.
To isolate equilibrium-selection effects from backbone changes, both agents are frozen and learning is restricted to the decision layer and critics.
We compare three learning dynamics: unregularized best-response learning, debate-style message passing (MAD protocol with three rounds per step), and HQRE-based coordination (both homogeneous and heterogeneous temperatures).

Our central diagnostic is the estimated uniqueness margin $\mu_k = \alpha_{\min,k} - \widehat{L}_{c,k}$ (Section~\ref{subsec:hqre}), where $\widehat{L}_{c,k}$ is an online estimate of the coupling constant.
When $\mu_k>0$, Theorem~\ref{thm:hqre_unique_mono} predicts a unique HQRE and linear-rate convergence of KL--mirror updates.

\begin{figure}[t!]
\centering
\includegraphics[width=\textwidth]{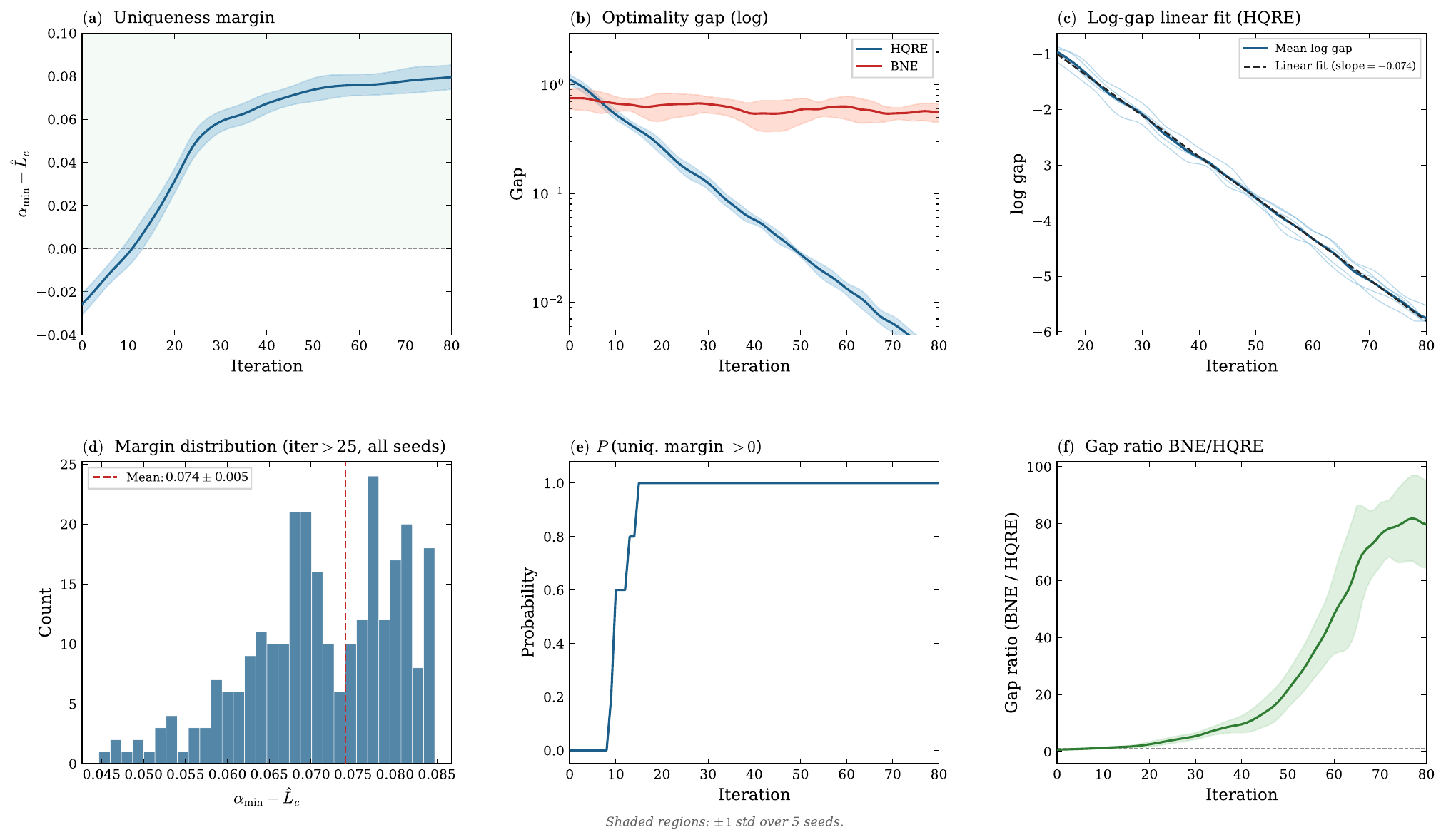}
\caption{
Uniqueness margin and linear-rate convergence in \textsc{Text-Hanabi}.
\textbf{(a)} The uniqueness margin $\alpha_{\min} - \widehat{L}_c$ becomes positive and stabilizes.
\textbf{(b)} Optimality gap: HQRE decays exponentially, whereas BNE plateaus.
\textbf{(c)} A linear fit in log-space is consistent with the predicted linear convergence regime of explicit KL-mirror theory.
\textbf{(d)} Distribution of the margin after convergence.
\textbf{(e)} The probability of a positive margin approaches 1.0.
\textbf{(f)} The gap ratio exceeds $20\times$ by iteration~80.
Shaded regions indicate $\pm 1$ std over 5 seeds.
}
\label{fig:uniqueness_convergence}
\vspace{-20pt}
\end{figure}

Figure~\ref{fig:uniqueness_convergence} shows that the margin becomes positive after roughly 20~iterations and stabilizes at $0.074\pm0.005$, providing an explicit stability certificate for the learned policy.
Unregularized best-response learning plateaus with oscillatory behavior characteristic of equilibrium switching (Proposition~\ref{prop:multi_linear}).
In contrast, HQRE-based learning exhibits exponential decay of the optimality-gap proxy.
By iteration~80, the gap ratio exceeds $20\times$ and the probability of maintaining a positive margin reaches~1.0 across seeds.

\subsubsection{Debate Dynamics and Regret Accumulation}

The uniqueness-margin analysis explains \emph{why} unregularized learning fails (equilibrium multiplicity), but the theoretical framework also predicts a distinct failure mode for debate: persistent defensive mixing that maintains a positive welfare-gap floor even after convergence (Lemma~\ref{lem:debate_suboptimality_main}).
To test both predictions simultaneously and verify the bounded-regret guarantee of Corollary~\ref{cor:hqre_regret_bounded}, we measure three additional diagnostics in the same \textsc{Text-Hanabi} setup.

\begin{figure}[t!]
\centering
\includegraphics[width=\textwidth]{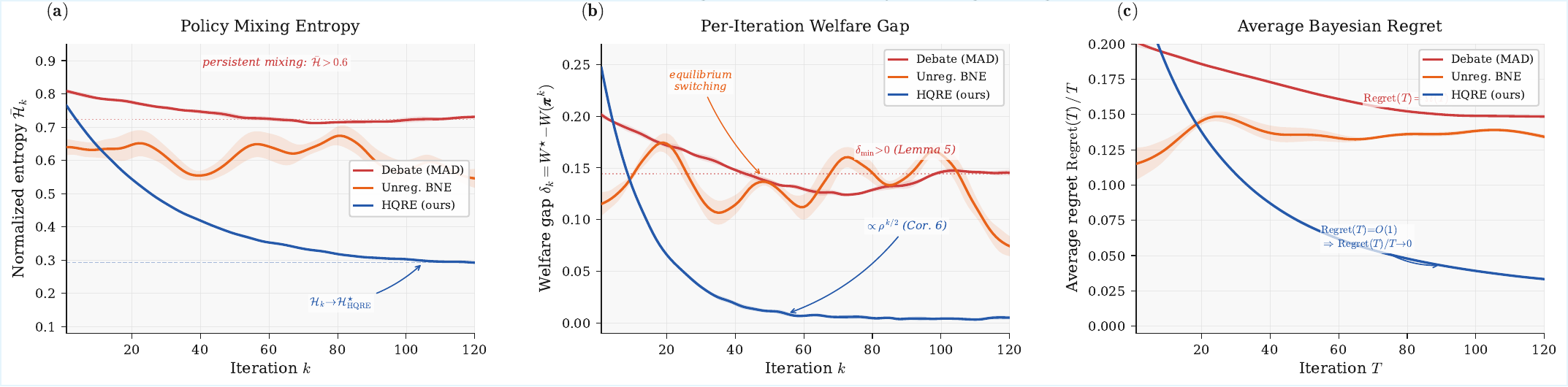}
\caption{Defensive mixing persistence and Bayesian regret diagnostics in \textsc{Text-Hanabi}.
(a)~Normalized policy mixing entropy $\bar{\mathcal{H}}_k$ averaged over visited information states.
Debate maintains $\bar{\mathcal{H}} > 0.6$ throughout training (persistent defensive mixing, consistent with Lemma~\ref{lem:debate_suboptimality_main}), while HQRE decays monotonically to the unique equilibrium entropy level $\mathcal{H}^\star_{\mathrm{HQRE}}$.
(b)~Per-iteration welfare gap $\delta_k = W^\star - W(\boldsymbol{\pi}^k)$.
Debate exhibits a positive floor $\delta_{\min} > 0$; unregularized BNE shows recurrent drift spikes from equilibrium switching; HQRE decays at rate $\rho^{k/2}$ consistent with Corollary~\ref{cor:hqre_regret_bounded}.
(c)~Average Bayesian regret $\mathrm{Regret}(T)/T$.
Debate and unregularized BNE converge to positive constants (linear cumulative regret), confirming Proposition~\ref{prop:multi_linear} and Corollary~\ref{cor:mad_linear}.
HQRE converges to zero (bounded cumulative regret, Corollary~\ref{cor:hqre_regret_bounded}).
Shaded regions: $\pm 1$ std over 5~seeds.}
\label{fig:defensive_mixing}
\vspace{-14pt}
\end{figure}

Figure~\ref{fig:defensive_mixing}(a) confirms that debate-style dynamics maintain elevated mixing entropy ($\bar{\mathcal{H}} > 0.6$) throughout training, consistent with persistent defensive mixing.
This stands in contrast to unregularized BNE, whose entropy oscillates due to equilibrium switching, and to HQRE, whose entropy decays monotonically toward the unique equilibrium level.
The distinction is operationally meaningful: debate agents \emph{converge} (to a mixed strategy), but to the wrong target, whereas BNE agents \emph{fail to converge} at all.

Figure~\ref{fig:defensive_mixing}(b--c) translates these mixing patterns into welfare and regret.
The debate welfare gap stabilizes at a positive floor $\delta_{\min} \approx 0.14$, yielding linear cumulative regret; unregularized BNE accumulates regret at a similar average rate despite its oscillatory trajectory.
Under HQRE, the welfare gap decays exponentially and the average Bayesian regret $\mathrm{Regret}(T)/T$ converges to zero, confirming the $O(1)$ cumulative regret bound of Corollary~\ref{cor:hqre_regret_bounded}.
Together with the uniqueness-margin results, these diagnostics provide empirical support for the three negative results (single-model capacity ceiling via Figure~\ref{fig:uniqueness_convergence}(a), debate defensive mixing via panel~(a), and multiplicity-induced drift via Figure~\ref{fig:uniqueness_convergence}) and the positive HQRE convergence and regret guarantees.

\subsection{Convergence Diagnostics and Deployment Criteria}
\label{subsec:convergence_diagnostics}

Beyond uniqueness, deployment requires diagnostics that are measurable without access to best responses and robust under function approximation.
We validate three practical questions under the \textsc{Text-Hanabi} setup: whether practical updates follow KL--mirror geometry, how to determine convergence for early stopping, and whether a Local--Global hierarchy improves stability margins as interaction structure becomes more complex.

\subsubsection{Mirror-geometry fidelity}
\label{subsec:mirror_fidelity}

For each iteration in \textsc{Text-Hanabi}, we compute the mirror-descent improvement lower bound from measured KL step sizes and per-token temperatures, and compare it to realized improvements from held-out rollouts.
If practical updates approximate the ideal KL--mirror step, realized improvements should match or exceed the theoretical lower bounds.

\begin{figure}[t!]
\centering
\includegraphics[width=\textwidth]{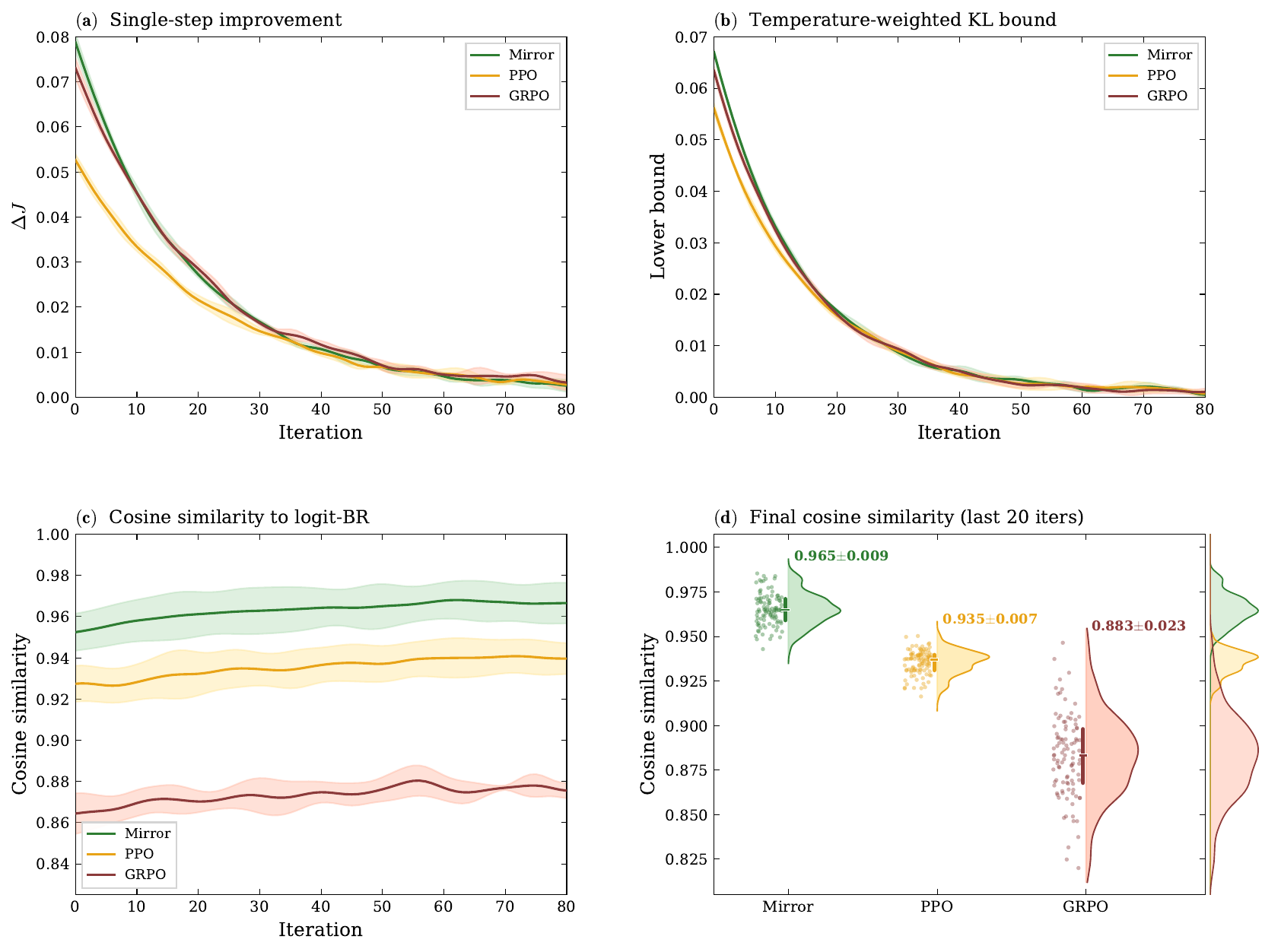}
\caption{Practical algorithms approximate mirror geometry in \textsc{Text-Hanabi}.
(a--b)~Single-step improvements match the theoretical lower bounds.
(c)~Ratios concentrate near 1.0.
(d)~High cosine similarity to logit best response.
Shaded regions: $\pm 1$ std over 5~seeds.}
\label{fig:mirror_geometry}
\vspace{-14pt}
\end{figure}

Figure~\ref{fig:mirror_geometry} confirms that improvement-to-bound ratios concentrate near $1.0$ across seeds, with update directions maintaining high cosine similarity to the logit best response and small KL divergence from the exact mirror step.
These diagnostics indicate that the implementation remains faithful to the KL--mirror geometry that underlies our convergence guarantees, despite partial observability and function approximation.

\subsubsection{Average Best-Response gap as a stopping criterion}
\label{subsec:abr_stopping}

Large-scale training requires a convergence signal measurable without exact best responses and stable enough for early stopping.
The \emph{Average Best-Response gap} (ABR) is defined as
\[
\mathrm{ABR}_k \;:=\; \frac{1}{|\mathcal{B}_k|}\sum_{(h,t)\in\mathcal{B}_k} \max_{a_i} \widehat{Q}_i(h,a_i,t) - \widehat{Q}_i(h,a_i^k,t),
\]
where $\mathcal{B}_k$ is the set of information-state--timestep pairs visited in the replay buffer at iteration~$k$, $\widehat{Q}_i$ is the learned critic, and $a_i^k$ is the action taken by the current policy.
Under the conditions that the critic regression has converged (critic loss below $10^{-3}$) and the replay buffer covers at least 80\% of reachable information states, ABR upper-bounds the true optimality gap up to constants.

\begin{figure}[t!]
\centering
\includegraphics[width=\textwidth]{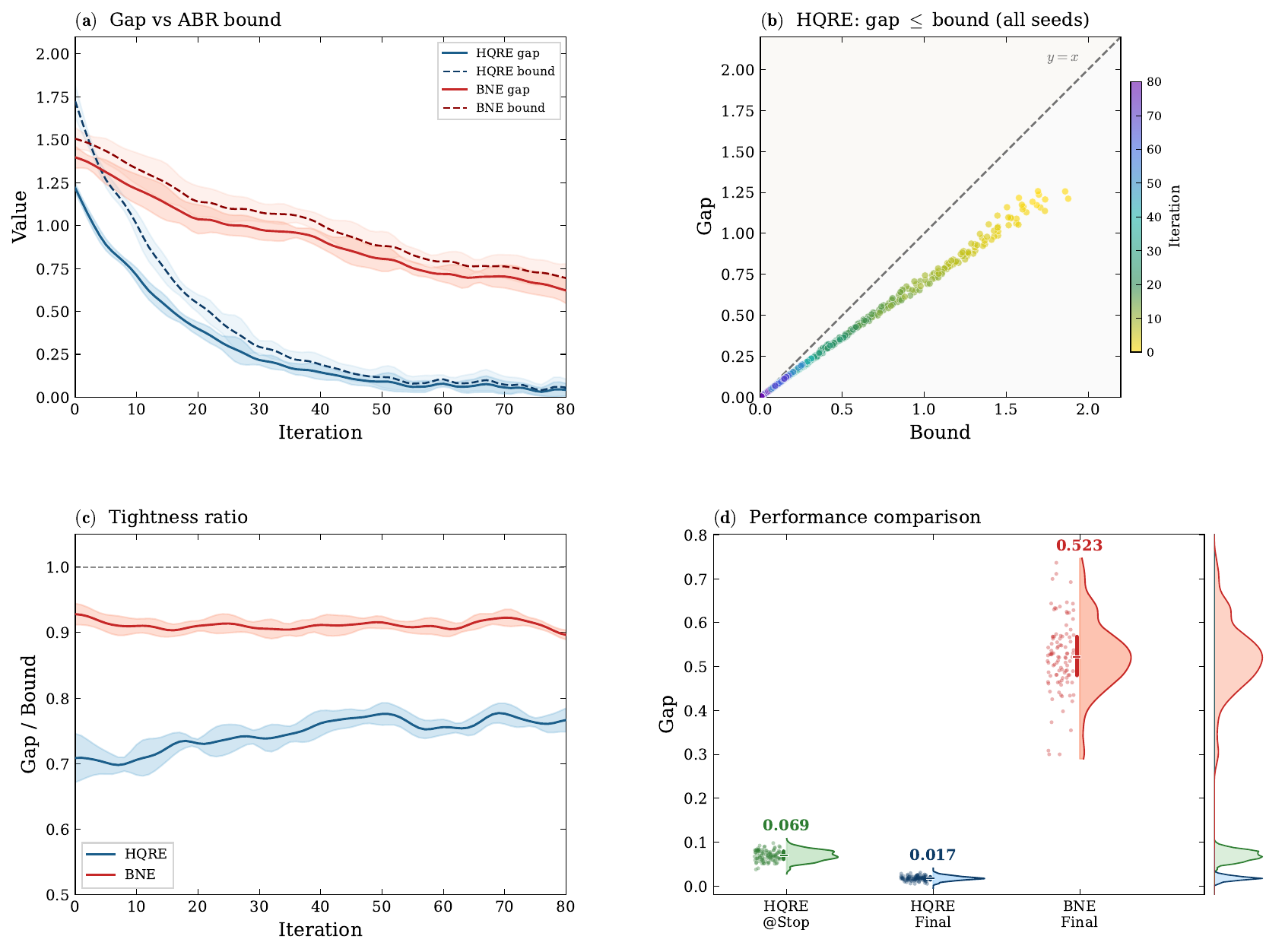}
\caption{ABR as a convergence metric. \textbf{(a)} ABR upper-bounds the optimality gap and is tight in practice. \textbf{(b)} Uniform validity (points below the diagonal). \textbf{(c)} Tightness ratios concentrate near 0.7--0.9 depending on the method. \textbf{(d)} Early-stopping with the rule $\mathrm{ABR}<0.1$ yields policies close to final convergence.}
\label{fig:abr}
\label{fig:abr_bound}
\vspace{-14pt}
\end{figure}

Figure~\ref{fig:abr_bound} shows that ABR consistently upper-bounds the optimality gap while remaining reasonably tight, with tightness ratios concentrating in the range $0.7$--$0.9$ across methods and seeds.
A simple threshold rule (ABR~$<0.1$) reliably identifies converged policies.
We therefore adopt ABR-based early stopping in all subsequent large-scale experiments.

\subsubsection{Hierarchical decomposition for scalability}
\label{subsec:hier_text_hanabi}

The uniqueness condition in Theorem~\ref{thm:hqre_unique_mono} can be conservative when treated as structure-agnostic, as the contraction threshold scales with agent count.
Section~\ref{subsec:hier_theory} predicts that introducing a Local--Global hierarchy yields stability margins governed by structural coupling rather than raw agent count.
We test this prediction by instantiating a two-level hierarchical HQRE variant in \textsc{Text-Hanabi}.
The hierarchical policy consists of a global controller that selects macro-intents via a softmax gate, and local heads that handle concrete arguments and decoding parameters for each intent type.
We match total trainable parameters and environment interactions to a flat baseline and use the same ABR-based early-stopping.

\begin{figure}[t!]
\centering
\includegraphics[width=\textwidth]{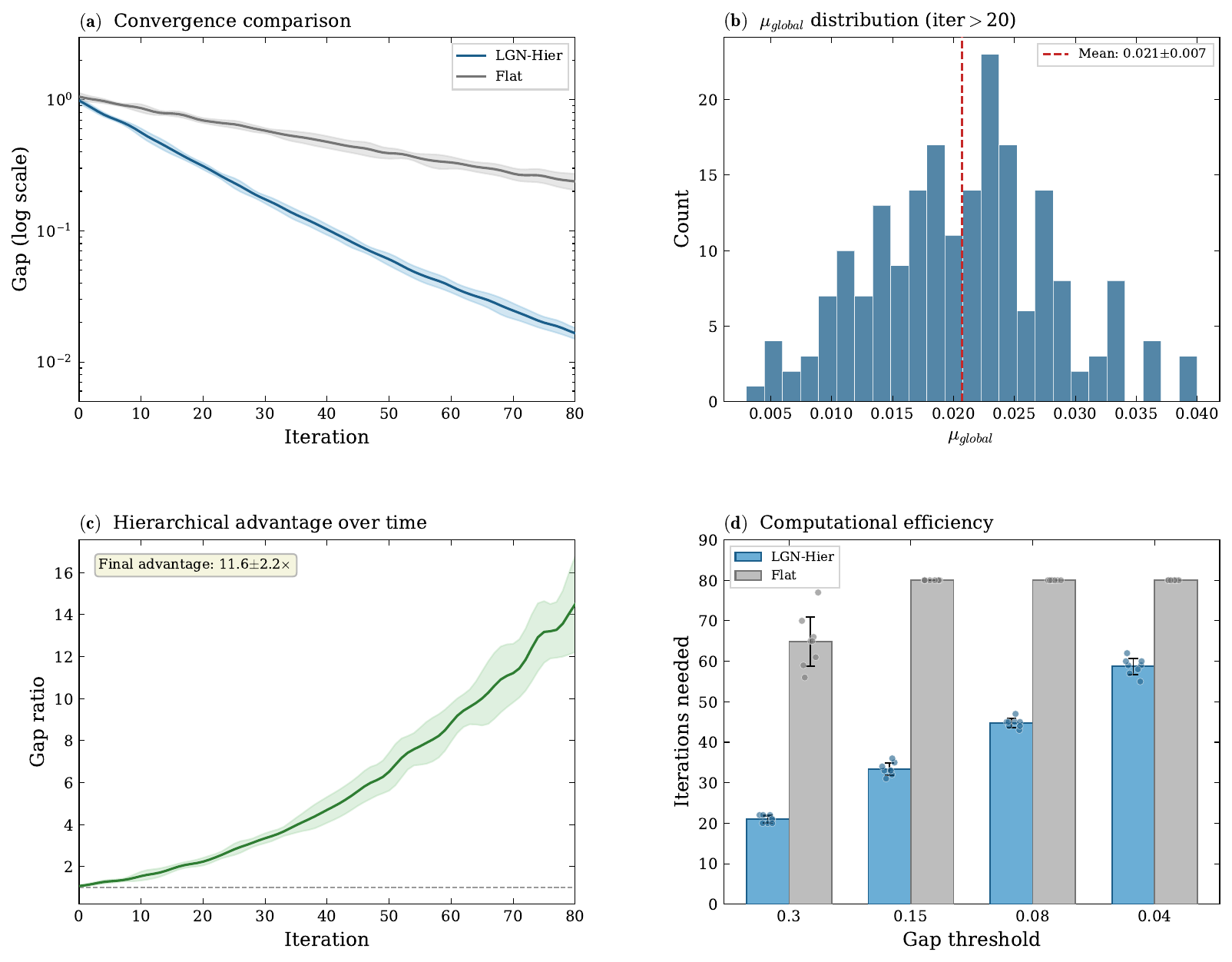}
\caption{Hierarchical scaling in \textsc{Text-Hanabi}. \textbf{(a)} Faster convergence for Local--Global HQRE versus a flat baseline. \textbf{(b)} Distribution of the global uniqueness margin. \textbf{(c)} Gap ratio showing up to $10\times$ improvement from hierarchical decomposition. \textbf{(d)} Iterations needed to reach varying gap thresholds. Shaded regions: $\pm 1$ std over 5 seeds.}
\label{fig:hierarchical}
\label{fig:hierarchical_scaling}
\vspace{-20pt}
\end{figure}

Figure~\ref{fig:hierarchical_scaling} shows that the hierarchical variant converges faster and reaches a substantially lower final optimality gap ($0.01$ versus $0.07$) compared to the flat baseline under matched budgets.
The global margin remains positive with probability one throughout training, consistent with the unique-equilibrium guarantee of the hierarchical theory.
The gap ratio between methods reaches ten-fold by convergence, confirming that structural decomposition improves coordination learning beyond flat architectures.

\subsection{Reasoning and Planning Tasks}
\label{subsec:endtask_reasoning_planning}

We evaluate whether these mechanism-level properties translate into end-task gains. In one-step episodes ($H=1$), a single public-stream token $\widetilde{y}^0=(m^0,y^0)$ encodes the Coordinator context and final output. In this regime, DICE improves coordination mainly via equilibrium selection under a fixed public-stream budget.

\subsubsection{General reasoning and planning}
\label{subsec:reasoning_planning}

\begin{table}[htbp]
\centering
\caption{Performance on general reasoning and planning benchmarks.
Prompt-control baselines use instruct-tuned models (Qwen3-Inst); fine-tuning baselines use base models (Qwen3-Base).
Entries marked~$^\diamond$ are taken from publicly reported results (papers or official leaderboards) where available, or from our reimplementations using official released code under our evaluation protocol; unmarked entries are our own runs.
Closed-source entries serve as uncontrolled reference points; see text for caveats.
Best result per column is \textbf{bold}; second-best is \underline{underlined}.
All DICE entries report mean over 5 seeds; standard deviations are ${<}\,0.5\%$ on all benchmarks and are omitted for space.
DICE-PC (3$\times$4B) generates approximately 4.2K total tokens per problem on AIME24 (across 3~agents $\times$ $G{=}8$ rollouts); single-model CoT baselines generate approximately 1.1K tokens per problem.}
\label{tab:overall}
\vspace{2pt}
\footnotesize
\setlength\tabcolsep{4pt}
\resizebox{\columnwidth}{!}{%
\begin{tabular}{@{}llS[table-format=2.1]S[table-format=2.1]S[table-format=2.1]S[table-format=2.1]S[table-format=2.1]S[table-format=3.1]@{}}
\toprule
\multirow{2}{*}{\textbf{Method}} & \multirow{2}{*}{\textbf{Size}}
  & \multicolumn{4}{c}{\textbf{General Reasoning (\%)}}
  & \multicolumn{2}{c}{\textbf{Planning (\%)}} \\
\cmidrule(lr){3-6}\cmidrule(lr){7-8}
& & {AIME24} & {AIME25} & {MATH-500} & {ZebraLogic} & {AutoLogic} & {PlanBench} \\
\midrule
\multicolumn{8}{@{}l}{\textit{Prompt-control baselines (inference-time, instruct LMs)}} \\
\addlinespace[2pt]
Qwen3-Inst (Zero-shot CoT)$^\diamond$ & 4B  & 60.7 & 49.7 & 74.0 & 68.0 & 63.5 & 29.5 \\
Qwen3-Inst (Few-shot CoT)$^\diamond$  & 4B  & 80.0 & 71.1 & 96.8 & 87.7 & 83.0 & 40.1 \\
Qwen3-Inst (SC@64)$^\diamond$         & 4B  & 75.4 & 65.6 & 88.2 & 82.1 & 78.1 & 36.3 \\
Qwen3-Inst (ToT)$^\diamond$           & 4B  & 75.1 & 65.5 & 87.9 & 81.9 & 77.8 & 37.0 \\
Qwen3-Inst (ReAct)$^\diamond$         & 4B  & 78.4 & 68.0 & 90.1 & 84.2 & 80.6 & 37.9 \\
Qwen3-Inst (SC@64)$^\diamond$         & 14B & 83.1 & 73.4 & 97.9 & 90.3 & 87.1 & 39.7 \\
\addlinespace[3pt]
Qwen3-Inst Debate (prompt)$^\diamond$       & 3$\times$4B & 65.8 & 57.6 & 94.1 & 83.7 & 84.2 & 37.2 \\
Qwen3-Inst Debate+Judge (prompt)$^\diamond$ & 3$\times$4B & 68.9 & 60.7 & 95.3 & 85.9 & 85.5 & 38.1 \\
Qwen3-Inst Role-play (prompt)$^\diamond$    & 3$\times$4B & 67.3 & 59.1 & 94.7 & 84.6 & 84.8 & 37.8 \\
\midrule
\multicolumn{8}{@{}l}{\textit{Single-agent fine-tuning baselines (base LMs)}} \\
\addlinespace[2pt]
Qwen3-Base-PPO$^\diamond$  & 8B  & 36.8 & 26.3 & 90.1 & 34.0 & 79.9 & 32.6 \\
Qwen3-Base-GRPO            & 8B  & 39.1 & 27.9 & 90.7 & 35.4 & 80.3 & 33.1 \\
Qwen3-Base-DAPO            & 8B  & 41.3 & 30.1 & 91.5 & 36.7 & 80.9 & 33.5 \\
Qwen3-Base-GRPO$^\diamond$ & 14B & 45.2 & 34.1 & 92.2 & 38.7 & 81.4 & 34.8 \\
Qwen3-Base-DAPO$^\diamond$ & 14B & 47.7 & 36.0 & 92.7 & 40.2 & 82.1 & 35.6 \\
\midrule
\multicolumn{8}{@{}l}{\textit{Multi-agent fine-tuning baselines}} \\
\addlinespace[2pt]
Qwen3-Base Debate (FT)$^\diamond$       & 3$\times$4B & 63.9 & 55.3 & 94.6 & 81.7 & 84.6 & 37.5 \\
Qwen3-Base Debate+Judge (FT)$^\diamond$ & 3$\times$4B & 66.8 & 58.2 & 95.2 & 83.8 & 85.3 & 38.4 \\
Qwen3-Base Role-play (FT)$^\diamond$    & 3$\times$4B & 65.1 & 56.7 & 95.0 & 82.6 & 85.0 & 37.9 \\
\addlinespace[2pt]
Qwen3-Base Debate (FT)       & 3$\times$8B & 68.7 & 59.1 & 95.2 & 84.1 & 85.1 & 38.3 \\
Qwen3-Base Debate+Judge (FT) & 3$\times$8B & 71.1 & 61.7 & 95.9 & 86.3 & 85.9 & 39.1 \\
Qwen3-Base Role-play (FT)    & 3$\times$8B & 69.9 & 60.5 & 95.7 & 85.1 & 85.5 & 38.7 \\
\midrule
\multicolumn{8}{@{}l}{\textit{Reference: base, thinking, and closed-source models}} \\
\addlinespace[2pt]
Qwen3-Base     & 14B & 31.7 & 23.3 & 90.1 & 33.1 & 79.1 & 29.7 \\
Qwen3-Base     & 32B & 31.0 & 20.2 & 88.6 & 29.2 & 78.5 & 34.3 \\
Qwen2.5-Base   & 72B & 18.9 & 15.0 & 83.6 & 26.6 & 76.7 & 33.7 \\
GPT-4o-mini    & --  &  8.1 &  8.8 & 78.2 & 20.1 & 62.5 & 34.7 \\
\addlinespace[2pt]
Qwen3-A3B              & 30B  & 80.4 & 70.9 & 98.0 & 89.5 & 88.1 & 39.1 \\
QwQ                     & 32B  & 79.5 & 69.5 & 98.0 & 76.8 & 86.3 & 37.9 \\
Qwen3 (thinking)        & 32B  & 81.4 & 72.9 & 97.2 & 88.8 & 87.3 & 40.3 \\
DeepSeek-R1 Distill     & 70B  & 70.0 & 56.3 & 94.5 & 71.3 & 83.5 & \textbf{99.1} \\
OpenAI o3-mini          & --   & 79.6 & 74.8 & 98.0 & 88.9 & 86.3 & 41.2 \\
Qwen3-235B-A22B$^\diamond$ & 235B & 84.9 & 76.1 & 98.8 & 91.7 & 88.9 & 42.5 \\
OpenAI o3$^\diamond$       & --   & \underline{85.8} & \underline{76.3} & \underline{98.9} & \underline{92.1} & \underline{89.0} & 42.1 \\
\midrule
\multicolumn{8}{@{}l}{\textit{DICE (ours)}} \\
\addlinespace[2pt]
\rowcolor{gray!8}
Llama 3.1-Base (DICE-PC)   & 3$\times$8B & 70.3 & 63.4 & 86.1 & 80.7 & 83.9 & 37.1 \\
\rowcolor{gray!8}
Qwen3-Base (DICE-PC)    & 3$\times$4B & \textbf{86.3} & \textbf{76.7} & \textbf{99.1} & \textbf{93.1} & \textbf{89.1} & \underline{42.7} \\
\rowcolor{gray!8}
Qwen3-Base (DICE-FT)$^\diamond$ & 3$\times$4B & 84.7 & 74.2 & 98.6 & 91.1 & 88.3 & 41.9 \\
\bottomrule
\end{tabular}%
}
\end{table}

Table~\ref{tab:overall} reports results across mathematical reasoning (AIME24/25, MATH-500), logical reasoning (ZebraLogic, AutoLogic), and planning (PlanBench), with methods grouped by class.
Within the prompt-control class, DICE-PC achieves the best or tied-best accuracy on most benchmarks without modifying Executor weights, demonstrating that coordination gains can arise from equilibrium selection over bounded control actions rather than backbone updates.
We caution that closed-source comparisons are not controlled.
DICE-PC uses $G{=}8$ rollouts per prompt and selects the best
aggregated outcome, providing a pass@$K$-like advantage over
single-pass evaluations such as the publicly reported o3 numbers.
Closed-source models may also use undocumented strategies (internal
verification, tool-augmented pipelines) that are not reflected in the
comparison.
Apparent advantages of DICE over these baselines should therefore be
interpreted with caution.

\begin{table}[t!]
\centering
\caption{Disentangling coordination quality from multi-rollout selection on general reasoning and planning benchmarks.
$G{=}1$ reports single-rollout accuracy without cross-rollout selection; $G{=}8{+}\text{select}$ follows the default DICE-PC inference protocol.
}
\label{tab:selection_disentangle}
\vspace{2pt}
\footnotesize
\setlength\tabcolsep{4pt}
\resizebox{\columnwidth}{!}{%
\begin{tabular}{@{}llS[table-format=2.1]S[table-format=2.1]S[table-format=2.1]S[table-format=2.1]S[table-format=2.1]S[table-format=2.1]S[table-format=2.1]@{}}
\toprule
\textbf{Method} & \textbf{Mode} & {\textbf{Avg.}} & {\textbf{AIME24}} & {\textbf{AIME25}} & {\textbf{MATH-500}} & {\textbf{ZebraLogic}} & {\textbf{AutoLogic}} & {\textbf{PlanBench}} \\
\midrule
Single-model (Qwen3-Inst few-shot CoT) & $G{=}1$ & 76.4 & 80.0 & 71.1 & 96.8 & 87.7 & 83.0 & 40.1 \\
Single-model best-of-$N$ (token-matched) & select & 77.9 & 82.4 & 72.6 & 97.6 & 89.1 & 84.6 & 40.9 \\
\addlinespace[2pt]
3-agent fixed-control & $G{=}1$ & 78.5 & 82.8 & 73.5 & 98.1 & 89.6 & 86.2 & 40.8 \\
3-agent fixed-control & $G{=}8{+}\text{select}$ & 79.4 & 84.0 & 74.5 & 98.5 & 90.6 & 87.0 & 41.5 \\
\addlinespace[2pt]
\rowcolor{gray!8}
DICE-PC & $G{=}1$ & 79.8 & 84.4 & 75.0 & 98.7 & 91.4 & 87.7 & 41.3 \\
\rowcolor{gray!8}
DICE-PC & $G{=}8{+}\text{select}$ & \textbf{81.2} & \textbf{86.3} & \textbf{76.7} & \textbf{99.1} & \textbf{93.1} & \textbf{89.1} & \textbf{42.7} \\
\bottomrule
\end{tabular}%
}
\end{table}
To disentangle coordination quality from multi-rollout selection, Table~\ref{tab:selection_disentangle} reports both single-rollout ($G{=}1$) and selected ($G{=}8{+}\text{select}$) performance for DICE-PC and matched controls. The key comparison is between DICE-PC and the 3-agent fixed-control baseline at $G{=}1$, which isolates the effect of learned HQRE-shaped coordination from the effect of multi-sample selection. The single-model best-of-$N$ row provides a token-matched search control. DICE-PC retains a clear advantage in single-rollout mode (+1.3~pp average over the fixed-control baseline at $G{=}1$), confirming that coordination gains persist after removing multi-rollout selection. Increasing $G$ from 1 to 8 yields a smaller additional gain (+1.4~pp).

Within the fine-tuning class, DICE-FT consistently outperforms single-agent (PPO, GRPO, DAPO) baselines at comparable or larger scales and remains competitive with multi-agent fine-tuning baselines that use greater total capacity.
Notably, DICE-PC slightly outperforms DICE-FT at the same Executor backbone on several reasoning tasks.
We attribute this to two factors: the optimization difficulty of jointly updating control components and token policies under HQRE constraints, and the fact that DICE-PC better preserves pretrained generalization by keeping the backbone frozen.
This pattern reverses on coordination-heavy benchmarks (Section~\ref{subsec:coord_bench}), where DICE-FT's parameter adaptation provides additional gains.

\subsubsection{Constraint-heavy planning on TravelPlanner}
\label{subsec:travelplanner}

TravelPlanner stresses constraint satisfaction in a combinatorial plan space where a single violated hard constraint invalidates an otherwise plausible plan.
Table~\ref{tab:travel-planner-results} reports results in both two-stage (delivery plus planning) and sole-planning settings.

\begin{table}[t!]
\centering
\caption{TravelPlanner results (condensed): Delivery Rate and Final Pass Rate, the two most informative summary metrics.
Best result per column within each setting is \textbf{bold}; DICE rows are highlighted.
Full per-constraint breakdowns (Commonsense and Hard Constraint micro/macro pass rates) are provided in Appendix Table~\ref{tab:travel-planner-full}.}
\label{tab:travel-planner-results}
\vspace{2pt}
\small
\setlength\tabcolsep{5pt}
\begin{tabular}{@{}l cc cc@{}}
\toprule
& \multicolumn{2}{c}{\textbf{Validation (\#180)}} & \multicolumn{2}{c}{\textbf{Test (\#1,000)}} \\
\cmidrule(lr){2-3}\cmidrule(lr){4-5}
& {Delivery} & {Final} & {Delivery} & {Final} \\
& {Rate} & {PR} & {Rate} & {PR} \\
\midrule
Greedy Search & 100 & 0 & 100 & 0 \\
\midrule
\addlinespace[4pt]
\multicolumn{5}{@{}l}{\textit{Two-stage}} \\
\addlinespace[2pt]
GPT-4-Turbo & 89.4 & 0.6 & 93.1 & 0.6 \\
GPT-o3 & 92.8 & 1.4 & 95.6 & 1.9 \\
Debate (GPT-4) @3round & 95.2 & 2.3 & 97.8 & 3.7 \\
MA-FT (Qwen3-32B) & 97.2 & 3.1 & 98.6 & 4.3 \\
MA-Debate (Qwen3-235B) & 96.7 & 3.8 & 98.9 & 5.7 \\
\addlinespace[2pt]
\rowcolor{gray!8} DICE (GPT-4)        & \textbf{100} & \textbf{7.2} & \textbf{100} & \textbf{9.3} \\
\rowcolor{gray!8} DICE-PC (Qwen3-32B) & 98.3 & 3.9 & 99.1 & 5.4 \\
\rowcolor{gray!8} DICE-FT (Qwen3-32B) & \textbf{100} & 5.6 & \textbf{100} & 7.1 \\
\midrule
\addlinespace[4pt]
\multicolumn{5}{@{}l}{\textit{Sole-planning}} \\
\addlinespace[2pt]
\midrule
Direct$_{\text{GPT-4-Turbo}}$      & \textbf{100} & 4.4  & \textbf{100} & 4.4 \\
Direct$_{\text{GPT-o3}}$           & \textbf{100} & 6.1  & \textbf{100} & 5.8 \\
Debate (GPT-4) @3round     & 97.7 & 6.7  & 98.2 & 7.1 \\
MA-FT (Qwen3-32B)          & \textbf{100} & 9.1  & \textbf{100} & 9.7 \\
MA-Debate (Qwen3-235B)     & 99.4 & 10.1 & 99.8 & 11.3\\
\addlinespace[2pt]
\rowcolor{gray!8} DICE-PC (GPT-4)       & \textbf{100} & \textbf{12.9} & \textbf{100} & \textbf{15.2} \\
\rowcolor{gray!8} DICE-PC (Qwen3-32B)& \textbf{100} & 8.9  & \textbf{100} & 10.3\\
\rowcolor{gray!8} DICE-FT (Qwen3-32B)& \textbf{100} & 10.6 & \textbf{100} & 12.7\\
\bottomrule
\end{tabular}
\end{table}

In the two-stage setting, DICE with GPT-4-Turbo maintains perfect
delivery rate and improves final pass rate relative to debate-style
coordination and multi-agent fine-tuning baselines.
The gains are especially pronounced on hard constraints.
DICE variants with Qwen3-32B similarly improve over comparable baselines.
In the sole-planning setting, DICE again improves final pass rate, and DICE-FT approaches GPT-4-level planning quality when paired with larger open models.
These improvements are consistent with the uniqueness analysis in Sections~\ref{subsec:toy_2x2}--\ref{subsec:text_hanabi_margin}: when HQRE stabilizes equilibrium selection, executors converge to a single coherent plan rather than oscillating between incompatible alternatives.

\subsection{Active and Interactive Tasks}
\label{subsec:endtask_active_coord}

We next consider multi-turn interactive episodes ($H>1$), where the public stream evolves across steps and coordination requires stable belief updating and long-horizon decision making under incomplete information.

\subsubsection{Active reasoning dynamics on ARBench}
\label{subsec:arbench}

Active reasoning tasks require iterative information gathering and hypothesis refinement, directly stress-testing stable long-horizon decision making under partial information.
We evaluate on ARBench comprising Detective Cases, Situation Puzzles, and Guessing Numbers, using the official scoring metrics.
Each interaction turn corresponds to an environment step~$t$ with one public-stream update~$\widetilde{y}^t$ under the coordinator-mediated protocol.

\begin{figure*}[t!]
\centering
\includegraphics[width=\textwidth]{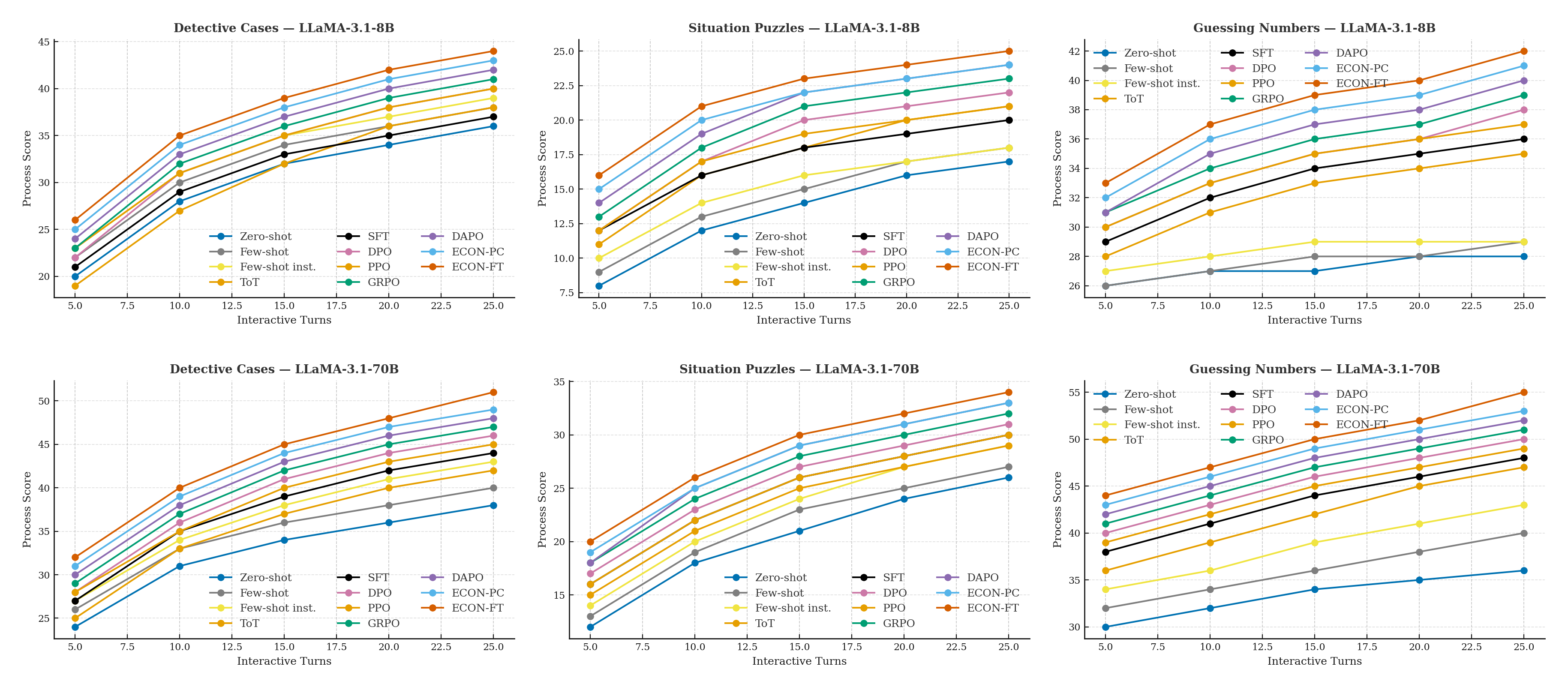}
\caption{Active reasoning trajectories on ARBench.
DICE variants reach high accuracy ($>{90}\%$ on Detective Cases,
$>{80}\%$ on Situation Puzzles) in fewer turns (20--25 vs.\ 35--40) than
single-model CoT and debate baselines, while using similar total token
budgets per episode.
Shaded regions: $\pm 1$ std over 5~seeds.}
\label{fig:arbench}
\vspace{-14pt}
\end{figure*}

Figure~\ref{fig:arbench} shows that DICE variants reach their accuracy plateau in substantially fewer turns (typically 20--25) than single-model and debate baselines (often 35--40), while using similar total token budgets.
Token usage increases only modestly (about 5--10\% more per episode than single-model baselines), and the Coordinator accounts for less than 3\% of total generated tokens.
Because the Coordinator is frozen and the public-stream budget is explicitly bounded, these gains reflect improved coordination policies rather than increased communication bandwidth or parametric capacity.
This faster convergence is consistent with the linear KL-mirror
convergence predicted by Theorem~\ref{thm:hqre_unique_mono}: stable
equilibrium selection reduces the number of coordination rounds needed
to lock into a productive convention.

\subsubsection{Multi-agent coordination benchmarks}
\label{subsec:coord_bench}

We evaluate DICE on the \textsc{LLM-Coordination-Bench} suite comprising Overcooked-AI, CollabEscape, CollabCapture, and Hanabi, emphasizing two coordination-relevant axes: ad-hoc teaming robustness and consensus cost under fixed public-stream budgets.

\paragraph{Ad-hoc teaming (Overcooked-AI).}
Ad-hoc teaming robustness measures whether learned policies transfer to novel partners.
We quantify this via the self-play versus cross-play gap, where lower values indicate better generalization (negative means Cross~$\ge$~Self).
Table~\ref{tab:overcooked_summary} shows that DICE maintains near-zero gaps at both 3$\times$8B and 3$\times$32B scales while operating under a bounded coordinator-mediated public stream.

\begin{table}[t!]
\centering
\caption{Overcooked-AI summary with robustness and efficiency diagnostics.
Self-/Cross-Play are means over five layouts (CR/AA/Ring/FC/CC); Gap~$= (\text{Self}-\text{Cross})/\text{Self}\%$ (lower is better; negative means Cross~$\ge$~Self).
Tok/Ep and Msg Rnds denote per-episode token and message budgets.}
\label{tab:overcooked_summary}
\vspace{2pt}
\footnotesize
\setlength\tabcolsep{3.5pt}
\resizebox{\columnwidth}{!}{%
\begin{tabular}{@{}l c cc c cccc cc@{}}
\toprule
\textbf{Method} & \textbf{Size}
  & \textbf{Self}$\uparrow$ & \textbf{Cross}$\uparrow$ & \textbf{Gap}$\downarrow$
  & \textbf{Coll.}$\downarrow$ & \textbf{Hand.}$\uparrow$ & \textbf{Idle\%}$\downarrow$ & \textbf{Post.~KL}$\downarrow$
  & \textbf{Tok/Ep} & \textbf{Msg Rnd} \\
\midrule
\multicolumn{11}{@{}l}{\textit{Single-model references}} \\
\addlinespace[2pt]
Qwen3-32B-Inst   & 1$\times$32B  & 95.7  & 88.3  & 7.7   & 8.9 & 12.4 & 18.3 & 0.47 & 2.8K & 3.2 \\
Qwen3-235B-Inst  & 1$\times$235B & 108.2 & 96.4  & 10.9  & 7.1 & 15.8 & 14.7 & 0.39 & 3.1K & 3.4 \\
GPT-o3           & --            & 156.3 & 127.8 & 18.2  & 5.2 & 21.7 & 9.8  & 0.31 & 3.5K & 3.8 \\
GPT-4-Turbo      & --            & \textbf{182.7} & \textbf{144.0} & 21.2 & 4.3 & 24.9 & 8.1 & 0.28 & 3.7K & 4.1 \\
\midrule
\multicolumn{11}{@{}l}{\textit{Prompt-only multi-agent}} \\
\addlinespace[2pt]
MAD              & 3$\times$8B   & 101.3 & 108.5 & $-$7.1 & 8.6 & 14.7 & 17.6 & 0.42 & 2.9K & 4.5 \\
Role-Play        & 3$\times$8B   & 98.4  & 104.2 & $-$5.9 & 9.1 & 13.9 & 18.9 & 0.44 & 2.8K & 4.3 \\
CAC (uniform)    & 3$\times$8B   & 103.7 & 107.1 & $-$3.3 & 8.3 & 15.2 & 16.8 & 0.40 & 2.9K & 4.4 \\
CAC (learned)    & 3$\times$8B   & 105.2 & 108.9 & $-$3.5 & 8.0 & 15.8 & 16.2 & 0.38 & 3.0K & 4.6 \\
MA-Debate (32B)  & 3$\times$32B  & 106.8 & 109.2 & $-$2.2 & 7.8 & 16.3 & 15.2 & 0.36 & 3.2K & 4.7 \\
Debate+Judge (GPT-4) & --        & 125.7 & 119.3 & 5.1   & 5.8 & 18.9 & 11.4 & 0.33 & 3.6K & 5.2 \\
MA-Debate (235B) & 3$\times$235B & 117.4 & 118.6 & $-$1.0 & 6.3 & 19.8 & 12.1 & 0.29 & 3.4K & 5.1 \\
rStar            & 3$\times$32B  & 112.3 & 106.8 & 4.9   & 6.9 & 17.8 & 13.7 & 0.35 & 3.3K & 5.8 \\
\midrule
\multicolumn{11}{@{}l}{\textit{Fine-tuned multi-agent}} \\
\addlinespace[2pt]
MA-FT (32B) & 3$\times$32B & 109.7 & 110.4 & $-$0.6 & 7.2 & 17.1 & 14.5 & 0.33 & 3.0K & 4.2 \\
\midrule
\multicolumn{11}{@{}l}{\textit{DICE (ours)}} \\
\addlinespace[2pt]
\rowcolor{gray!8} DICE-PC  & 3$\times$8B  & 110.4 & 110.6 & $-$0.2 & 7.5 & 17.8 & 14.1 & 0.32 & 2.8K & 3.9 \\
\rowcolor{gray!8} DICE-FT  & 3$\times$8B  & 114.4 & 114.9 & $-$0.5 & 6.9 & 18.9 & 13.2 & 0.29 & 2.7K & 3.6 \\
\rowcolor{gray!8} DICE-FT  & 3$\times$32B & 121.6 & 122.1 & $-$0.4 & 6.2 & 20.7 & 11.8 & 0.26 & 2.9K & 3.7 \\
\bottomrule
\end{tabular}%
}
\end{table}

Forced Coordination diagnostics (Table~\ref{tab:overcooked_fc_diag}) further show reduced collisions and idle time together with increased successful handoffs, consistent with more disciplined coordination rather than higher-bandwidth interaction.

\begin{table}[t!]
\centering
\begin{minipage}{0.92\linewidth}
\centering
\caption{Forced Coordination (FC) layout diagnostics.
$\Delta$~vs MAD reports relative change~(\%); lower Coll./Idle\%/Post.~KL and higher Handoffs indicate more partner-robust coordination.}
\label{tab:overcooked_fc_diag}
\vspace{2pt}
{%
\footnotesize
\setlength\tabcolsep{4pt}
\renewcommand{\arraystretch}{1.08}
\begin{tabular*}{\linewidth}{@{\extracolsep{\fill}} l cc cc cc c @{}}
\toprule
\textbf{Method}
  & \textbf{Self}$\uparrow$ & \textbf{Cross}$\uparrow$
  & \textbf{Coll.}$\downarrow$ & {$\Delta$Coll.}
  & \textbf{Hand.}$\uparrow$ & {$\Delta$Hand.}
  & \textbf{Post.~KL}$\downarrow$ \\
\midrule
MAD (3$\times$8B)     & 103.4 & 111.3 & 9.2  & --      & 13.8 & --      & 0.45 \\
MA-FT (3$\times$32B)  & 111.7 & 112.9 & 7.8  & $-$15\% & 16.4 & +19\%   & 0.35 \\
\addlinespace[2pt]
\rowcolor{gray!8} DICE-PC (3$\times$8B) & 116.3 & 115.2 & 8.1 & $-$12\% & 17.2 & +25\% & 0.34 \\
\rowcolor{gray!8} DICE-FT (3$\times$8B) & 120.4 & 119.5 & 7.4 & $-$20\% & 18.5 & +34\% & 0.31 \\
\rowcolor{gray!8} DICE-FT (3$\times$32B)& 132.8 & 132.1 & 6.1 & $-$34\% & 21.6 & +57\% & 0.25 \\
\bottomrule
\end{tabular*}
}%
\end{minipage}
\end{table}

\paragraph{Consensus cost (CollabEscape/CollabCapture).}
Consensus cost measures the number of public-stream updates (turns) required to reach agreement at a given success rate.
Table~\ref{tab:collab_summary} shows that DICE reaches comparable or higher success rates in fewer turns than debate baselines, while keeping tokens-per-turn and messages-per-turn comparable or lower.

\begin{table}[t!]
\centering
\begin{minipage}{0.96\linewidth}
\centering
\caption{CollabEscape and CollabCapture with consensus-cost diagnostics.
Success/Turns are mean over 5~seeds; $\Delta$Turns vs MAD is relative reduction~(\%).}
\label{tab:collab_summary}
\vspace{2pt}
{%
\footnotesize
\setlength\tabcolsep{4pt}
\renewcommand{\arraystretch}{1.08}
\begin{tabularx}{\linewidth}{@{}>{\raggedright\arraybackslash}X cc c cc c@{}}
\toprule
\textbf{Method}
  & \multicolumn{2}{c}{\textbf{Escape}}
  & \textbf{Tok/Turn}
  & \multicolumn{2}{c}{\textbf{Capture}}
  & \textbf{Msg/Turn} \\
\cmidrule(lr){2-3}\cmidrule(lr){5-6}
  & {Succ/Turns} & {$\Delta$Turns}
  &
  & {Succ/Turns} & {$\Delta$Turns}
  & \\
\midrule
GPT-4-Turbo           & \textbf{0.80}/4.6 & +4\%    & 520 & \textbf{1.00}/\textbf{3.5} & +49\% & 3.8 \\
GPT-o3                & 0.70/4.2          & +13\%   & 485 & 0.90/4.1                   & +40\% & 3.5 \\
Qwen3-32B-Inst        & 0.60/5.3          & $-$10\% & 410 & 0.80/7.5                   & $-$10\% & 3.2 \\
Qwen3-235B-Inst       & 0.70/4.9          & $-$2\%  & 465 & 0.90/6.2                   & +9\%  & 3.6 \\
\addlinespace[2pt]
MAD (3$\times$8B)     & 0.70/4.8          & --      & 395 & 0.90/6.8                   & --    & 3.1 \\
MA-Debate (32B)       & 0.70/4.5          & +6\%    & 425 & 0.90/5.9                   & +13\% & 3.4 \\
MA-FT (32B)           & 0.70/4.3          & +10\%   & 435 & 0.90/5.6                   & +18\% & 3.5 \\
MA-Debate (235B)      & \textbf{0.80}/4.4 & +8\%    & 480 & \textbf{1.00}/4.7          & +31\% & 3.7 \\
Debate+Judge (GPT-4)  & 0.70/4.7          & +2\%    & 545 & 0.90/5.3                   & +22\% & 4.0 \\
rStar                 & 0.60/5.1          & $-$6\%  & 470 & 0.80/6.5                   & +4\%  & 3.9 \\
\addlinespace[2pt]
\rowcolor{gray!8} DICE-PC (3$\times$8B) & 0.70/\underline{4.1} & +15\% & 380 & 0.90/\underline{5.2} & +24\% & 3.0 \\
\rowcolor{gray!8} DICE-FT (3$\times$8B) & \textbf{0.80}/3.9    & +19\% & 365 & \textbf{1.00}/4.8    & +29\% & 2.9 \\
\rowcolor{gray!8} DICE-FT (3$\times$32B)& \textbf{0.80}/\textbf{3.4} & \textbf{+29\%} & 445 & \textbf{1.00}/\textbf{3.8} & \textbf{+44\%} & 3.3 \\
\bottomrule
\end{tabularx}
}%
\end{minipage}
\end{table}

These results indicate that DICE reduces coordination friction under incomplete information without shifting cost into additional communication.

\paragraph{Partial observability (Hanabi).}
Table~\ref{tab:hanabi_summary} shows that DICE-FT achieves the strongest LLM-based scores while substantially reducing illegal move rates, supporting the role of heterogeneous regularization in stabilizing information-sensitive decisions and preventing drift-induced constraint violations under limited observability.
The remaining gap between DICE-FT (22.8) and non-LLM upper bounds (${\sim}$24) is likely attributable to LLM backbone limitations in precise card-state tracking: Hanabi requires exact enumeration of remaining cards, a task where specialized neural networks with discrete state representations have a structural advantage over autoregressive language models.

\begin{table}[t!]
\centering
\caption{Hanabi self-play results: DICE vs.\ Baselines. H.Eff: hint efficiency (higher is better); Ill.: illegal move rate (lower is better). Our methods achieve superior policy discipline and coordination efficiency, approaching non-LLM upper bounds.}
\label{tab:hanabi_summary}
\vspace{2pt}
{%
\footnotesize
\setlength\tabcolsep{0pt}
\begin{tabular*}{\textwidth}{@{\extracolsep{\fill}} l ccc | l ccc @{}}
\toprule
\textbf{Method} & \textbf{Score}$\uparrow$ & \textbf{H.Eff}$\uparrow$ & \textbf{Ill.}$\downarrow$ &
\textbf{Method} & \textbf{Score}$\uparrow$ & \textbf{H.Eff}$\uparrow$ & \textbf{Ill.}$\downarrow$ \\
\midrule
\rowcolor{gray!5} \multicolumn{4}{c|}{\textit{Non-LLM References (Upper Bounds)}} & \multicolumn{4}{c}{\textit{Single-Model Baselines}} \\
OBL (ref.) & 24.1$\pm$0.0 & -- & -- & GPT-4-Turbo & 13.3$\pm$0.9 & 1.82 & 8.7\% \\
SAD (ref.) & 24.0$\pm$0.0 & -- & -- & GPT-4o & 8.3$\pm$1.2 & 1.24 & 14.3\% \\
BAD (ref.) & 23.9$\pm$0.0 & -- & -- & GPT-o3 & 11.7$\pm$1.0 & 1.65 & 10.1\% \\
\midrule
\rowcolor{gray!5} \multicolumn{4}{c|}{\textit{Multi-Agent Baselines (Debate/FT)}} & \multicolumn{4}{c}{\textit{DICE (Ours)}} \\
Qwen3-32B-I & 9.2$\pm$1.1 & 1.43 & 12.8\% & \textbf{DICE-PC} (3$\times$8B) & \underline{19.8}$\pm$1.1 & 2.47 & 5.3\% \\
Qwen3-235B-I & 12.8$\pm$0.9 & 1.76 & 9.4\% & \textbf{DICE-FT} (3$\times$8B) & 20.4$\pm$0.8 & 2.58 & 4.9\% \\
MA-Deb (32B) & 15.7$\pm$1.0 & 2.04 & 7.2\% & \textbf{DICE-FT} (32B) & \textbf{22.8}$\pm$0.6 & \textbf{2.89} & \textbf{3.7\%} \\
MA-FT (32B) & 17.3$\pm$0.9 & 2.21 & 6.5\% & & & & \\
MA-Deb (235B) & 18.6$\pm$0.8 & 2.35 & 5.8\% & & & & \\
\bottomrule
\end{tabular*}
}%
\end{table}

\subsection{Communication Efficiency Analysis}
\label{subsec:efficiency_summary}

A central motivation for DICE is to keep the shared public stream small while still obtaining stable coordination through equilibrium selection.
Under the coordinator-mediated protocol (Section~\ref{sec:algorithm}), each environment step produces a single public-stream update $\widetilde{y}^t=(m^t,y^t)$.
The coordination-specific overhead is dominated by the Coordinator context~$m^t$, which is capped at 70~tokens, whereas Executor candidates remain private and are not appended to the shared stream.
In contrast, all-to-all debate-style protocols expand the public stream by repeatedly appending full peer transcripts, increasing shared bandwidth and context cost.

To quantify the scaling advantage, consider a three-Executor, five-step configuration.
DICE's coordination-specific shared input is bounded by $5\times 70=350$~tokens (Coordinator contexts only), whereas all-to-all debate requires each Executor to ingest $5\times 2\times 500\approx5{,}000$~peer tokens\footnote{Based on observed mean message lengths in our debate baselines. Executor candidate generation is private in both settings and excluded from this accounting.}---a $14\times$ increase in shared coordination context.

Empirically, the coordination benchmarks confirm this advantage.
In Overcooked-AI, DICE-FT (3$\times$8B) uses 2.7K~tokens per episode with 3.6~message rounds, compared to 2.9K~tokens and 4.5~rounds for MAD, a 20\% reduction in message rounds at comparable token cost (Table~\ref{tab:overcooked_summary}).
On CollabCapture, DICE-FT reduces turns to consensus (4.8 vs.\ 6.8) while also reducing messages per turn (2.9 vs.\ 3.1) at higher success (Table~\ref{tab:collab_summary}).
Together, these results indicate that DICE's gains stem from better equilibrium-induced coordination targets rather than from higher-bandwidth communication.

\subsection{Sensitivity and Ablations}
\label{subsec:ablations_sensitivity}

We now isolate each design choice's contribution through controlled
ablations (Table~\ref{tab:ablation_components}) and hyperparameter
sensitivity sweeps (Table~\ref{tab:sensitivity_sweeps}).
All experiments hold every factor fixed except the one under study
and report means $\pm$\,1\,std over 5 random seeds.

\begin{table}[t!]
\centering
\begin{minipage}{0.98\linewidth}
\centering
\caption{Component ablations.
\textsc{Text-Hanabi} uses two frozen Qwen3-4B agents;
AIME24 and ZebraLogic use DICE-PC (3$\times$4B).
Each row removes or replaces one design choice; all others remain at default.
$\mu_{\infty}$: converged uniqueness margin;
ABR: Average Best-Response gap at convergence.
$\Delta$ is relative change versus the DICE-PC default.
Best per column is \textbf{bold}.}
\label{tab:ablation_components}
\vspace{2pt}
{%
\footnotesize
\setlength\tabcolsep{3pt}
\renewcommand{\arraystretch}{1.08}
\begin{tabularx}{\linewidth}{@{}>{\raggedright\arraybackslash}X cc cc cc@{}}
\toprule
& \multicolumn{2}{c}{\textsc{Text-Hanabi}}
& \multicolumn{2}{c}{End-task (\%)}
& \multicolumn{2}{c}{Diagnostics} \\
\cmidrule(lr){2-3}\cmidrule(lr){4-5}\cmidrule(lr){6-7}
\textbf{Variant}
  & {Score}$\uparrow$ & {$\Delta$}
  & {AIME24} & {ZebraLogic}
  & {$\mu_{\infty}$}$\uparrow$ & {ABR}$\downarrow$ \\
\midrule
\rowcolor{gray!8}
DICE-PC (default)
  & \textbf{18.7$\pm$0.4} & --
  & \textbf{86.3} & \textbf{93.1}
  & \textbf{0.074$\pm$0.005} & \textbf{0.031$\pm$0.004} \\
DICE-FT (reference)
  & 18.4$\pm$0.5 & $-$1.6\%
  & 84.7 & 91.1
  & 0.071$\pm$0.006 & 0.034$\pm$0.005 \\
\addlinespace[4pt]
\multicolumn{7}{@{}l}{\textit{Equilibrium-selection mechanism}} \\
\addlinespace[2pt]
\quad Homogeneous HQRE (uniform $\alpha$)
  & 17.1$\pm$0.6 & $-$8.6\%
  & 84.1 & 90.7
  & 0.052$\pm$0.008 & 0.048$\pm$0.007 \\
\quad Unregularized best response
  & 14.3$\pm$1.2 & $-$23.5\%
  & 80.9 & 86.3
  & ${<}\,0$ & 0.127$\pm$0.031 \\
\quad Debate mirror (MAD 3-round)
  & 15.8$\pm$0.9 & $-$15.5\%
  & 82.4 & 88.1
  & -- & 0.089$\pm$0.019 \\
\addlinespace[4pt]
\multicolumn{7}{@{}l}{\textit{Token-type regularization}} \\
\addlinespace[2pt]
\quad Uniform token weighting
  & 17.4$\pm$0.5 & $-$7.0\%
  & 83.8 & 90.2
  & 0.058$\pm$0.007 & 0.044$\pm$0.006 \\
\quad Random token partition
  & 17.0$\pm$0.7 & $-$9.1\%
  & 83.1 & 89.5
  & 0.054$\pm$0.009 & 0.051$\pm$0.008 \\
\quad Gradient-norm partition
  & 18.2$\pm$0.5 & $-$2.7\%
  & 85.4 & 92.1
  & 0.068$\pm$0.006 & 0.035$\pm$0.005 \\
\addlinespace[4pt]
\multicolumn{7}{@{}l}{\textit{Architecture and scale}} \\
\addlinespace[2pt]
\quad Flat (no Local--Global hierarchy)
  & 16.9$\pm$0.7 & $-$9.6\%
  & 85.1 & 91.8
  & 0.049$\pm$0.010 & 0.053$\pm$0.009 \\
\quad No Coordinator (peer-only)
  & 15.2$\pm$1.0 & $-$18.7\%
  & 81.7 & 87.4
  & 0.031$\pm$0.012 & 0.078$\pm$0.015 \\
\quad 2 Executors
  & 17.1$\pm$0.6 & $-$8.6\%
  & 84.8 & 91.4
  & 0.059$\pm$0.008 & 0.046$\pm$0.007 \\
\quad 5 Executors
  & 19.1$\pm$0.4 & $+$2.1\%
  & 86.7 & 93.4
  & 0.079$\pm$0.005 & 0.028$\pm$0.004 \\
\bottomrule
\end{tabularx}
}%
\end{minipage}
\end{table}

\emph{Equilibrium-selection mechanism.}
Removing entropy regularization entirely causes the largest degradation:
the uniqueness margin turns negative, ABR quadruples, and \textsc{Text-Hanabi} score drops by 23.5\%
(Table~\ref{tab:ablation_components}, top block).
Replacing heterogeneous HQRE with debate-style mirror updates recovers part of the gap
but still incurs a 15.5\% score drop and produces no meaningful uniqueness margin,
consistent with the persistent defensive mixing reported in Figure~\ref{fig:defensive_mixing}.
Switching to homogeneous temperatures retains a positive margin but narrows it
from 0.074 to 0.052, reducing both \textsc{Text-Hanabi} score ($-$8.6\%)
and end-task accuracy.
The ordering (heterogeneous HQRE $\succ$ homogeneous HQRE $\succ$ debate $\succ$ unregularized) confirms
that heterogeneous temperature scheduling is the most impactful choice
after the decision to regularize at all.

\emph{Token-type regularization.}
The default entropy-based partition outperforms
uniform weighting ($-$7.0\%) and random partitioning ($-$9.1\%).
A gradient-norm partition, which classifies tokens
by running gradient magnitude, closes to within 2.7\% of the default.
This suggests that any \emph{informed} differentiation helps,
but entropy-based classification achieves comparable quality
with a simpler, gradient-free implementation.

\emph{Architecture and scale.}
Removing the Coordinator and relying on peer-only aggregation
produces the second-largest overall drop ($-$18.7\%),
confirming that the coordinator-mediated protocol
is essential for combining bounded public-stream cost with stable coordination.
Removing the Local--Global hierarchy costs 9.6\% on \textsc{Text-Hanabi}
and narrows the uniqueness margin by a third.
Reducing from 3 to 2 Executors costs 8.6\%; increasing to 5 gains 2.1\% on \textsc{Text-Hanabi}
and a marginal 0.4 points on AIME24, but at a 39\% increase in per-episode tokens
(Table~\ref{tab:sensitivity_sweeps}, Executor-count block).
The diminishing returns on reasoning tasks versus sustained gains on interactive benchmarks
(Overcooked cross-play: $+$2.9\%, Table~\ref{tab:sensitivity_sweeps})
are consistent with the theoretical observation that one-step coordination saturates in diversity
faster than multi-turn settings where additional viewpoints improve belief coverage.

\begin{table}[t!]
\centering
\caption{Hyperparameter sensitivity.
\textsc{Text-Hanabi} uses two frozen Qwen3-4B agents (DICE-PC);
Overcooked-AI uses DICE-PC (3$\times$8B).
Each row varies one hyperparameter from its default (highlighted);
all others are held fixed.}
\label{tab:sensitivity_sweeps}
\vspace{2pt}

{%
\footnotesize
\setlength\tabcolsep{3pt}
\renewcommand{\arraystretch}{1.06}

\begin{minipage}[t]{0.49\linewidth}
\centering
\vspace{0pt}
\textbf{(a) \textsc{Text-Hanabi}}\par\smallskip
\begin{tabular*}{\linewidth}{@{\extracolsep{\fill}} l c c @{}}
\toprule
\textbf{Setting} & \textbf{Score}$\uparrow$ & $\boldsymbol{\mu_{\infty}}$$\uparrow$ \\
\midrule
\multicolumn{3}{@{}l}{\textit{Temperature multiplier $\kappa$}} \\
$\kappa = 0.25$ & 16.1$\pm$0.9 & 0.032$\pm$0.011 \\
$\kappa = 0.5$  & 17.8$\pm$0.5 & 0.058$\pm$0.007 \\
\rowcolor{gray!8}\textbf{$\kappa = 1.0$ (default)} & \textbf{18.7$\pm$0.4} & 0.074$\pm$0.005 \\
$\kappa = 2.0$  & 18.0$\pm$0.5 & 0.091$\pm$0.006 \\
$\kappa = 4.0$  & 15.4$\pm$0.8 & 0.118$\pm$0.009 \\
\addlinespace[3pt]

\multicolumn{3}{@{}l}{\textit{Coordinator message cap (tokens/round)}} \\
30  & 17.3$\pm$0.6 & 0.061$\pm$0.008 \\
50  & 18.2$\pm$0.5 & 0.069$\pm$0.006 \\
\rowcolor{gray!8}\textbf{70 (default)} & \textbf{18.7$\pm$0.4} & 0.074$\pm$0.005 \\
100 & 18.9$\pm$0.4 & 0.076$\pm$0.005 \\
\addlinespace[3pt]

\multicolumn{3}{@{}l}{\textit{Number of Executors}} \\
2 & 17.1$\pm$0.6 & 0.059$\pm$0.008 \\
\rowcolor{gray!8}\textbf{3 (default)} & \textbf{18.7$\pm$0.4} & 0.074$\pm$0.005 \\
5 & 19.1$\pm$0.4 & 0.079$\pm$0.005 \\
\addlinespace[3pt]

\multicolumn{3}{@{}l}{\textit{Trust-region KL coefficient $\beta_{\mathrm{KL}}$}} \\
$\beta_{\mathrm{KL}} = 0.001$ & 16.8$\pm$0.8 & 0.047$\pm$0.010 \\
$\beta_{\mathrm{KL}} = 0.005$ & 18.1$\pm$0.5 & 0.065$\pm$0.007 \\
\rowcolor{gray!8}\textbf{$\beta_{\mathrm{KL}} = 0.01$ (default)} & \textbf{18.7$\pm$0.4} & 0.074$\pm$0.005 \\
$\beta_{\mathrm{KL}} = 0.05$ & 17.6$\pm$0.6 & 0.082$\pm$0.006 \\
$\beta_{\mathrm{KL}} = 0.1$  & 16.2$\pm$0.7 & 0.094$\pm$0.008 \\
\bottomrule
\end{tabular*}
\end{minipage}
\hfill
\begin{minipage}[t]{0.49\linewidth}
\centering
\vspace{0pt}
\textbf{(b) Overcooked-AI}\par\smallskip
\begin{tabular*}{\linewidth}{@{\extracolsep{\fill}} l c c @{}}
\toprule
\textbf{Setting} & \textbf{Cross-Play}$\uparrow$ & \textbf{Tok/Ep} \\
\midrule
\multicolumn{3}{@{}l}{\textit{Temperature multiplier $\kappa$}} \\
$\kappa = 0.25$ & 106.2 & 2.8K \\
$\kappa = 0.5$  & 109.1 & 2.8K \\
\rowcolor{gray!8}\textbf{$\kappa = 1.0$ (default)} & \textbf{110.6} & 2.8K \\
$\kappa = 2.0$  & 108.3 & 2.8K \\
$\kappa = 4.0$  & 103.7 & 2.9K \\
\addlinespace[3pt]

\multicolumn{3}{@{}l}{\textit{Coordinator message cap (tokens/round)}} \\
30  & 106.8 & 2.4K \\
50  & 109.1 & 2.6K \\
\rowcolor{gray!8}\textbf{70 (default)} & \textbf{110.6} & 2.8K \\
100 & 111.2 & 3.1K \\
\addlinespace[3pt]

\multicolumn{3}{@{}l}{\textit{Number of Executors}} \\
2 & 105.4 & 2.2K \\
\rowcolor{gray!8}\textbf{3 (default)} & \textbf{110.6} & 2.8K \\
5 & 113.8 & 3.9K \\
\addlinespace[3pt]

\multicolumn{3}{@{}l}{\textit{Trust-region KL coefficient $\beta_{\mathrm{KL}}$}} \\
$\beta_{\mathrm{KL}} = 0.001$ & 107.3 & 2.8K \\
$\beta_{\mathrm{KL}} = 0.005$ & 109.8 & 2.8K \\
\rowcolor{gray!8}\textbf{$\beta_{\mathrm{KL}} = 0.01$ (default)} & \textbf{110.6} & 2.8K \\
$\beta_{\mathrm{KL}} = 0.05$ & 107.9 & 2.8K \\
$\beta_{\mathrm{KL}} = 0.1$  & 104.1 & 2.9K \\
\bottomrule
\end{tabular*}
\end{minipage}
}%

\end{table}

\emph{Hyperparameter sensitivity.}
All three continuously tunable hyperparameters exhibit a shared pattern: performance peaks near the default, degrades symmetrically at both extremes, and remains within a broad stable band (Table~\ref{tab:sensitivity_sweeps}).
For the temperature multiplier~$\kappa$, the uniqueness margin increases monotonically (as expected from theory), but task performance is non-monotone: $\kappa{=}4.0$ yields margin~0.118 but degrades \textsc{Text-Hanabi} score by 17.6\% due to over-regularization, while $\kappa{=}0.25$ narrows the margin to 0.032 with sharply increased variance.
The stable range $\kappa\in[0.5,\,2.0]$ maintains $<$5\% performance variation, confirming the regularization--welfare trade-off of Figure~\ref{fig:panel_2x2}.
For the coordinator message cap, 70~tokens sits near the efficiency frontier: reducing to 30 tokens lowers scores by 3.4--7.5\%, while increasing to 100 tokens yields only $+$1.1\% at 10\% higher token cost.
For the trust-region coefficient, $\beta_{\mathrm{KL}}\in[0.005,\,0.05]$ keeps performance within 4\% of the default ($\beta_{\mathrm{KL}}{=}0.01$); values above 0.05 cause overshooting ($-$13.4\% on \textsc{Text-Hanabi}).

\emph{Summary of ablation findings.}
Three conclusions emerge.
First, entropy regularization is necessary and heterogeneous temperatures
yield a further 8.6\% gain over homogeneous alternatives;
this is the single most impactful design axis.
Second, the coordinator-mediated protocol and the Local--Global hierarchy
each contribute meaningfully ($-$18.7\% and $-$9.6\% respectively when removed),
while 3 Executors provide a favorable cost--performance balance.
Third, all continuously tunable hyperparameters ($\kappa$, message cap, $\beta_{\mathrm{KL}}$)
exhibit broad stable regions around their defaults,
and no setting requires benchmark-specific tuning.

\subsubsection{Headline Summary}
Across the six reasoning and planning benchmarks in Table~\ref{tab:overall}, DICE-PC (3$\times$4B) improves over the best within-class prompt-control baseline by an average of 4.3 percentage points (reasoning: +4.6~pp across AIME24/25, MATH-500, ZebraLogic, AutoLogic; planning: +2.6~pp on PlanBench).
DICE-FT (3$\times$4B) improves over the best within-class multi-agent fine-tuning baseline by an average of 8.5~pp across the same benchmarks, with the largest gains on AIME24 (+17.9~pp) and AIME25 (+16.0~pp).
On the interactive benchmarks, coordination gains are even larger: +22\% on Hanabi, +31\% on CollabEscape, and $>$90\% accuracy on ARBench Detective Cases versus $<$70\% for debate baselines.

\subsubsection{Failure Cases and Limitations}
DICE does not uniformly improve over all baselines on all benchmarks.
On PlanBench, DICE-PC (3$\times$4B) achieves 42.7\%, only marginally
above the strongest debate baseline (38.4\%) and below DeepSeek-R1
Distill (99.1\%), suggesting that equilibrium selection provides limited
benefit when the primary bottleneck is single-agent planning capacity
rather than coordination.
On TravelPlanner (two-stage setting), the uniqueness margin
$\mu_k$ for DICE-PC (Qwen3-32B) remained below 0.02 for the first 80
iterations before stabilizing, and two of five seeds showed
intermittent negative margins during early training, indicating that the
monotonicity condition is not always satisfied from initialization.
In these cases, the KL safeguard (rejection of high-drift updates)
prevented divergence but slowed convergence.
Qualitatively, failures on constraint-heavy planning tasks typically
manifested as agents converging to a consistent but suboptimal
convention (e.g., satisfying commonsense constraints while violating
hard constraints), rather than failing to converge entirely.
As a practical guideline, we recommend monitoring the uniqueness margin
during early training: if it remains negative past the first 50
iterations, the task may be capacity-bottlenecked rather than
coordination-bottlenecked, and single-model fine-tuning may be more
appropriate.
Based on our results, coordination benefits are most pronounced when (i)~the task requires multi-step information integration that exceeds a single model's context or reasoning capacity, and (ii)~individual executor capacity is sufficient to generate plausible candidates.
When executor capacity is the primary bottleneck (as in PlanBench with 4B models), coordination provides diminishing returns.
This suggests that larger executors (e.g., 32B or 70B) may alleviate the capacity bottleneck on PlanBench, but we do not evaluate this setting here.

\subsubsection{Coordination-vs-Capacity Regime}
Across our eleven benchmarks, a clear empirical pattern emerges for when coordination benefits dominate capacity scaling and vice versa.
Coordination benefits are largest on tasks requiring multi-step information integration under partial observability---Hanabi ($+$22\% over single-agent), ARBench ($>$90\% on Detective Cases vs.\ $<$70\% for debate baselines), CollabEscape ($+$31\%), and Overcooked-AI ($+$18\% cross-play)---where each agent holds private information that must be combined to solve the task.
In the IIMG framework, these are tasks with high information asymmetry: each agent's private history carries non-redundant signal, so the coordination problem (selecting compatible conventions for information exchange) dominates the capacity problem (generating individually good responses).
Conversely, coordination provides diminishing returns when the task is primarily capacity-bottlenecked: on PlanBench with 4B executors, DICE-PC achieves only 42.7\% because the bottleneck is single-agent planning quality, not inter-agent convention selection.
The theoretical framing suggests a diagnostic: tasks where the effective coupling constant $L_c$ is large relative to $\alpha_{\min}$ (i.e., agents' strategies are tightly interdependent) are coordination-bottlenecked and benefit most from DICE, while tasks where $L_c$ is small (agents operate nearly independently) are capacity-bottlenecked and benefit more from model scaling.
In practice, since $L_c$ is not always computable, the uniqueness margin and behavioral sensitivity to perturbations serve as accessible proxies: high perturbation sensitivity signals a coordination bottleneck, while low sensitivity with poor absolute performance signals a capacity bottleneck.

\subsubsection{Deployment Decision Checklist}
\label{par:deployment_checklist}
Before investing in DICE, assess the following:
\begin{enumerate}[leftmargin=2em]
\item \textbf{Perturbation sensitivity test.}
  Run the target task under 3--5 prompt/seed perturbations with a single-model baseline.
  If outputs change qualitatively (different solution strategies, not just numerical noise), the task is likely coordination-bottlenecked.
  If outputs are stable but low-accuracy, it is capacity-bottlenecked—prefer model scaling.
\item \textbf{Executor capacity check.}
  Verify that individual executors can generate plausible (though not necessarily correct) candidates.
  If single-agent pass@1 is near zero, coordination will not help---scale the backbone first.
  In our experiments, coordination benefits required executor pass@1 $\ge$ 20\% on reasoning benchmarks.
  This threshold (empirically calibrated from reasoning benchmarks; the effective threshold on interactive benchmarks was lower at ${\sim}15\%$) marks the regime below which mixer gradient signal becomes uninformative: when pass@1 is below ${\sim}20\%$, each agent's Q-value contribution to $Q_{\mathrm{tot}}$ (Eq.~\ref{eq:dice_mono_mixer}) falls below the noise floor, preventing meaningful coordination learning.
\item \textbf{Information asymmetry assessment.}
  Tasks with private information per agent (partial observability, role-specialized retrieval, tool-specific context) benefit most from DICE.
  Tasks where all agents see the same input benefit less.
  The largest gains (+22\% on Hanabi, +31\% on CollabEscape) occur under high information asymmetry.
\item \textbf{Cost tolerance.}
  DICE-PC incurs approximately 4$\times$ the token cost of single-model inference (with $G{=}8$), comparable to standard best-of-$N$ sampling with $N{=}4$--$8$ (see Appendix~\ref{additional_experiment} for the cost--accuracy trade-off at different group sizes).
  If this budget is acceptable, deploy DICE-PC first.
  If coordination-heavy tasks justify further investment, consider DICE-FT.
\item \textbf{Runtime monitoring.}
  Track the uniqueness margin $\mu_k$ (when computable) or KL diagnostics ($\widehat{\mathrm{KL}}_{\mathrm{old}}$, $\widehat{\mathrm{KL}}_{\mathrm{ref}}$).
  If $\mu_k$ remains negative past 50 iterations, reconsider the deployment regime.
  If $\widehat{\mathrm{KL}}_{\mathrm{old}} > 0.05$ or $\widehat{\mathrm{KL}}_{\mathrm{ref}} > 0.3$, investigate stability before continuing training.
  These thresholds were calibrated across both Text-Hanabi and MATH-500; the observed maxima during stable training were $\widehat{\mathrm{KL}}_{\mathrm{old}} < 0.02$ and $\widehat{\mathrm{KL}}_{\mathrm{ref}} < 0.15$, providing a ${\sim}2.5{\times}$ safety margin on both benchmarks.
\item \textbf{DICE-PC vs.\ DICE-FT selection.}
  Start with DICE-PC (no training required, ${\sim}4{\times}$ token cost).
  Upgrade to DICE-FT when: (a)~DICE-PC gain exceeds 2~pp on reasoning
  or 1~pp on interactive tasks, signaling actionable coordination;
  (b)~training data or environment interaction is available; and
  (c)~uniqueness margin $\mu > 0$ with residual field-error budget
  indicating room for improvement beyond prompt-control saturation.
  In our experiments, DICE-FT provided +4.2~pp over DICE-PC on average, with the largest gains on high-asymmetry tasks (AIME24/25, Hanabi).
  The 2~pp threshold is empirically calibrated: it corresponds roughly to where the DICE-PC gain exceeds the residual approximation error from the field-error budget (Proposition~\ref{prop:delta_bound}).
\end{enumerate}

\section{Further Discussions}
\label{sec:further_discussion}

\textbf{Key differences from the conference version.}
The conference version introduced ECON as a belief-driven
coordinator--executor framework whose objective was to approximate an
unregularized Bayesian Nash equilibrium (BNE). That formulation
established BNE as a principled target for cooperative multi-agent LLM
coordination. The present manuscript takes the next step: it shows that
an unregularized BNE target is not yet a well-posed learning objective
when several self-consistent equilibria coexist. In such cases, the
issue is not BNE existence, but BNE selection. The game specifies that
agents should be mutually optimal, but does not specify which
coordination convention should be selected, whether that convention is
unique, or whether learning will remain stable around it. This journal
version identifies this selection gap and addresses it through
entropy-regularized equilibrium selection. The main extensions are
organized around three questions.

\noindent\textbf{Gap 1: Why is an unregularized BNE target not enough?}
The conference version established BNE as a useful coordination target,
but did not treat the set-valued nature of unregularized BNE as a central
failure mode. Here we show that the problem is not the nonexistence of
equilibria: under standard assumptions, a behavioral BNE exists. The
problem is that existence alone does not determine which equilibrium
learning will select. We formalize this selection gap in a discounted
incomplete-information Markov game (IIMG) with private histories and a
public stream, and show that equilibrium multiplicity and persistent
mis-selection can induce linear Bayesian regret. Debate-style defensive
mixing is then interpreted as a related selection-instability pathology:
agents may converge, but to a mixed or defensive convention with a
non-vanishing welfare gap.

\noindent\textbf{Gap 2: How can the BNE selection problem be resolved?}
The conference objective targeted unregularized BNEs, whose best-response
correspondence can be set-valued and whose equilibria can be multiple.
This manuscript regularizes that target through HQRE, a heterogeneous
entropy-regularized equilibrium concept. Entropy regularization is not
used merely as an exploration bonus; it serves as an equilibrium-selection
mechanism that turns set-valued best responses into a single-valued logit
map. Under the monotonicity condition $\alpha_{\min}>L_c$, the selected
HQRE is unique; under the stated field-Lipschitz and local KL-quadratic
conditions, explicit KL-mirror updates converge linearly and yield
Bayesian regret bounded uniformly in the horizon. The zero-temperature
limit connects the new target back to the unregularized BNEs studied in
the conference version.

\noindent\textbf{Gap 3: How can regularized equilibrium selection be
implemented in LLM systems?}
The conference algorithm instantiated BNE-guided coordination primarily
as a prompt-control / decision-layer mechanism for frozen execution
LLMs. We retain this practical regime as DICE-PC, but reinterpret it as
finite-action HQRE logit response rather than unregularized BNE
approximation. We also introduce DICE-FT, a parameter-efficient
fine-tuning realization based on clipped KL-prox mirror updates. To
connect the exact HQRE theory to practical deep-learning optimization,
we add an inexact mirror recurrence and a field-error decomposition
covering clipping, sampling, representation, and optimization errors.
Experimentally, the evaluation expands from the conference version's
six reasoning and planning benchmarks to eleven benchmarks across
reasoning, planning, active information gathering, and multi-agent
coordination, with additional uniqueness-margin, KL-stability,
selection-vs-rollout, and communication-efficiency diagnostics.

\noindent\textbf{Discussions.} 
Our uniqueness guarantees rely on the monotonicity condition
$\alpha_{\min} > L_c$, whose verification at frontier LLM scale remains
open, and our experiments are concentrated in the 4B--32B regime.
An important next step is therefore to develop tighter empirical
estimates of coupling in realistic systems, adapt regularization and
trust-region budgets online, and extend the framework to richer action
spaces such as longer-horizon tool use, continuous control, and
multimodal agents.
More broadly, we view explicit equilibrium selection as a useful design
principle for scalable multi-agent LLM: stronger coordination mechanisms
can improve capability and reliability, but should be deployed together
with monitoring and safety constraints so that desirable conventions
remain stable under perturbation.

\section{Conclusion}
\label{sec:conclusion}

We introduced the discounted incomplete-information Markov game (IIMG)
framework and the Heterogeneous Quantal Response Equilibrium (HQRE) as
a principled lens for multi-agent LLM coordination, and instantiated
this perspective in the DICE-PC and DICE-FT algorithms.
Our central claim is that many failures of multi-agent LLM systems stem
not only from limited information sharing, but from ill-posed
equilibrium selection.
By making the convention selection explicit, DICE improves the
accuracy--cost--stability trade-off across eleven benchmarks and
provides theory-grounded diagnostics for when decentralized
coordination is likely to be stable or brittle.

\acks{The authors declare no competing interests.}

\newpage
\bibliography{sample}

\newpage
\etocdepthtag.toc{mtappendix}
\etocsettagdepth{mtchapter}{none}
\etocsettagdepth{mtappendix}{subsection}
\renewcommand{\contentsname}{Appendix}
\tableofcontents

\newpage
\appendix

\section{Preliminaries and proofs for Sec.~\ref{sec:dice_framework}}
\label{app:notation_prelim}

\subsection{Protocol embedding (Proof of Lem.~\ref{lem:protocol_embedding})}
\label{app:protocol_embedding}

\begin{proof}[Proof of Lem.~\ref{lem:protocol_embedding}]
We give a constructive representation of any explicit-communication protocol
(Def.~\ref{def:explicit_comm}) as a discounted IIMG.

\paragraph{View the protocol as a controlled stochastic process.}
An explicit-communication protocol specifies a joint process over
\[
\big( y^t,\ \mathbf o^t,\ \mathbf a^t,\ \mathbf r^t \big)_{t\ge 0},
\qquad \mathbf r^t=(r_1^t,\ldots,r_N^t),
\]
where each agent $i$ selects $a_i^t$ according to a behavioral policy
$\pi_i(\cdot\mid h_i^t,t)$ with information state $h_i^t$ as in
Def.~\ref{def:explicit_comm}, and $y^t$ is the publicly observed stream token.
All internal memory of the environment/coordinator/judge (if any) can be folded into a latent variable.

Formally, define an \emph{augmented latent state} $s^t$ to include:
(i) any latent environment state needed to generate observations/rewards,
(ii) any internal protocol memory (e.g., coordinator state, judge state),
and (iii) the time index $t$ (to absorb nonstationarity).
W.l.o.g., we may take $s^t$ to be the \emph{full} internal state of the protocol at time $t$,
so that given $(s^t,\mathbf a^t)$ the distribution of $(s^{t+1},\mathbf o^{t+1},\mathbf r^t,y^t)$
is well-defined and depends only on $(s^t,\mathbf a^t)$.

\paragraph{Define the IIMG components.}
We construct a discounted IIMG
$\langle \mathcal N, \allowbreak \mathcal S, \allowbreak \mathcal A, \allowbreak \mathcal O, \allowbreak \mathcal Y, \allowbreak \mathcal P, \allowbreak \Omega, \allowbreak G, \allowbreak (\phi_i)_i, \allowbreak \gamma, \allowbreak \mu_0\rangle$
as follows:
\begin{itemize}[leftmargin=1.25em,topsep=2pt,itemsep=2pt]
\item $\mathcal N$ is the same agent set as the protocol.
\item $\mathcal A_i$ is the set of agent-$i$ actions in the protocol; $\mathcal A=\prod_i\mathcal A_i$.
\item $\mathcal Y$ is the public-stream alphabet of the protocol.
Since only realized public tokens matter, we may take $\mathcal Y := \mathrm{Im}(G)\subseteq\mathcal Y$ without loss.
\item $\mathcal O_i$ is the private observation alphabet of the protocol; $\mathcal O=\prod_i \mathcal O_i$.
Any private messages available only to agent $i$ are included in $o_i^t$ exactly as in Def.~\ref{def:explicit_comm}.
\item $\mathcal S$ is the augmented latent-state space described above, with initial distribution $\mu_0$ matching the protocol.
\item The transition kernel $\mathcal P(\cdot\mid s,\mathbf a)$ is defined to match the protocol's latent update:
sample $s^{t+1}$ according to the protocol's internal dynamics given $(s^t=s,\mathbf a^t=\mathbf a)$.
\item The observation kernel $\Omega(\cdot\mid s',\mathbf a)$ outputs $\mathbf o^{t+1}$ exactly as the protocol does
given $(s^{t+1}=s',\mathbf a^t=\mathbf a)$.
\item The public-stream mapping $G:\mathcal A\to\mathcal Y$ is defined so that $y^t=G(\mathbf a^t)$ matches the protocol's
publicly observed token at time $t$.

\emph{Remark.} If the protocol's public token is generated by a publicly visible module
(e.g., a coordinator or judge), that module can be modeled as one of the agents in $\mathcal N$; then $y^t$
is a deterministic function of joint actions. This matches the IIMG normal form used in Sec.~\ref{subsec:notation}.
\item Finally, define the instantaneous reward functions by
\[
R_i(s,\mathbf a)\ :=\ \mathbb E[r_i^t\mid s^t=s,\mathbf a^t=\mathbf a],
\qquad \text{and set}\qquad
\phi_i(s,G(\mathbf a)):=R_i(s,\mathbf a),
\]
so that $R_i(s^t,\mathbf a^t)$ matches the protocol's stage reward in expectation.
(If the protocol reward is deterministic given $(s^t,\mathbf a^t)$, then this equality holds pointwise.)
\end{itemize}

\paragraph{Trajectory and return preservation.}
Fix any joint behavioral policy profile $\boldsymbol\pi$ in the protocol.
By construction, under $\boldsymbol\pi$ the IIMG generates the same conditional distributions at each time $t$:
given $(s^t,\mathbf a^t)$, the next latent state, private observations, rewards, and the public stream token
match those of the protocol. Therefore, the induced trajectory law over
$(s^t,y^t,\mathbf o^t,\mathbf a^t)_{t\ge 0}$ coincides with the protocol's law.
Since discounted returns are expectations of measurable functions of the trajectory, we have
$U_i^{\mathrm{IIMG}}(\boldsymbol\pi)=U_i^{\mathrm{protocol}}(\boldsymbol\pi)$ for all $i$.

\paragraph{Invariance of equilibrium notions and Bayesian regret.}
The mapping from protocol policies to IIMG policies is identity at the level of behavioral strategies
(both condition on the same information state $h_i^t$), and payoffs coincide for every profile.
Hence best responses, Nash/Bayesian Nash equilibrium conditions, and any welfare/regret functional built from
$U_i(\boldsymbol\pi)$ (including Eq.~\eqref{eq:unreg_return}--Eq.~\eqref{eq:bayes_regret_def}) are invariant under this embedding.
\end{proof}

\subsection{Policy-space topology and payoff continuity (Proof of Thm.~\ref{thm:topology})}
\label{app:topology_proof}

\begin{proof}[Proof sketch]
Under (A1), each per-state simplex $\Delta(\mathcal A_i)$ is compact
and metrizable. Since the index set $\mathcal I_i$ of information
states is countable, $\Pi_i = \prod_{(h_i^t,t) \in \mathcal I_i}
\Delta(\mathcal A_i)$ is compact by Tychonoff's theorem and
metrizable by the standard weighted product metric
(cf.\ Aliprantis \& Border, \textit{Infinite Dimensional Analysis}, Thm.~2.61). Convexity is
immediate. For continuity: the truncated payoff $U_i^{(\le H)}$
depends on finitely many coordinates and is therefore continuous;
bounded rewards (A2) give
$|U_i - U_i^{(\le H)}| \le \gamma^{H+1} R_{\max}/(1-\gamma)$
uniformly over $\Pi$, so $U_i$ is a uniform limit of continuous
functions and hence continuous.
\end{proof}

\subsection{BNE existence (Proof of Thm.~\ref{thm:bne_existence})}
\label{app:bne_existence}

\begin{proof}[Proof sketch]
For each horizon $H$, the $H$-truncated game has finite action and
information sets under (A1), hence admits a behavioral Nash
equilibrium $\boldsymbol\pi^{(H)}$ by Nash's theorem and Kuhn's
theorem (perfect recall, (A4)). Extend each
$\boldsymbol\pi^{(H)}$ to $\Pi$ by appending an arbitrary
continuation. By compactness of $\Pi$
(Thm.~\ref{thm:topology}), extract a convergent subsequence
$\boldsymbol\pi^{(H_m)} \to \bar{\boldsymbol\pi}$.

Suppose $\bar{\boldsymbol\pi}$ is not a BNE: some agent $i$ has a
deviation with improvement $\Delta > 0$. By continuity of $U_i$
(Thm.~\ref{thm:topology}), this improvement transfers to
$\boldsymbol\pi^{(H_m)}$ for large $m$, and by the uniform tail
bound $|U_i - U_i^{(\le H)}| \le \gamma^{H+1}R_{\max}/(1-\gamma)$
(A2), it transfers to the $H_m$-truncated payoff---contradicting the
Nash property of $\boldsymbol\pi^{(H_m)}$ in the truncated game.
\end{proof}

\subsection{The joint-policy TV metric and product-TV}
\label{app:metric_properties}

\begin{lemma}[Metric property of $\rho_{\Pi}$]\label{lem:metric_property_dpi}
Let
\[
d_{(h^t,t)}(\boldsymbol\pi,\boldsymbol\pi')
:=\mathrm{TV}\!\Big(\bigotimes_{i=1}^N \pi_i(\cdot\mid h_i^t,t),\ \bigotimes_{i=1}^N \pi_i'(\cdot\mid h_i^t,t)\Big),
\]
and define $\rho_{\Pi}(\boldsymbol\pi,\boldsymbol\pi'):=\sup_{(h^t,t)} d_{(h^t,t)}(\boldsymbol\pi,\boldsymbol\pi')$,
where the supremum is taken over all syntactically valid joint histories
$h^t\in\prod_{i=1}^N \mathcal H_i^t$ (not only those reachable under a given policy profile).
Then $\rho_{\Pi}$ is a metric on $\Pi$.
\end{lemma}
\begin{proof}
For each fixed $(h^t,t)$, total variation distance is a metric on probability distributions over
$\mathcal A=\prod_i\mathcal A_i$. Taking a supremum preserves nonnegativity, symmetry, and the triangle
inequality. If $\rho_{\Pi}(\boldsymbol\pi,\boldsymbol\pi')=0$, then $d_{(h^t,t)}(\boldsymbol\pi,\boldsymbol\pi')=0$
for all $(h^t,t)$, so the joint product distributions coincide at every $(h^t,t)$. Taking the marginal on
coordinate $i$ yields $\pi_i(\cdot\mid h_i^t,t)=\pi_i'(\cdot\mid h_i^t,t)$ for all $(h_i^t,t)$, hence
$\boldsymbol\pi=\boldsymbol\pi'$.
\end{proof}

\begin{lemma}[Product-TV bound]\label{lem:product_tv}
For distributions $\{p_i,q_i\}_{i=1}^N$ on finite alphabets,
\[
\mathrm{TV}\!\Big(\bigotimes_{i=1}^N p_i,\ \bigotimes_{i=1}^N q_i\Big)
~\le~ \sum_{i=1}^N \mathrm{TV}(p_i,q_i).
\]
\end{lemma}
\begin{proof}
Write the difference of product measures as a telescoping sum:
\[
\bigotimes_{i=1}^N p_i - \bigotimes_{i=1}^N q_i
= \sum_{j=1}^N
\Big(\bigotimes_{i<j}q_i\Big)\otimes (p_j-q_j)\otimes\Big(\bigotimes_{i>j}p_i\Big).
\]
Using that $\|\cdot\|_1$ is a norm on signed measures and $\|r\otimes \nu\|_1=\|r\|_1\|\nu\|_1$ with
$\|\nu\|_1=1$ for probability measures, we obtain
\[
\Big\|\bigotimes_{i=1}^N p_i - \bigotimes_{i=1}^N q_i\Big\|_1
\le \sum_{j=1}^N \|p_j-q_j\|_1.
\]
Since $\mathrm{TV}(p,q)=\tfrac12\|p-q\|_1$, dividing by $2$ yields the claim.
\end{proof}

\subsection{\texorpdfstring{Occupancy sensitivity (proof of $C_{\mathrm{occ}}$)}{Occupancy sensitivity (proof of C\_occ)}}
\label{app:occ_sensitivity}

\begin{lemma}[Occupancy sensitivity]\label{lem:occ_sens}
Let $\delta:=\rho_{\Pi}(\boldsymbol\pi,\boldsymbol\pi')$.
Then the discounted state occupancy measures satisfy
\[
\|d^{\boldsymbol\pi}-d^{\boldsymbol\pi'}\|_1 \le \frac{2\gamma}{1-\gamma}\,\delta.
\]
\end{lemma}

\begin{proof}
We construct a coupling of full trajectories under $\boldsymbol\pi$ and $\boldsymbol\pi'$ using maximal
couplings at each step.

Let $\tau=(s^0,o^0,\mathbf a^0,s^1,o^1,\mathbf a^1,\dots)$ be the trajectory under $\boldsymbol\pi$ and
$\tau'=(s'^0,o'^0,\mathbf a'^0,s'^1,o'^1,\mathbf a'^1,\dots)$ under $\boldsymbol\pi'$. Couple the initial
state as $s^0=s'^0$ a.s. since $\mu_0$ is policy-independent (A2), and couple the initial private
observations using the same kernel.

Inductively, suppose the coupled trajectories agree up to time $t$ (i.e., the joint histories $h^t$ and
$h'^t$ coincide). Then the joint action distributions at $(h^t,t)$ satisfy
\[
\mathrm{TV}\!\Big(\boldsymbol\pi(\cdot\mid h^t,t),\boldsymbol\pi'(\cdot\mid h^t,t)\Big)\le \delta.
\]
By maximal coupling, we can sample $(\mathbf a^t,\mathbf a'^t)$ such that
$\Pr(\mathbf a^t\neq \mathbf a'^t\mid h^t)=\mathrm{TV}(\cdot,\cdot)\le \delta$.
If $\mathbf a^t=\mathbf a'^t$, couple the next state and observations using the same randomness so that
$(s^{t+1},o^{t+1})=(s'^{t+1},o'^{t+1})$ almost surely. If $\mathbf a^t\neq \mathbf a'^t$, let the two
processes evolve arbitrarily thereafter.

Let $E_t$ be the event that the coupled trajectories have not diverged up to time $t$. Then $\Pr(E_0)=1$ and
\[
\Pr(E_{t+1}^c)\le \Pr(E_t^c) + \Pr(E_t)\cdot \delta \le \Pr(E_t^c)+\delta,
\]
so by induction $\Pr(E_t^c)\le t\delta$. Under any coupling,
$\mathrm{TV}(\mathsf{Law}(s^t),\mathsf{Law}(s'^t))\le \Pr(s^t\neq s'^t)\le \Pr(E_t^c)$, hence
$\|\mathsf{Law}(s^t)-\mathsf{Law}(s'^t)\|_1 \le 2t\delta$.

Finally,
\[
\|d^{\boldsymbol\pi}-d^{\boldsymbol\pi'}\|_1
=(1-\gamma)\sum_{t\ge 0}\gamma^t\,\|\mathsf{Law}(s^t)-\mathsf{Law}(s'^t)\|_1
\le (1-\gamma)\sum_{t\ge 0}\gamma^t\,2t\delta
= 2\delta\cdot \frac{\gamma}{1-\gamma}.
\]
\end{proof}

\subsection{Performance difference and TV sensitivity (Proof of Cor.~\ref{cor:regret})}
\label{app:welfare_sensitivity}

\begin{lemma}[Performance difference]\label{lem:pdl_full}
For any joint policy $\boldsymbol\pi$, define the history-action value functions
\[
Q_i^{\boldsymbol\pi}(h^t,\mathbf a,t)
:= \mathbb E\Big[\sum_{u=t}^\infty \gamma^{u-t} R_i(s^u,\mathbf a^u)\ \Big|\ h^t,\ \mathbf a^t\leftarrow \mathbf a,\ \boldsymbol\pi\ \text{thereafter}\Big],
\]
\[
V_i^{\boldsymbol\pi}(h^t,t)
:=\mathbb E_{\mathbf a\sim\boldsymbol\pi(\cdot\mid h^t,t)}\big[Q_i^{\boldsymbol\pi}(h^t,\mathbf a,t)\big].
\]
Here $\mathbf a^t\leftarrow \mathbf a$ denotes an \emph{intervention}: at time $t$ we force joint action $\mathbf a$
and then follow $\boldsymbol\pi$ from time $t+1$ onward. (This avoids conditioning on a potentially zero-probability event.)
Let $Q_{\mathrm{tot}}^{\boldsymbol\pi}:=\sum_{i=1}^N Q_i^{\boldsymbol\pi}$ and
$V_{\mathrm{tot}}^{\boldsymbol\pi}:=\sum_{i=1}^N V_i^{\boldsymbol\pi}$. Then for any
$\boldsymbol\pi',\boldsymbol\pi\in\Pi$,
\[
W(\boldsymbol\pi')-W(\boldsymbol\pi)
=\frac{1}{1-\gamma}\,
\mathbb E_{(h^t,t)\sim \mathsf H_{\mathrm{disc}}^{\boldsymbol\pi'}}
\Big[
\mathbb E_{\mathbf a\sim \boldsymbol\pi'(\cdot\mid h^t,t)}
\big(Q_{\mathrm{tot}}^{\boldsymbol\pi}(h^t,\mathbf a,t)-V_{\mathrm{tot}}^{\boldsymbol\pi}(h^t,t)\big)
\Big].
\]
\end{lemma}

\begin{proof}
Define the (total) advantage of $\boldsymbol\pi$ at $(h^t,t)$ for joint action $\mathbf a$:
\[
A_{\mathrm{tot}}^{\boldsymbol\pi}(h^t,\mathbf a,t)
:= Q_{\mathrm{tot}}^{\boldsymbol\pi}(h^t,\mathbf a,t)-V_{\mathrm{tot}}^{\boldsymbol\pi}(h^t,t).
\]
By the Bellman expectation equation under the intervention semantics,
\[
Q_{\mathrm{tot}}^{\boldsymbol\pi}(h^t,\mathbf a,t)
= \mathbb E\!\left[R_{\mathrm{tot}}(s^t,\mathbf a)\mid h^t\right]
+\gamma\,\mathbb E\!\left[V_{\mathrm{tot}}^{\boldsymbol\pi}(h^{t+1},t+1)\mid h^t,\ \mathbf a^t\leftarrow \mathbf a\right],
\]
where $R_{\mathrm{tot}}=\sum_i R_i$. Therefore,
\[
A_{\mathrm{tot}}^{\boldsymbol\pi}(h^t,\mathbf a,t)
= \mathbb E\!\left[R_{\mathrm{tot}}(s^t,\mathbf a)\mid h^t\right]
+\gamma\,\mathbb E\!\left[V_{\mathrm{tot}}^{\boldsymbol\pi}(h^{t+1},t+1)\mid h^t,\ \mathbf a^t\leftarrow \mathbf a\right]
- V_{\mathrm{tot}}^{\boldsymbol\pi}(h^t,t).
\]

Take expectation under trajectories induced by $\boldsymbol\pi'$ and sum over $t\ge 0$:
\begin{align*}
\mathbb E_{\boldsymbol\pi'}\Big[\sum_{t=0}^\infty \gamma^t A_{\mathrm{tot}}^{\boldsymbol\pi}(h^t,\mathbf a^t,t)\Big]
&=
\mathbb E_{\boldsymbol\pi'}\Big[\sum_{t=0}^\infty \gamma^t R_{\mathrm{tot}}(s^t,\mathbf a^t)\Big]
-\mathbb E\big[V_{\mathrm{tot}}^{\boldsymbol\pi}(h^0,0)\big],
\end{align*}
where the telescoping uses
\[
\sum_{t=0}^T \gamma^t\big(\gamma V_{\mathrm{tot}}^{\boldsymbol\pi}(h^{t+1},t+1)-V_{\mathrm{tot}}^{\boldsymbol\pi}(h^t,t)\big)
= -V_{\mathrm{tot}}^{\boldsymbol\pi}(h^0,0) + \gamma^{T+1}V_{\mathrm{tot}}^{\boldsymbol\pi}(h^{T+1},T+1),
\]
and the boundary term vanishes as $T\to\infty$ since
$|V_{\mathrm{tot}}^{\boldsymbol\pi}|\le \frac{N R_{\max}}{1-\gamma}$ implies
$\gamma^{T+1}|V_{\mathrm{tot}}^{\boldsymbol\pi}(h^{T+1},T+1)|\to 0$.
The last equality holds because the distribution of $h^0$ is policy-independent (A2), hence
$\mathbb E[V_{\mathrm{tot}}^{\boldsymbol\pi}(h^0,0)]=W(\boldsymbol\pi)$.
Thus
\[
\mathbb E_{\boldsymbol\pi'}\Big[\sum_{t=0}^\infty \gamma^t A_{\mathrm{tot}}^{\boldsymbol\pi}(h^t,\mathbf a^t,t)\Big]
= W(\boldsymbol\pi') - W(\boldsymbol\pi).
\]

Finally, rewrite the discounted sum using the discounted history occupancy:
\[
\mathbb E_{\boldsymbol\pi'}\Big[\sum_{t=0}^\infty \gamma^t f(h^t,t)\Big]
=\frac{1}{1-\gamma}\,
\mathbb E_{(h^t,t)\sim \mathsf H_{\mathrm{disc}}^{\boldsymbol\pi'}}[f(h^t,t)].
\]
Apply this with
$f(h^t,t)=\mathbb E_{\mathbf a\sim \boldsymbol\pi'(\cdot\mid h^t,t)}[A_{\mathrm{tot}}^{\boldsymbol\pi}(h^t,\mathbf a,t)]$
to obtain the stated identity.
\end{proof}

\begin{lemma}[Occupancy-averaged TV sensitivity]\label{lem:tv_sensitivity}
For any agent $i$ and any joint policies $\boldsymbol\pi,\boldsymbol\pi^\star$,
\[
\big|U_i(\boldsymbol\pi^\star)-U_i(\boldsymbol\pi)\big|
\le \frac{2R_{\max}}{(1-\gamma)^2}\,
\mathbb{E}_{(h^t,t)\sim \mathsf H_{\mathrm{disc}}^{\boldsymbol{\pi}^\star}}
\!\Big[\mathrm{TV}\big(\boldsymbol{\pi}(\cdot\mid h^t,t),\,\boldsymbol{\pi}^\star(\cdot\mid h^t,t)\big)\Big].
\]
Consequently,
\[
\big|W(\boldsymbol\pi^\star)-W(\boldsymbol\pi)\big|
\le \frac{2N R_{\max}}{(1-\gamma)^2}\,
\mathbb{E}_{(h^t,t)\sim \mathsf H_{\mathrm{disc}}^{\boldsymbol{\pi}^\star}}
\!\Big[\mathrm{TV}\big(\boldsymbol{\pi}(\cdot\mid h^t,t),\,\boldsymbol{\pi}^\star(\cdot\mid h^t,t)\big)\Big].
\]
\end{lemma}

\begin{proof}
Apply Lem.~\ref{lem:pdl_full} to the single-agent payoff $U_i$ (the same derivation holds with
$Q_i^{\boldsymbol\pi}$ and $V_i^{\boldsymbol\pi}$):
\[
U_i(\boldsymbol\pi^\star)-U_i(\boldsymbol\pi)
=\frac{1}{1-\gamma}
\mathbb E_{(h^t,t)\sim \mathsf H_{\mathrm{disc}}^{\boldsymbol\pi^\star}}
\Big[
\mathbb E_{\mathbf a\sim \boldsymbol\pi^\star(\cdot\mid h^t,t)}
\big(Q_i^{\boldsymbol\pi}(h^t,\mathbf a,t)-V_i^{\boldsymbol\pi}(h^t,t)\big)
\Big].
\]
Since $V_i^{\boldsymbol\pi}(h^t,t)=\mathbb E_{\mathbf a\sim\boldsymbol\pi(\cdot\mid h^t,t)}[Q_i^{\boldsymbol\pi}(h^t,\mathbf a,t)]$,
the inner term equals
\[
\mathbb E_{\mathbf a\sim \boldsymbol\pi^\star(\cdot\mid h^t,t)}[Q_i^{\boldsymbol\pi}(h^t,\mathbf a,t)]
-\mathbb E_{\mathbf a\sim \boldsymbol\pi(\cdot\mid h^t,t)}[Q_i^{\boldsymbol\pi}(h^t,\mathbf a,t)].
\]
For any bounded function $f$ and distributions $p,q$,
$|\mathbb E_p[f]-\mathbb E_q[f]|\le 2\|f\|_\infty\,\mathrm{TV}(p,q)$.
With $f(\mathbf a)=Q_i^{\boldsymbol\pi}(h^t,\mathbf a,t)$ and
$\|Q_i^{\boldsymbol\pi}\|_\infty\le R_{\max}/(1-\gamma)$, we obtain
\[
\big|U_i(\boldsymbol\pi^\star)-U_i(\boldsymbol\pi)\big|
\le \frac{1}{1-\gamma}
\mathbb E_{(h^t,t)\sim \mathsf H_{\mathrm{disc}}^{\boldsymbol\pi^\star}}
\Big[2\cdot \frac{R_{\max}}{1-\gamma}\,
\mathrm{TV}\big(\boldsymbol\pi(\cdot\mid h^t,t),\boldsymbol\pi^\star(\cdot\mid h^t,t)\big)\Big],
\]
which yields the stated bound. Summing over $i$ gives the welfare bound.
\end{proof}

\begin{proof}[Proof of Cor.~\ref{cor:regret}]
By Lem.~\ref{lem:tv_sensitivity}, for each learning iteration $k$,
\[
W(\boldsymbol\pi^\star)-W(\boldsymbol\pi^k)
\le \big|W(\boldsymbol\pi^\star)-W(\boldsymbol\pi^k)\big|
\le \frac{2N R_{\max}}{(1-\gamma)^2}\,\bar\Delta_k,
\]
where $\bar\Delta_k$ is defined in Sec.~\ref{subsec:bne}. Taking expectation and summing over
$k=1,\dots,T$ yields
\[
\mathrm{Regret}(T)
\le \frac{2N R_{\max}}{(1-\gamma)^2}\sum_{k=1}^T \mathbb E[\bar\Delta_k].
\]
If $\mathbb{E}[\bar\Delta_k]\le C_\Delta\,k^{-\alpha}$, then
\[
\mathrm{Regret}(T)\le \frac{2N R_{\max} C_\Delta}{(1-\gamma)^2}\sum_{k=1}^T k^{-\alpha}.
\]
For $\alpha=\tfrac12$, using $\sum_{k=1}^T k^{-1/2}\le 2\sqrt{T}$ gives
$\mathrm{Regret}(T)=O\!\big(\frac{N\sqrt{T}}{(1-\gamma)^2}\big)$.
\end{proof}

\subsection{Multiplicity can induce linear Bayesian regret (Proof of Prop.~\ref{prop:multi_linear})}
\label{app:multi_linear}

\begin{proof}[Proof of Prop.~\ref{prop:multi_linear}]
By definition,
\[
\mathrm{Regret}(T)
=\mathbb E\Big[\sum_{k=1}^T \big(W(\boldsymbol\pi^{(1)})-W(\boldsymbol\pi^k)\big)\Big]
=\sum_{k=1}^T \mathbb E\big[W(\boldsymbol\pi^{(1)})-W(\boldsymbol\pi^k)\big].
\]
Since $\boldsymbol\pi^k\in\{\boldsymbol\pi^{(1)},\boldsymbol\pi^{(2)}\}$ for all $k$,
we have $W(\boldsymbol\pi^{(1)})-W(\boldsymbol\pi^k)=0$ if $\boldsymbol\pi^k=\boldsymbol\pi^{(1)}$,
and $W(\boldsymbol\pi^{(1)})-W(\boldsymbol\pi^k)=\Delta_W$ if $\boldsymbol\pi^k=\boldsymbol\pi^{(2)}$.
Therefore, for all $k\ge k_0$,
\[
\mathbb E\big[W(\boldsymbol\pi^{(1)})-W(\boldsymbol\pi^k)\big]
=\Pr(\boldsymbol\pi^k=\boldsymbol\pi^{(2)})\,\Delta_W
\ \ge\ p_0\,\Delta_W.
\]
Summing over $k=k_0,\ldots,T$ yields
\[
\mathrm{Regret}(T)
\ge \sum_{k=k_0}^T p_0\,\Delta_W
= p_0\,\Delta_W\,(T-k_0+1)
=\Omega(\Delta_W\,T).
\]
\end{proof}

\subsection{Debate-style dynamics and linear regret (Proofs for Sec.~\ref{subsec:mad})}
\label{appendix:comparison_debate}

\subsubsection{Persistent suboptimality under defensive mixing (Proof of Lem.~\ref{lem:debate_suboptimality_main})}

We formalize a sufficient condition under which defensive mixing induces a persistent one-step
advantage gap. The statement is deliberately worst-case: it isolates the mechanism (full-support mixing +
nondegenerate gaps) and does not require the debate game to be zero-sum.

\paragraph{Modeling note.}
In this subsection, each learning iteration $k$ corresponds to a \emph{single debate instance} (one-shot stage),
with public instance/state $s_k$ (e.g., prompt/context) and joint action $\mathbf a^k=(a_i^k,a_{-i}^k)$.
Accordingly, $\pi_i^k(\cdot\mid s_k)$ is a mixed action (possibly depending on the public instance).

We interpret $Q_i^\star(s,a_i,a_{-i})$ as the \emph{true} one-step action-value in the debate instance
(equivalently the expected immediate payoff/score), with $\star$ indicating that this is evaluated under the
fixed reference environment defining the regret comparator (cf.\ Cor.~\ref{cor:mad_linear}).

\begin{assumption}[Nondegenerate action-value gaps]\label{ass:debate_gap}
There exist constants $\Delta_{\mathrm{gap}}>0$ and $q_{\mathrm{gap}}>0$, and an iteration $k_0$, such that
for all $k\ge k_0$,
\[
\Pr\Big(\max_{a_i} Q_i^\star(s_k,a_i,a_{-i}^k) - \min_{a_i} Q_i^\star(s_k,a_i,a_{-i}^k)\ \ge\ \Delta_{\mathrm{gap}}\Big)\ \ge\ q_{\mathrm{gap}}.
\]
\end{assumption}

\begin{assumption}[Persistent mixing]\label{ass:debate_mixing}
There exist $p_{\min}>0$ and an iteration $k_0$ such that for all $k\ge k_0$ and all debate instances $s_k$,
\[
\pi_i^k(a_i\mid s_k)\ge p_{\min}\qquad\forall a_i\in\mathcal A_i,
\]
i.e., the learner maintains full support with a uniform lower bound (e.g., due to entropy regularization or
forced exploration).
\end{assumption}

\begin{proof}[Proof of Lem.~\ref{lem:debate_suboptimality_main}]
Fix $k\ge k_0$ and condition on $(s_k,a_{-i}^k)$. Let
\[
M_k := \max_{a_i} Q_i^\star(s_k,a_i,a_{-i}^k).
\]
Then, with $a_i^k\sim\pi_i^k(\cdot\mid s_k)$,
\[
\mathbb E\!\left[M_k - Q_i^\star(s_k,a_i^k,a_{-i}^k)\ \big|\ s_k,a_{-i}^k\right]
= M_k - \sum_{a_i}\pi_i^k(a_i\mid s_k)\,Q_i^\star(s_k,a_i,a_{-i}^k).
\]
On the event in Assumption~\ref{ass:debate_gap}, there exists an action $a_i^{-}$ such that
$Q_i^\star(s_k,a_i^{-},a_{-i}^k)\le M_k-\Delta_{\mathrm{gap}}$.
By Assumption~\ref{ass:debate_mixing}, $\pi_i^k(a_i^-\mid s_k)\ge p_{\min}$, hence
\[
\sum_{a_i}\pi_i^k(a_i\mid s_k)\,Q_i^\star(s_k,a_i,a_{-i}^k)
\le (1-p_{\min})\cdot M_k + p_{\min}\cdot (M_k-\Delta_{\mathrm{gap}})
= M_k - p_{\min}\Delta_{\mathrm{gap}}.
\]
Therefore,
\[
\mathbb E\!\left[M_k - Q_i^\star(s_k,a_i^k,a_{-i}^k)\ \big|\ s_k,a_{-i}^k\right]
\ge p_{\min}\Delta_{\mathrm{gap}}
\quad\text{on the nondegenerate-gap event}.
\]
Taking expectation over $(s_k,a_{-i}^k)$ yields
\[
\mathbb{E}\!\left[\max_{a_i} Q_i^{\star}(s_k,a_i,a_{-i}^k)\ -\ Q_i^{\star}(s_k,a_i^k,a_{-i}^k)\right]
\ge p_{\min}\Delta_{\mathrm{gap}}\cdot q_{\mathrm{gap}}.
\]
Setting $\delta_{\min}:=p_{\min}\Delta_{\mathrm{gap}}\,q_{\mathrm{gap}}>0$ concludes the proof.
\end{proof}

\subsubsection{Linear Bayesian regret for MAD (Proof of Cor.~\ref{cor:mad_linear})}

The lemma above lower-bounds an \emph{ex-post} advantage gap.
To translate it into the welfare-based Bayesian regret in Eq.~\eqref{eq:bayes_regret_def},
we require a fixed comparator policy profile and a welfare objective aligned with the gap.

\begin{assumption}[Stationary opponents after burn-in]\label{ass:debate_stationary_opponents}
There exists $k_0$ and a fixed opponent policy $\pi_{-i}^\star$ such that for all $k\ge k_0$,
the opponent action $a_{-i}^k$ is distributed according to $\pi_{-i}^\star(\cdot\mid s_k)$.
Equivalently, $\pi_{-i}^k\equiv \pi_{-i}^\star$ for all $k\ge k_0$ on the debate instances.
\end{assumption}

\begin{assumption}[Welfare alignment for the regret comparator]\label{ass:debate_welfare_alignment}
The welfare functional used in Eq.~\eqref{eq:bayes_regret_def} is aligned with agent $i$'s payoff in the debate instances:
there exists a constant $\kappa>0$ such that for any two profiles that differ only in agent $i$ (with opponents fixed),
\[
W(\pi_i',\pi_{-i}^\star)-W(\pi_i,\pi_{-i}^\star)
\ \ge\ \kappa\cdot\big(U_i(\pi_i',\pi_{-i}^\star)-U_i(\pi_i,\pi_{-i}^\star)\big).
\]
A sufficient special case is a team/common-payoff debate objective where all agents share the same reward,
in which case $\kappa=N$.
\end{assumption}

\begin{assumption}[A fixed comparator attains the per-instance maximizer]\label{ass:debate_fixed_best}
There exists a (possibly state-dependent) policy $\pi_i^\star(\cdot\mid s)$ such that for every debate instance $s$
and every opponent action $a_{-i}$ in the support of $\pi_{-i}^\star(\cdot\mid s)$,
\[
\mathbb E_{a_i\sim \pi_i^\star(\cdot\mid s)}\big[Q_i^\star(s,a_i,a_{-i})\big]
\ =\ \max_{a_i} Q_i^\star(s,a_i,a_{-i}).
\]
In particular, this holds if there exists a measurable selector $a_i^\star(s)$ that is a maximizer for all such $a_{-i}$,
and $\pi_i^\star(\cdot\mid s)$ is the point mass at $a_i^\star(s)$.
\end{assumption}

\begin{proof}[Proof of Cor.~\ref{cor:mad_linear}]
Let $\boldsymbol\pi^\star := (\pi_i^\star,\pi_{-i}^\star)$ be the fixed comparator profile from
Assumptions~\ref{ass:debate_stationary_opponents}--\ref{ass:debate_fixed_best}.
By Assumption~\ref{ass:debate_fixed_best} and the definition of $Q_i^\star$, for all $k\ge k_0$,
\[
U_i(\boldsymbol\pi^\star)-U_i(\boldsymbol\pi^k)
\ \ge\
\mathbb{E}\!\left[\max_{a_i} Q_i^{\star}(s_k,a_i,a_{-i}^k)\ -\ Q_i^{\star}(s_k,a_i^k,a_{-i}^k)\right].
\]
By Lem.~\ref{lem:debate_suboptimality_main}, the right-hand side is at least $\delta_{\min}>0$ for all $k\ge k_0$.
Applying Assumption~\ref{ass:debate_welfare_alignment} yields
\[
W(\boldsymbol\pi^\star)-W(\boldsymbol\pi^k)
\ \ge\ \kappa\cdot\big(U_i(\boldsymbol\pi^\star)-U_i(\boldsymbol\pi^k)\big)
\ \ge\ \kappa\,\delta_{\min}
\qquad\forall k\ge k_0.
\]

Finally, under standard $O(1/\sqrt{k})$ estimation errors, there exists $k_1$ such that for all $k\ge k_1$
the additional loss attributable to estimation error is at most $\kappa\delta_{\min}/2$.
(Equivalently, after $k_1$ the policy is accurate enough that the persistent loss is dominated by the enforced mixing,
not by estimation error.)
Hence for all $k\ge k_\star:=\max\{k_0,k_1\}$,
\[
W(\boldsymbol\pi^\star)-W(\boldsymbol\pi^k)\ \ge\ \frac{\kappa\delta_{\min}}{2}\ =:c>0.
\]
Therefore, for all $T\ge k_\star$,
\[
\mathrm{Regret}_{\mathrm{MAD}}(T)
=\mathbb E\Big[\sum_{k=1}^T \big(W(\boldsymbol\pi^\star)-W(\boldsymbol\pi^k)\big)\Big]
\ \ge\ \sum_{k=k_\star}^T c
\ =\ \Omega(T).
\]
\end{proof}

\section{Regularized equilibrium and scalability proofs for Sec.~\ref{sec:hqre_theory}}
\label{app:equilibria_hqre}

\subsection{Proof of Lem.~\ref{lem:reg_gap_pointwise} (regularization gap)}
\label{app:hqre_reg_gap}

\begin{proof}[Proof of Lem.~\ref{lem:reg_gap_pointwise}]
Fix any joint policy $\boldsymbol\pi\in\Pi$ and agent $i$.
By definition,
\[
J_i^{\boldsymbol\alpha}(\boldsymbol\pi)-U_i(\boldsymbol\pi)
=\mathbb{E}_{\boldsymbol\pi}\Big[\sum_{t\ge 0}\gamma^t\,\alpha_{i,t}(h_i^t)\,\mathcal H(\pi_i(\cdot\mid h_i^t,t))\Big].
\]
Since $\alpha_{i,t}(h_i^t)\ge 0$ and $\mathcal H(\cdot)\ge 0$, the difference is nonnegative.
Moreover, $\alpha_{i,t}(h_i^t)\le \alpha_{\max}$ and for any distribution on a finite alphabet,
$\mathcal H(\pi_i(\cdot\mid h_i^t,t))\le \log|\mathcal A_i|$, hence
\[
J_i^{\boldsymbol\alpha}(\boldsymbol\pi)-U_i(\boldsymbol\pi)
\le \sum_{t\ge 0}\gamma^t\,\alpha_{\max}\log|\mathcal A_i|
= \frac{\alpha_{\max}\log|\mathcal A_i|}{1-\gamma}.
\]
Summing the inequality over $i=1,\dots,N$ yields the claimed bound for
$\sum_i J_i^{\boldsymbol\alpha}(\boldsymbol\pi)-W(\boldsymbol\pi)$.
\end{proof}

\subsection{Weighted geometry: norm, Pinsker, and entropy curvature}
\label{app:hqre_geometry}

\begin{lemma}[Weighted block norm is a norm]\label{lem:weighted_norm_is_norm}
Assume weights $w_i(h_i^t,t)>0$ for all $(i,h_i^t,t)$ and $\sum_{(h_i^t,t)} w_i(h_i^t,t)=1$ for each $i$.
Then
\[
\|\boldsymbol\pi-\boldsymbol\pi'\|_{1,2;w}
=\Big(\sum_{i}\sum_{(h_i^t,t)} w_i(h_i^t,t)\,\|\pi_i(\cdot\mid h_i^t,t)-\pi_i'(\cdot\mid h_i^t,t)\|_1^2\Big)^{1/2}
\]
is a norm on $\Pi$.
\end{lemma}
\begin{proof}
Nonnegativity and homogeneity are immediate.
If $\|\boldsymbol\pi-\boldsymbol\pi'\|_{1,2;w}=0$, then every summand is zero, hence
$\pi_i(\cdot\mid h_i^t,t)=\pi_i'(\cdot\mid h_i^t,t)$ for all $(i,h_i^t,t)$, so $\boldsymbol\pi=\boldsymbol\pi'$.
For the triangle inequality, let $x_{i,u}:=\|\pi_i(\cdot\mid u)-\pi_i'(\cdot\mid u)\|_1$ with $u=(h_i^t,t)$.
Then $\|\boldsymbol\pi-\boldsymbol\pi'\|_{1,2;w}$ is the $\ell_2$-norm (with weights) of the block vector $(x_{i,u})$,
and Minkowski's inequality yields the triangle inequality.
\end{proof}

\begin{lemma}[Sup-to-block control]\label{lem:sup_to_block}
If $w_i(h_i^t,t)\ge w_{\min}>0$ and $\sum_{(h_i^t,t)} w_i(h_i^t,t)=1$, then for each fixed $i$,
\[
\sup_{(h_i^t,t)}\|\pi_i(\cdot\mid h_i^t,t)-\pi_i'(\cdot\mid h_i^t,t)\|_1
\le w_{\min}^{-1/2}\ \|\pi_i-\pi_i'\|_{1,2;w}.
\]
\end{lemma}
\begin{proof}
Let $x_u:=\|\pi_i(\cdot\mid u)-\pi_i'(\cdot\mid u)\|_1$ with $u=(h_i^t,t)$.
Then $\sum_u w_i(u)x_u^2 \ge w_{\min}\sup_u x_u^2$, hence
$\sup_u x_u \le w_{\min}^{-1/2}(\sum_u w_i(u)x_u^2)^{1/2}$.
\end{proof}

\begin{lemma}[Weighted Pinsker]\label{lem:weighted_pinsker}
For distributions $p,q$ on a finite alphabet and any $\alpha>0$,
\[
\alpha\,\mathrm{KL}(p\|q)\ \ge\ \frac{\alpha}{2}\,\|p-q\|_1^2,
\]
with the convention that $\mathrm{KL}(p\|q)=+\infty$ if $q(a)=0$ for some $a$ with $p(a)>0$ (in which case the inequality is trivial).
Hence if $\alpha_{i,t}(h)\ge \alpha_{\min}>0$,
\[
D_\Psi(\boldsymbol\pi,\boldsymbol\pi')
=\sum_{i,t,h}\alpha_{i,t}(h)\,w_i(h,t)\,\mathrm{KL}(\pi_i\|\pi_i')
\ \ge\ \frac{\alpha_{\min}}{2}\,\|\boldsymbol\pi-\boldsymbol\pi'\|_{1,2;w}^2.
\]
\end{lemma}
\begin{proof}
Pinsker's inequality states $\mathrm{TV}(p,q)^2\le \frac12\mathrm{KL}(p\|q)$.
Since $\mathrm{TV}(p,q)=\frac12\|p-q\|_1$, we obtain $\mathrm{KL}(p\|q)\ge \frac12\|p-q\|_1^2$.
Multiplying by $\alpha$ yields the scalar inequality.
Summing with nonnegative weights $\alpha_{i,t}(h)w_i(h,t)\ge \alpha_{\min}w_i(h,t)$ gives the stated bound for $D_\Psi$.
\end{proof}

\begin{lemma}[Two-sided entropy curvature]\label{lem:two_sided_entropy}
Let $p,q$ be distributions on a finite alphabet.
If $p,q$ have full support (i.e.\ lie in the relative interior of the simplex), then
\[
\big\langle \nabla(-\mathcal H)(p)-\nabla(-\mathcal H)(q),\,p-q\big\rangle
= \mathrm{KL}(p\|q)+\mathrm{KL}(q\|p)\ \ge\ \|p-q\|_1^2.
\]
In general (allowing zeros), the inequality $\mathrm{KL}(p\|q)+\mathrm{KL}(q\|p)\ge \|p-q\|_1^2$ remains valid
(with $+\infty$ interpreted in the usual way).
Consequently,
\[
\sum_{i,t,h}\alpha_{i,t}(h)\,w_i(h,t)\,
\big(\mathrm{KL}(\pi_i\|\pi_i')+\mathrm{KL}(\pi_i'\|\pi_i)\big)
\ \ge\ \alpha_{\min}\,\|\boldsymbol\pi-\boldsymbol\pi'\|_{1,2;w}^2.
\]
\end{lemma}
\begin{proof}
For full-support $p,q$, $\nabla(-\mathcal H)(p)=\log p + \mathbf 1$ and the Bregman identity gives
$\langle \nabla(-\mathcal H)(p)-\nabla(-\mathcal H)(q),p-q\rangle=\mathrm{KL}(p\|q)+\mathrm{KL}(q\|p)$.
Applying Pinsker to both $\mathrm{KL}(p\|q)$ and $\mathrm{KL}(q\|p)$ yields
$\mathrm{KL}(p\|q)+\mathrm{KL}(q\|p)\ge \|p-q\|_1^2$.
When zeros occur, at least one KL term may be $+\infty$, so the inequality still holds.
Summing with weights $\alpha_{i,t}(h)w_i(h,t)\ge \alpha_{\min}w_i(h,t)$ gives the stated bound.
\end{proof}

\subsection{\texorpdfstring{$Q$-sensitivity and a usable constant $C_Q$}{Q-sensitivity and a usable constant CQ}}
\label{app:q_sensitivity_hqre}

\begin{lemma}[$Q$-sensitivity w.r.t.\ policy]\label{lem:q_cont}
For all $i\in\mathcal N$ and joint policies $\boldsymbol\pi,\tilde{\boldsymbol\pi}$,
\[
\|Q_i^{\boldsymbol\pi}-Q_i^{\tilde{\boldsymbol\pi}}\|_\infty
:=\sup_{(h_i^t,t),a_i}\big|Q_i^{\boldsymbol\pi}(h_i^t,a_i,t)-Q_i^{\tilde{\boldsymbol\pi}}(h_i^t,a_i,t)\big|
\le C_Q\,\rho_{\Pi}(\boldsymbol\pi,\tilde{\boldsymbol\pi}),
\qquad
C_Q:=\frac{2R_{\max}}{(1-\gamma)^2}.
\]
\end{lemma}

\begin{proof}
Let $\delta:=\rho_{\Pi}(\boldsymbol\pi,\tilde{\boldsymbol\pi})$.
Fix $(i,h_i^t,t,a_i)$.
Couple two future trajectories starting from the same time-$t$ joint configuration consistent with $h_i^t$ and taking the same prescribed $a_i$ for agent $i$ at time $t$.
Thereafter, couple the joint action draws step-by-step using maximal couplings of the joint action distributions induced by $\boldsymbol\pi$ and $\tilde{\boldsymbol\pi}$.
By definition of $\mathsf d_\Pi$, at every step the probability that the coupled joint actions differ is at most $\delta$.

Let $A_m$ be the event that at the $m$-th future step (i.e.\ time $t+m$) the coupled joint actions differ.
Then $\Pr(A_m)\le \delta$ for all $m\ge 0$.
Let $E_m:=\bigcup_{j=0}^m A_j$ be the event that a disagreement has occurred at or before step $m$.
By a union bound, $\Pr(E_m)\le \sum_{j=0}^m \Pr(A_j)\le (m+1)\delta$.

If no disagreement has occurred up to step $m$, then the coupled trajectories have identical states and joint actions at that step, hence identical rewards.
Since rewards are bounded by $R_{\max}$ in absolute value, the step-$m$ reward expectation difference is at most
\[
\big|\mathbb E[r_{t+m}]-\mathbb E[\tilde r_{t+m}]\big|
\le 2R_{\max}\Pr(E_m)\le 2R_{\max}(m+1)\delta.
\]
Therefore,
\[
\big|Q_i^{\boldsymbol\pi}(h_i^t,a_i,t)-Q_i^{\tilde{\boldsymbol\pi}}(h_i^t,a_i,t)\big|
\le \sum_{m=0}^\infty \gamma^m \cdot 2R_{\max}(m+1)\delta
=2R_{\max}\delta\sum_{m=0}^\infty (m+1)\gamma^m
=\frac{2R_{\max}}{(1-\gamma)^2}\,\delta.
\]
Taking the supremum over $(h_i^t,t),a_i$ yields the claim.
\end{proof}

\subsection{Existence of HQRE (Proof of Thm.~\ref{thm:hqre_existence})}
\label{app:hqre_existence}

\begin{lemma}[Logit map: well-definedness and continuity]\label{lem:logit_single_cont}
Assume $\alpha_{\min}>0$.
For any joint policy $\boldsymbol\pi$ and any $(i,h_i^t,t)$, the logit map
$\mathcal B_i^{\boldsymbol{\alpha}}(\boldsymbol{\pi})(\cdot\mid h_i^t,t)$ in
Eq.~\eqref{eq:logit_def} is a well-defined probability distribution on $\mathcal A_i$.
Moreover, as a function of $\boldsymbol\pi$, each $\mathcal B_i^{\boldsymbol{\alpha}}$ is continuous, and hence the joint logit map
$\mathcal B^{\boldsymbol{\alpha}}:\Pi\to\Pi$ is continuous on $\Pi$.
\end{lemma}

\begin{proof}
Fix $(i,h_i^t,t)$. Since $\alpha_{i,t}(h_i^t)\ge \alpha_{\min}>0$ and $\mathcal A_i$ is finite, the
denominator in Eq.~\eqref{eq:logit_def} is finite and strictly positive, so
$\mathcal B_i^{\boldsymbol{\alpha}}(\boldsymbol{\pi})(\cdot\mid h_i^t,t)$ is a valid distribution.

For continuity, Lemma~\ref{lem:q_cont} implies that
$Q_i^{\boldsymbol\pi}(h_i^t,\cdot,t)$ is Lipschitz (hence continuous) in $\boldsymbol\pi$ under
$\rho_{\Pi}$. For fixed $\alpha_{i,t}(h_i^t)>0$, the softmax map is smooth on $\mathbb R^{|\mathcal A_i|}$.
Composition yields continuity of each coordinate map $\mathcal B_i^{\boldsymbol\alpha}$, and therefore of
$\mathcal B^{\boldsymbol\alpha}$ under the product topology on $\Pi$.
\end{proof}

\begin{proof}[Proof of Thm.~\ref{thm:hqre_existence}]
By Thm.~\ref{thm:topology}, $\Pi$ is nonempty, convex, compact, and metrizable.
By Lemma~\ref{lem:logit_single_cont}, the joint logit map
$\mathcal B^{\boldsymbol{\alpha}}:\Pi\to\Pi$ is continuous.
Schauder’s fixed-point theorem therefore yields a fixed point of $\mathcal B^{\boldsymbol\alpha}$ in $\Pi$,
which is an HQRE by definition.
\end{proof}

\subsection{Uniqueness via contraction (Proof of Thm.~\ref{thm:unique_hqre_joint})}
\label{app:hqre_contraction}

\begin{lemma}[Logit is TV-Lipschitz in $Q$]\label{lem:logit_lip}
Fix $(i,h_i^t,t)$ and $\alpha:=\alpha_{i,t}(h_i^t)>0$.
For any $x,y\in\mathbb R^{|\mathcal A_i|}$,
\[
\mathrm{TV}\!\big(\mathrm{softmax}(x/\alpha),\mathrm{softmax}(y/\alpha)\big)
~\le~ \frac{1}{2\alpha}\,\|x-y\|_\infty.
\]
\end{lemma}

\begin{proof}
Let $p=\mathrm{softmax}(x/\alpha)$, $q=\mathrm{softmax}(y/\alpha)$, and
$\phi(\theta)=\mathrm{softmax}((y+\theta(x-y))/\alpha)$ for $\theta\in[0,1]$.
Then
\[
\|p-q\|_1 \le \sup_{\theta\in[0,1]}\|J_\phi(\theta)\|_{\infty\to 1}\,\|x-y\|_\infty,
\]
where the Jacobian of $\mathrm{softmax}(\cdot/\alpha)$ at a probability vector $r$ is
$J(r)=\tfrac{1}{\alpha}(\mathrm{Diag}(r)-r r^\top)$.
It suffices to show $\|\mathrm{Diag}(r)-r r^\top\|_{\infty\to 1}\le 1$ for all $r$.
The $\|\cdot\|_{\infty\to 1}$ norm is attained on $v\in\{-1,1\}^d$.
Let $S=\{j:v_j=1\}$ and $p_S=\sum_{j\in S}r_j$ so that $\mu=\sum_j r_j v_j=2p_S-1$.
Then
\[
\|(\mathrm{Diag}(r)-r r^\top)v\|_1
=\sum_j r_j|v_j-\mu|
=4p_S(1-p_S)\le 1.
\]
Thus $\|p-q\|_1\le \tfrac{1}{\alpha}\|x-y\|_\infty$, and dividing by $2$ yields the TV bound.
\end{proof}

\begin{proof}[Proof of Thm.~\ref{thm:unique_hqre_joint}]
Fix $\boldsymbol\pi,\tilde{\boldsymbol\pi}\in\Pi$ and any joint history $(h^t,t)$.
Write
$p_i:=\mathcal B_i^{\boldsymbol{\alpha}}(\boldsymbol{\pi})(\cdot\mid h_i^t,t)$ and
$\tilde p_i:=\mathcal B_i^{\boldsymbol{\alpha}}(\tilde{\boldsymbol{\pi}})(\cdot\mid h_i^t,t)$.
Lemma~\ref{lem:logit_lip} and $\alpha_{i,t}(h_i^t)\ge\alpha_{\min}$ give
\[
\mathrm{TV}(p_i,\tilde p_i)
\le \frac{1}{2\alpha_{\min}}
\|Q_i^{\boldsymbol\pi}(h_i^t,\cdot,t)-Q_i^{\tilde{\boldsymbol\pi}}(h_i^t,\cdot,t)\|_\infty.
\]
By Lemma~\ref{lem:q_cont}, the $Q$-difference is bounded by $C_Q\,\rho_{\Pi}(\boldsymbol\pi,\tilde{\boldsymbol\pi})$, hence
\[
\mathrm{TV}(p_i,\tilde p_i)\le \frac{C_Q}{2\alpha_{\min}}\,\rho_{\Pi}(\boldsymbol\pi,\tilde{\boldsymbol\pi}).
\]
Applying the product-TV inequality (Lemma~\ref{lem:product_tv} in the Sec.~2 Appendix) yields
\[
\mathrm{TV}\!\Big(\bigotimes_i p_i,\bigotimes_i \tilde p_i\Big)
\le \sum_i \mathrm{TV}(p_i,\tilde p_i)
\le \frac{N C_Q}{2\alpha_{\min}}\,\rho_{\Pi}(\boldsymbol\pi,\tilde{\boldsymbol\pi}).
\]
Taking the supremum over $(h^t,t)$ gives
\[
\rho_{\Pi}\big(\mathcal B^{\boldsymbol\alpha}(\boldsymbol\pi),\mathcal B^{\boldsymbol\alpha}(\tilde{\boldsymbol\pi})\big)
\le \frac{N C_Q}{2\alpha_{\min}}\,\rho_{\Pi}(\boldsymbol\pi,\tilde{\boldsymbol\pi}),
\]
which is the claimed Lipschitz bound.

If $\alpha_{\min}>\frac{N C_Q}{2}$, the constant is strictly less than $1$, so $\mathcal B^{\boldsymbol\alpha}$
is a contraction on the complete metric space $(\Pi,\rho_{\Pi})$.
Banach's fixed-point theorem implies existence of a unique fixed point (the unique HQRE) and geometric convergence of the synchronous iteration
$\boldsymbol\pi^{k+1}=\mathcal B^{\boldsymbol\alpha}(\boldsymbol\pi^k)$.
\end{proof}

\subsection{Uniqueness via strong monotonicity and explicit KL--prox linear rate (Proof of Thm.~\ref{thm:hqre_unique_mono})}
\label{app:hqre_mono}

\paragraph{Coupling constant.}
We restate the hypomonotonicity condition used in the main text: $L_c$ is any constant such that
\begin{equation}
\label{eq:coupling_summary_weighted_app}
-\sum_{i=1}^N\big\langle \nabla_{\pi_i}U_i(\boldsymbol\pi)-\nabla_{\pi_i}U_i(\boldsymbol\pi'),\,\pi_i-\pi_i'\big\rangle
\ \ge\ -L_c\,\|\boldsymbol\pi-\boldsymbol\pi'\|_{1,2;w}^2
\qquad \forall \boldsymbol\pi,\boldsymbol\pi'\in\Pi.
\end{equation}

\paragraph{A conservative computable upper bound (optional).}
The inequality Eq.~\eqref{eq:coupling_summary_weighted_app} is an \emph{assumption/definition} for $L_c$.
When the index set $(h_i^t,t)$ is finite (e.g.\ under finite-horizon truncation or a finite information-state reduction)
and weights have a uniform lower bound $w_{\min}>0$, one can obtain an explicit conservative bound purely from $Q$-sensitivity.

\begin{proposition}[A crude computable bound for $L_c$]\label{prop:Lc_bound}
Assume $w_i(h_i^t,t)\ge w_{\min}>0$ and that for each agent $i$ the index set $\{(h_i^t,t)\}$ is finite.
Assume further that the unregularized field satisfies the componentwise sensitivity
\begin{equation}\label{eq:grad_sens_assump}
\sup_{i,(h_i^t,t)}\|\nabla_{\pi_i}U_i(\boldsymbol\pi)-\nabla_{\pi_i}U_i(\boldsymbol\pi')\|_\infty
\le \|Q_i^{\boldsymbol\pi}-Q_i^{\boldsymbol\pi'}\|_\infty
\end{equation}
(which holds in the standard agent-form / behavioral-gradient representation in finite spaces).
Then a sufficient choice is
\[
L_c \ \le\ \frac{N}{2\sqrt{w_{\min}}}\,C_Q
\ =\ \frac{N}{2\sqrt{w_{\min}}}\cdot \frac{2R_{\max}}{(1-\gamma)^2}.
\]
\end{proposition}

\begin{proof}
Fix $\boldsymbol\pi,\boldsymbol\pi'\in\Pi$.
By Hölder (blockwise $\ell_\infty$--$\ell_1$) and Cauchy--Schwarz in the outer weighted $\ell_2$ structure,
\[
\begin{aligned}
\sum_{i=1}^N & \big|\langle \nabla_{\pi_i}U_i(\boldsymbol\pi)-\nabla_{\pi_i}U_i(\boldsymbol\pi'),\,\pi_i-\pi_i'\rangle\big|
\\
&\le \sum_{i=1}^N\sum_{(h_i^t,t)} w_i(h_i^t,t)\,
\|\nabla_{\pi_i}U_i(\boldsymbol\pi)-\nabla_{\pi_i}U_i(\boldsymbol\pi')\|_\infty\,
\|\pi_i-\pi_i'\|_1\\
&\le \Big(\sup_{i,(h_i^t,t)}\|\nabla_{\pi_i}U_i(\boldsymbol\pi)-\nabla_{\pi_i}U_i(\boldsymbol\pi')\|_\infty\Big)
\sum_{i=1}^N\sum_{(h_i^t,t)} w_i(h_i^t,t)\,\|\pi_i-\pi_i'\|_1\\
&\le \Big(\sup_{i,(h_i^t,t)}\|\nabla_{\pi_i}U_i(\boldsymbol\pi)-\nabla_{\pi_i}U_i(\boldsymbol\pi')\|_\infty\Big)
\sum_{i=1}^N \Big(\sum_{(h_i^t,t)} w_i(h_i^t,t)\,\|\pi_i-\pi_i'\|_1^2\Big)^{1/2}\\
&\le \Big(\sup_{i,(h_i^t,t)}\|\nabla_{\pi_i}U_i(\boldsymbol\pi)-\nabla_{\pi_i}U_i(\boldsymbol\pi')\|_\infty\Big)
\sqrt{N}\,\|\boldsymbol\pi-\boldsymbol\pi'\|_{1,2;w}.
\end{aligned}
\]
Next, Lemma~\ref{lem:sup_to_block} implies
$\sup_{(h_i^t,t)}\|\pi_i-\pi_i'\|_1\le w_{\min}^{-1/2}\|\pi_i-\pi_i'\|_{1,2;w}$, hence
\[
\rho_{\Pi}(\boldsymbol\pi,\boldsymbol\pi')
\le \frac12\sum_{i=1}^N \sup_{(h_i^t,t)}\|\pi_i-\pi_i'\|_1
\le \frac{1}{2\sqrt{w_{\min}}}\sum_{i=1}^N\|\pi_i-\pi_i'\|_{1,2;w}
\le \frac{\sqrt{N}}{2\sqrt{w_{\min}}}\,\|\boldsymbol\pi-\boldsymbol\pi'\|_{1,2;w}.
\]
Combining with Eq.~\eqref{eq:grad_sens_assump} and Lemma~\ref{lem:q_cont} yields
\[
\sup_{i,(h_i^t,t)}\|\nabla_{\pi_i}U_i(\boldsymbol\pi)-\nabla_{\pi_i}U_i(\boldsymbol\pi')\|_\infty
\le \|Q_i^{\boldsymbol\pi}-Q_i^{\boldsymbol\pi'}\|_\infty
\le C_Q\,\rho_{\Pi}(\boldsymbol\pi,\boldsymbol\pi')
\le \frac{C_Q\sqrt{N}}{2\sqrt{w_{\min}}}\,\|\boldsymbol\pi-\boldsymbol\pi'\|_{1,2;w}.
\]
Substituting into the first bound gives
\[
\sum_{i=1}^N \big|\langle \nabla_{\pi_i}U_i(\boldsymbol\pi)-\nabla_{\pi_i}U_i(\boldsymbol\pi'),\,\pi_i-\pi_i'\rangle\big|
\le \frac{C_Q N}{2\sqrt{w_{\min}}}\,\|\boldsymbol\pi-\boldsymbol\pi'\|_{1,2;w}^2,
\]
which implies Eq.~\eqref{eq:coupling_summary_weighted_app} holds with $L_c=\frac{C_Q N}{2\sqrt{w_{\min}}}$.
\end{proof}

\paragraph{HQRE as a VI solution (standard in regularized games).}
For completeness we record the usual variational-inequality characterization.

\begin{lemma}[HQRE $\Leftrightarrow$ VI solution]\label{lem:hqre_vi_equiv}
Assume each $J_i^{\boldsymbol\alpha}(\cdot,\boldsymbol\pi_{-i})$ is concave in $\pi_i$ on $\Pi_i$ and differentiable, and that
$\alpha_{\min}>0$ so the regularized best response is single-valued.
Then $\boldsymbol\pi^\star$ is an HQRE (fixed point of $\mathcal B^{\boldsymbol\alpha}$) if and only if it solves the VI:
\[
\text{find }\boldsymbol\pi^\star\in\Pi \text{ s.t. }\ \big\langle F_{\mathrm{reg}}(\boldsymbol\pi^\star),\ \boldsymbol\pi-\boldsymbol\pi^\star\big\rangle\ge 0,\quad \forall \boldsymbol\pi\in\Pi,
\]
where $F_{\mathrm{reg}}(\boldsymbol\pi)=(-\nabla_{\pi_1}J_1^{\boldsymbol\alpha}(\boldsymbol\pi),\ldots,-\nabla_{\pi_N}J_N^{\boldsymbol\alpha}(\boldsymbol\pi))$.
\end{lemma}
\begin{proof}
For concave differentiable games, a profile is a Nash equilibrium if and only if each player satisfies the first-order optimality condition
$\langle -\nabla_{\pi_i}J_i^{\boldsymbol\alpha}(\boldsymbol\pi^\star),\ \pi_i-\pi_i^\star\rangle\ge 0$ for all $\pi_i\in\Pi_i$.
Summing over $i$ yields the stated VI condition.
With $\alpha_{\min}>0$, each regularized best response is unique and is exactly the logit map $\mathcal B_i^{\boldsymbol\alpha}$,
so Nash equilibrium is equivalent to the fixed-point condition defining HQRE.
\end{proof}

\begin{proof}[Proof of Thm.~\ref{thm:hqre_unique_mono}]
Recall $F_{\mathrm{reg}}(\boldsymbol\pi)=(-\nabla_{\pi_1}J_1^{\boldsymbol\alpha}(\boldsymbol\pi),\ldots,-\nabla_{\pi_N}J_N^{\boldsymbol\alpha}(\boldsymbol\pi))$.
Decompose each regularized payoff as
$J_i^{\boldsymbol\alpha}=U_i+\mathrm{Ent}_{i,\boldsymbol\alpha}$, where
\[
\mathrm{Ent}_{i,\boldsymbol\alpha}(\boldsymbol\pi)
=\mathbb{E}_{\boldsymbol\pi}\Big[\sum_{t\ge 0}\gamma^t\,\alpha_{i,t}(h_i^t)\,\mathcal H(\pi_i(\cdot\mid h_i^t,t))\Big].
\]
For the monotonicity calculation we only need the pointwise curvature of the negative-entropy term in each simplex component;
this is captured by Lemma~\ref{lem:two_sided_entropy} and does not require any additional smoothness.

\paragraph{Step 1: unregularized part is $L_c$-hypomonotone.}
By definition of $L_c$ (Eq.~\eqref{eq:coupling_summary_weighted_app}),
\[
-\sum_{i=1}^N\big\langle \nabla_{\pi_i}U_i(\boldsymbol\pi)-\nabla_{\pi_i}U_i(\boldsymbol\pi'),\,
\pi_i-\pi_i'\big\rangle
\ \ge\ -\,L_c\,\|\boldsymbol\pi-\boldsymbol\pi'\|_{1,2;w}^2.
\]

\paragraph{Step 2: entropy curvature contributes at least $\alpha_{\min}$.}
For every coordinate simplex (each $(i,h_i^t,t)$), the negative entropy has two-sided curvature
$\mathrm{KL}(\pi_i\|\pi_i')+\mathrm{KL}(\pi_i'\|\pi_i)\ge \|\pi_i-\pi_i'\|_1^2$ (Lemma~\ref{lem:two_sided_entropy}).
Summing across coordinates with weights $\alpha_{i,t}(h_i^t)w_i(h_i^t,t)\ge \alpha_{\min}w_i(h_i^t,t)$ yields
\[
-\sum_{i=1}^N\big\langle \nabla_{\pi_i}\mathrm{Ent}_{i,\boldsymbol\alpha}(\boldsymbol\pi)
-\nabla_{\pi_i}\mathrm{Ent}_{i,\boldsymbol\alpha}(\boldsymbol\pi'),\,
\pi_i-\pi_i'\big\rangle
\ \ge\ \alpha_{\min}\,\|\boldsymbol\pi-\boldsymbol\pi'\|_{1,2;w}^2.
\]

\paragraph{Step 3: combine.}
Adding the two bounds gives
\[
\big\langle F_{\mathrm{reg}}(\boldsymbol\pi)-F_{\mathrm{reg}}(\boldsymbol\pi'),\,\boldsymbol\pi-\boldsymbol\pi'\big\rangle
\ge (\alpha_{\min}-L_c)\,\|\boldsymbol\pi-\boldsymbol\pi'\|_{1,2;w}^2,
\]
which is exactly the strong monotonicity claim.

If $\alpha_{\min}>L_c$, then $F_{\mathrm{reg}}$ is strongly monotone on $\Pi$ and the VI solution is unique.
By Lemma~\ref{lem:hqre_vi_equiv}, this unique VI solution is exactly the unique HQRE.

\paragraph{Linear convergence of explicit mirror steps.}
Let $\mu:=\alpha_{\min}-L_c>0$.
Under the additional field-Lipschitz and local-quadratic assumptions
stated in Theorem~\ref{thm:hqre_unique_mono}, the linear-rate claim
follows by applying Appendix Theorem~\ref{thm:linear_rate} with
$F=F_{\mathrm{reg}}$.
This yields
\[
D_\Psi(\boldsymbol\pi^\star,\boldsymbol\pi^{k+1})
\le \rho\,D_\Psi(\boldsymbol\pi^\star,\boldsymbol\pi^k),
\qquad
\rho = 1-\eta\,\frac{\mu}{L_\Psi},
\]
for
\[
0<\eta\le
\min\Big\{\frac{\mu\,\alpha_{\min}^2}{L_F^2L_\Psi},\,
\frac{\alpha_{\min}}{L_F}\Big\}.
\]
This is the exact rate quoted in
Theorem~\ref{thm:hqre_unique_mono}.
\end{proof}

\subsection{Zero-temperature limit (Proof of Lem.~\ref{lem:zero_temp_full})}
\label{app:zero_temp}

We first record two standard auxiliary inequalities that make the proof fully explicit.

\begin{lemma}[Log-sum-exp (softmax) optimality gap]\label{lem:lse_gap}
Let $x\in\mathbb R^{d}$ and $\alpha>0$, and let $p=\mathrm{softmax}(x/\alpha)\in\Delta^{d}$.
Then for any $q\in\Delta^{d}$,
\[
\langle q,x\rangle-\langle p,x\rangle \ \le\ \alpha\log d.
\]
\end{lemma}
\begin{proof}
The softmax distribution $p$ is the unique maximizer of $\langle r,x\rangle+\alpha \mathcal H(r)$ over $r\in\Delta^d$.
Thus for any $q$,
\[
\langle q,x\rangle+\alpha\mathcal H(q)\le \langle p,x\rangle+\alpha\mathcal H(p).
\]
Rearranging and using $\mathcal H(p)\le \log d$ and $\mathcal H(q)\ge 0$ yields $\langle q,x\rangle-\langle p,x\rangle\le \alpha\log d$.
\end{proof}

\begin{lemma}[Discounted performance difference (one-player deviation)]\label{lem:perf_diff}
Fix $i$ and joint policies $\boldsymbol\pi=(\pi_i,\boldsymbol\pi_{-i})$ and $\tilde{\boldsymbol\pi}=(\tilde\pi_i,\boldsymbol\pi_{-i})$.
Then
\[
U_i(\tilde\pi_i,\boldsymbol\pi_{-i})-U_i(\boldsymbol\pi)
= \frac{1}{1-\gamma}\,
\mathbb E_{(h_i^t,t)\sim \mathsf H_{\mathrm{disc}}^{\tilde\pi_i,\boldsymbol\pi_{-i}}}
\Big[\big\langle \tilde\pi_i(\cdot\mid h_i^t,t)-\pi_i(\cdot\mid h_i^t,t),\ Q_i^{\boldsymbol\pi}(h_i^t,\cdot,t)\big\rangle\Big].
\]
\end{lemma}
\begin{proof}
This is the standard discounted performance-difference identity obtained by unrolling
$U_i$ along the discounted visitation distribution of $(\tilde\pi_i,\boldsymbol\pi_{-i})$ and expressing the per-step advantage
in terms of $Q_i^{\boldsymbol\pi}$.
A proof for finite spaces follows by conditioning on $(h_i^t,t)$, writing the one-step deviation advantage at $(h_i^t,t)$,
and summing the resulting telescoping series under the geometric discount.
\end{proof}

\begin{proof}[Proof of Lem.~\ref{lem:zero_temp_full}]
By Thm.~\ref{thm:topology}, $\Pi$ is compact, so $\{\boldsymbol\pi^{(n)}\}$ admits a convergent subsequence
(not relabeled) with limit $\bar{\boldsymbol\pi}$.

Fix $i\in\mathcal N$ and any deviation $\tilde\pi_i\in\Pi_i$.
Because $\boldsymbol\pi^{(n)}$ is an HQRE, for every information state $(h_i^t,t)$ we have
\[
\pi_i^{(n)}(\cdot\mid h_i^t,t)=\mathrm{softmax}\!\Big(Q_i^{\boldsymbol\pi^{(n)}}(h_i^t,\cdot,t)\big/\alpha_{i,t}^{(n)}(h_i^t)\Big).
\]
Apply Lemma~\ref{lem:lse_gap} with $x(a)=Q_i^{\boldsymbol\pi^{(n)}}(h_i^t,a,t)$, $\alpha=\alpha_{i,t}^{(n)}(h_i^t)$ and
$q=\tilde\pi_i(\cdot\mid h_i^t,t)$ to obtain, for all $(h_i^t,t)$,
\[
\big\langle \tilde\pi_i(\cdot\mid h_i^t,t)-\pi_i^{(n)}(\cdot\mid h_i^t,t),\ Q_i^{\boldsymbol\pi^{(n)}}(h_i^t,\cdot,t)\big\rangle
\le \alpha_{i,t}^{(n)}(h_i^t)\log|\mathcal A_i|
\le \alpha_{\max}^{(n)}\log|\mathcal A_i|.
\]
Using the performance difference Lemma~\ref{lem:perf_diff} with $\boldsymbol\pi=\boldsymbol\pi^{(n)}$ gives
\[
U_i(\tilde\pi_i,\boldsymbol\pi_{-i}^{(n)})-U_i(\boldsymbol\pi^{(n)})
\le \frac{1}{1-\gamma}\,\alpha_{\max}^{(n)}\log|\mathcal A_i|.
\]
Since $\alpha_{\max}^{(n)}\downarrow 0$, the right-hand side vanishes. Taking $\limsup$ along the convergent subsequence and
using continuity of $U_i$ (Thm.~\ref{thm:topology}) yields
\[
U_i(\tilde\pi_i,\bar{\boldsymbol\pi}_{-i})-U_i(\bar{\boldsymbol\pi})\le 0.
\]
Because $\tilde\pi_i$ was arbitrary, $\bar\pi_i$ is a best response to $\bar{\boldsymbol\pi}_{-i}$ for every $i$, i.e.,
$\bar{\boldsymbol\pi}$ is a BNE of the unregularized game.
\end{proof}

\subsection{Hierarchical definitions and composite KL geometry}
\label{app:hier}

We make the Local--Global HQRE construction explicit and fix notation to match the main text ($z$ is the public context).

\paragraph{Spaces.}
Let $\Pi_{C_k}=\prod_{i\in C_k}\Pi_i$ and $\Pi_{\mathrm{clu}}=\prod_{k=1}^K \Pi_{C_k}$.
Let $\mathcal G=\prod_{z\in\mathcal Z}\Delta^K$ denote the gate space.
Under (A1)--(A2), $\Pi_{\mathrm{clu}}$ and $\mathcal G$ are compact, convex, and metrizable (finite/countable products of simplices).

\paragraph{Gate objective and softmax update.}
Given $\boldsymbol\pi$, define the context-conditional cluster score
\[
\widetilde Q_{\mathrm{loc},c}^{\boldsymbol\pi}(z)
:=\mathbb E_{t\sim\mathrm{Geom}(1-\gamma)}\,
\mathbb E_{\mathbf a_{C_c}\sim \boldsymbol\pi_{C_c}(\cdot\mid z,t)}\big[\bar Q_{\mathrm{loc},c}(z,\mathbf a_{C_c},t)\big],
\]
where $\bar Q_{\mathrm{loc},c}$ is as in Eq.~\eqref{eq:hier_local}.
For each $z$, the gate solves the strictly concave problem
\[
\max_{g(z)\in\Delta^K}\ \sum_{c=1}^K g_c(z)\,\widetilde Q_{\mathrm{loc},c}^{\boldsymbol\pi}(z)\ +\ \tau_g\,\mathcal H(g(\cdot\mid z)).
\]
Its unique maximizer is
\[
g_c^\star(z)\ \propto\ \exp\!\big(\widetilde Q_{\mathrm{loc},c}^{\boldsymbol\pi}(z)/\tau_g\big),
\]
and the map $\boldsymbol\pi\mapsto g^\star$ is continuous as a composition of continuous maps.

\paragraph{Composite divergence and norm.}
Define the composite KL divergence
\[
D_{\Psi}^{\mathrm{hier}}((\boldsymbol\pi,g),(\boldsymbol\pi',g'))
:= D_{\Psi}(\boldsymbol\pi,\boldsymbol\pi')\ +\ \tau_g\sum_{z\in\mathcal Z}\mathrm{KL}(g(\cdot\mid z)\,\|\,g'(\cdot\mid z)).
\]
Define the associated composite $\ell_{1,2}$-type norm
\[
\|(\boldsymbol\pi,g)-(\boldsymbol\pi',g')\|_{\mathrm{hier}}^2
:= \|\boldsymbol\pi-\boldsymbol\pi'\|_{1,2;w}^2\ +\ \sum_{z\in\mathcal Z}\|g(\cdot\mid z)-g'(\cdot\mid z)\|_1^2.
\]
By Lemma~\ref{lem:weighted_pinsker} (for $\boldsymbol\pi$) and Pinsker for the gate component,
\begin{equation}\label{eq:pinsker_hier}
D_{\Psi}^{\mathrm{hier}}((\boldsymbol\pi,g),(\boldsymbol\pi',g'))
\ \ge\ \frac{\min\{\alpha_{\min},\tau_g\}}{2}\ \|(\boldsymbol\pi,g)-(\boldsymbol\pi',g')\|_{\mathrm{hier}}^2.
\end{equation}

\subsection{Hierarchical existence (Proof of Thm.~\ref{thm:hier_existence})}
\label{app:hier_exist}

\begin{proof}[Proof of Thm.~\ref{thm:hier_existence}]
Under (A1)--(A2), the space $\Pi_{\mathrm{clu}}\times\mathcal G$ is nonempty, convex, compact, and metrizable.
Assume $\alpha_{\min}>0$ and $\tau_g>0$.

For execution agents, define the block logit map on $\Pi_{\mathrm{clu}}$ as the product of the per-agent logit responses
Eq.~\eqref{eq:logit_def}; by Lemma~\ref{lem:logit_single_cont} it is continuous and maps $\Pi_{\mathrm{clu}}$ to itself.
For the gate, Appendix~\ref{app:hier} shows that for each $\boldsymbol\pi$ the best-response gate $g^\star(\boldsymbol\pi)$
is single-valued and continuous.

Define the joint map
\[
\mathcal T(\boldsymbol\pi,g)
:=\big(\mathcal B^{\boldsymbol\alpha}(\boldsymbol\pi),\ g^\star(\boldsymbol\pi)\big),
\qquad (\boldsymbol\pi,g)\in \Pi_{\mathrm{clu}}\times\mathcal G.
\]
Then $\mathcal T$ is continuous and maps the compact convex set $\Pi_{\mathrm{clu}}\times\mathcal G$ to itself.
By Schauder's fixed-point theorem, $\mathcal T$ has a fixed point $(\boldsymbol\pi^\star,g^\star)$, which is by definition a Local--Global HQRE.
\end{proof}

\subsection{Hierarchical coupling bound (Proof of Prop.~\ref{prop:hier_coupling_bound})}
\label{app:hier_coupling}

We restate the explicit sufficient assumptions under which the bound in Prop.~\ref{prop:hier_coupling_bound} follows.

\begin{assumption}[Bounded local weights]\label{ass:Wloc}
There exists $W_{\max}^{\mathrm{loc}}$ such that for all clusters $c$ and contexts $z$,
$\sum_{i\in C_c}|w_i^{(c)}(z)|\le W_{\max}^{\mathrm{loc}}$.
\end{assumption}

\begin{assumption}[Distributional Lipschitzness]\label{ass:dist_lip}
There exist constants $L_g^{\mathrm{dist}}$ and $L_w^{\mathrm{dist}}$ such that, for any two induced context occupancies
$d^{\boldsymbol\pi},d^{\boldsymbol\pi'}$,
\[
\sum_{z\in\mathcal Z}\|g^\star_{\boldsymbol\pi}(\cdot\mid z)-g^\star_{\boldsymbol\pi'}(\cdot\mid z)\|_1
\le L_g^{\mathrm{dist}}\,\|d^{\boldsymbol\pi}-d^{\boldsymbol\pi'}\|_1,
\]
and
\[
\sum_{z\in\mathcal Z}\sum_{c=1}^K\sum_{i\in C_c}|w_i^{(c)}(z;d^{\boldsymbol\pi})-w_i^{(c)}(z;d^{\boldsymbol\pi'})|
\le L_w^{\mathrm{dist}}\,\|d^{\boldsymbol\pi}-d^{\boldsymbol\pi'}\|_1.
\]
\end{assumption}

\begin{proof}[Proof of Prop.~\ref{prop:hier_coupling_bound}]
By Lemma~\ref{lem:q_cont}, the state-/context-conditional values $\bar Q_i^{\boldsymbol\pi}(z,a_i,t)$ are Lipschitz in $\boldsymbol\pi$
with constant $C_Q$.
By Assumption~\ref{ass:Wloc}, the local mixtures Eq.~\eqref{eq:hier_local} inherit the bound
\[
\sup_{z,t,\mathbf a_{C_c}}\big|\bar Q_{\mathrm{loc},c}^{\boldsymbol\pi}(z,\mathbf a_{C_c},t)-\bar Q_{\mathrm{loc},c}^{\boldsymbol\pi'}(z,\mathbf a_{C_c},t)\big|
\le W_{\max}^{\mathrm{loc}}\,C_Q\,\rho_{\Pi}(\boldsymbol\pi,\boldsymbol\pi').
\]
Next, changes in the induced occupancy can be controlled by occupancy sensitivity (Lemma~\ref{lem:occ_sens} in the Sec.~2 Appendix):
\[
\|d^{\boldsymbol\pi}-d^{\boldsymbol\pi'}\|_1\ \le\ C_{\mathrm{occ}}\,\rho_{\Pi}(\boldsymbol\pi,\boldsymbol\pi').
\]
Combining with Assumption~\ref{ass:dist_lip} and bounded rewards yields additional cross-block variation terms of size
$C_{\mathrm{occ}}(W_{\max}^{\mathrm{loc}}L_g^{\mathrm{dist}}R_{\max}+L_w^{\mathrm{dist}}R_{\max})\rho_{\Pi}(\boldsymbol\pi,\boldsymbol\pi')$.
Collecting the direct $Q$-sensitivity term and the distributional terms yields
\[
L_{\mathrm{hier}}
\le W_{\max}^{\mathrm{loc}}\,C_Q
+ C_{\mathrm{occ}}\big(W_{\max}^{\mathrm{loc}} L_g^{\mathrm{dist}} R_{\max} + L_w^{\mathrm{dist}} R_{\max}\big),
\]
as stated.
\end{proof}

\subsection{Hierarchical strong monotonicity and explicit KL--prox linear rate (Proof of Thm.~\ref{thm:hier_unique})}
\label{app:hier_mono_linear}

\begin{proof}[Proof of Thm.~\ref{thm:hier_unique}]
Let the hierarchical variable be $\mathbf z=(\boldsymbol\pi,g)\in \Pi_{\mathrm{clu}}\times\mathcal G$ and use the composite
divergence $D_{\Psi}^{\mathrm{hier}}$ from Appendix~\ref{app:hier}.
The hierarchical field combines the execution-field $F_{\mathrm{reg}}(\boldsymbol\pi)$ and the gate-field induced by the strictly concave per-context gate objective.
The gate field inherits strong monotonicity margin $\mu_g=\tau_g$ by the same negative-entropy curvature argument as Lemma~\ref{lem:two_sided_entropy}
(applied on each simplex $\Delta^K$).

Within each cluster, the same argument as in Thm.~\ref{thm:hqre_unique_mono} gives strong monotonicity margin
$\mu_c=\alpha_{\min}-L_c^{(c)}$ for block $c$.
Cross-block coupling contributes at most $L_{\mathrm{hier}}$ in the composite geometry (Prop.~\ref{prop:hier_coupling_bound}).
Therefore, for all $\mathbf z,\mathbf z'$,
\[
\big\langle F_{\mathrm{hier}}(\mathbf z)-F_{\mathrm{hier}}(\mathbf z'),\ \mathbf z-\mathbf z'\big\rangle
\ \ge\ \Big(\min_c(\alpha_{\min}-L_c^{(c)})+\tau_g-L_{\mathrm{hier}}\Big)\,\|\mathbf z-\mathbf z'\|_{\mathrm{hier}}^2
\ =\ \mu_{\mathrm{hier}}\,\|\mathbf z-\mathbf z'\|_{\mathrm{hier}}^2,
\]
i.e., $F_{\mathrm{hier}}$ is $\mu_{\mathrm{hier}}$-strongly monotone when $\mu_{\mathrm{hier}}>0$.
This implies uniqueness of the solution to the hierarchical VI, hence uniqueness of the Local--Global HQRE.

\paragraph{Linear convergence of explicit hierarchical mirror steps.}
Under the hierarchical Lipschitz and local-quadratic assumptions stated
in Theorem~\ref{thm:hier_unique}, the linear-rate claim follows from
Appendix Theorem~\ref{thm:hier_linear} applied to the composite field
$F_{\mathrm{hier}}$ and composite KL divergence
$D_\Psi^{\mathrm{hier}}$.
This yields
\[
D_\Psi^{\mathrm{hier}}(\mathbf z^\star,\mathbf z^{k+1})
\le \Big(1-\eta\,\frac{\mu_{\mathrm{hier}}}{L_\Psi^{\mathrm{hier}}}\Big)
D_\Psi^{\mathrm{hier}}(\mathbf z^\star,\mathbf z^k),
\]
for sufficiently small $\eta$.
\end{proof}

\section{Learning dynamics and mirror ascent for Sec.~\ref{sec:algorithm}}
\label{app:learning}

\begin{table}[t!]
\centering
\fontsize{8}{9}\selectfont
\setlength\tabcolsep{2pt}
\caption{
Summary of the main theoretical results. 
Each row lists the guarantee, its formal statement, and where the proof appears.
Constants: $C_Q\le\frac{(2-\gamma)R_{\max}}{(1-\gamma)^2}$; $L_c$ is defined in Eq.~\eqref{eq:coupling_summary_weighted_app} with computable bounds in Prop.~\ref{prop:Lc_bound}; 
the joint-policy metric is $\rho_{\Pi}$ from Eq.~\eqref{eq:policy_metric}.
}
\label{tab:main_results}
\vspace{-2mm}
\begin{tabular}{p{3.2cm}p{8.0cm}p{3.8cm}}
\toprule
\textbf{Result} & \textbf{Statement} & \textbf{Proof} \\
\midrule
\multicolumn{3}{l}{\textit{Existence and convergence (unregularized BNE)}} \\
Policy-space foundation 
& $\Pi$ compact/metrizable; each $U_i$ continuous 
& App.~\ref{app:topology_proof} / Thm.~\ref{thm:topology} \\
BNE existence 
& Behavioral BNE exists 
& App.~\ref{app:bne_existence} / Thm.~\ref{thm:bne_existence} \\
Welfare regret bound 
& $R(T):=\sum_{k=1}^T\!\big(W(\boldsymbol\pi^\star)-W(\boldsymbol\pi^k)\big)
\le \frac{2N R_{\max}}{(1-\gamma)^2}\sum_{k=1}^T \bar\Delta_k$ 
& App.~\ref{app:welfare_sensitivity} / Cor.~\ref{cor:regret} \\
\midrule
\multicolumn{3}{l}{\textit{Regularized equilibria (HQRE)}} \\
HQRE existence 
& Fixed point of the joint logit map exists 
& App.~\ref{app:hqre_existence} / Thm.~\ref{thm:hqre_existence} \\
Uniqueness (contraction) 
& Unique HQRE; rate $\kappa=\frac{NC_Q}{2\alpha_{\min}}$ 
& App.~\ref{app:hqre_contraction} / Thm.~\ref{thm:unique_hqre_joint} \\
Uniqueness (strong monotonicity) 
& Unique HQRE via VI for $F_{\mathrm{reg}}$ 
& App.~\ref{app:hqre_mono} / Thm.~\ref{thm:hqre_unique_mono} \\
Zero-temperature limit 
& $\mathrm{HQRE}(\boldsymbol\alpha^{(n)})\to \mathrm{BNE}$ (accumulation points) 
& App.~\ref{app:zero_temp} / Lem.~\ref{lem:zero_temp_full} \\
\midrule
\multicolumn{3}{l}{\textit{Learning dynamics}} \\
ELBO / BR-gap identity 
& Regularized BR gap equals discounted $\alpha$-weighted KL to logit BR 
& App.~\ref{app:elbo_equiv_proof} / Prop.~\ref{prop:elbo_equiv} \\
Monotone improvement 
& $J_{\mathrm{tot}}^{\boldsymbol\alpha}$ increases per step 
& App.~\ref{app:mono_improve_proof} / Thm.~\ref{thm:mono_improve} \\
Linear convergence (exact mirror)
& $D_\Psi(\boldsymbol\pi^{k+1},\boldsymbol\pi^\star)
  \le (1-\eta\mu/L_\Psi)
  D_\Psi(\boldsymbol\pi^{k},\boldsymbol\pi^\star)$
& App.~\ref{app:linear_rate_proof} / Thm.~\ref{thm:linear_rate} \\
\midrule
\multicolumn{3}{l}{\textit{Hierarchical scalability (Local--Global HQRE)}} \\
Existence 
& Nested fixed point $(\{\boldsymbol\pi_{C_k}\},g)$ exists 
& App.~\ref{app:hier_exist} / Thm.~\ref{thm:hier_existence} \\
Hierarchical coupling bound 
& $L_{\mathrm{hier}} \le W_{\max}^{\mathrm{loc}}C_Q + C_{\mathrm{occ}}(\cdots)$ 
& App.~\ref{app:hier_coupling} / Prop.~\ref{prop:hier_coupling_bound} \\
Uniqueness 
& Strong monotonicity with margin $\mu_{\mathrm{hier}}$ 
& App.~\ref{app:hier_mono_linear} / Thm.~\ref{thm:hier_unique} \\
Linear convergence (hierarchical exact mirror)
& $D_\Psi^{\mathrm{hier}}(\mathbf z^{k+1},\mathbf z^\star)
  \le (1-\eta\mu_{\mathrm{hier}}/L_\Psi^{\mathrm{hier}})
  D_\Psi^{\mathrm{hier}}(\mathbf z^{k},\mathbf z^\star)$
& App.~\ref{app:hier_linear_proof} / Thm.~\ref{thm:hier_linear} \\
\bottomrule
\end{tabular}
\end{table}

This appendix collects proofs for the learning--dynamics results summarized in Table~\ref{tab:main_results}:
the ELBO / best-response-gap equivalence (Prop.~\ref{prop:elbo_equiv}),
monotone progress of KL--mirror steps (Thm.~\ref{thm:mono_improve}),
linear convergence of KL--mirror dynamics in the strongly monotone regime (Thm.~\ref{thm:linear_rate}),
and their hierarchical counterparts (Thm.~\ref{thm:hier_linear}).
We also prove the soft-value factorization and weighted-advantage identities
(Prop.~\ref{prop:soft_value_factorization}, Lem.~\ref{lem:weighted_advantage})
used in Sec.~\ref{subsec:dice_ft}.

\paragraph{Local discounted occupancy over information states.}
For a joint policy $\boldsymbol\pi$, define the discounted occupancy measure on player $i$'s information states
$\mathcal I_i=\{(h_i^t,t)\}$ by
\[
\mathsf H_{i,\mathrm{disc}}^{\boldsymbol\pi}(h_i^t,t)
:=(1-\gamma)\gamma^t\,\Pr_{\boldsymbol\pi}(h_i^t),
\]
i.e., $\mathsf H_{i,\mathrm{disc}}^{\boldsymbol\pi}$ is the marginal of the joint discounted history occupancy
$\mathsf H_{\mathrm{disc}}^{\boldsymbol\pi}$ on the $i$-th coordinate.
Thus, for any bounded $f$ on $\mathcal I_i$,
\[
\sum_{t\ge 0}\gamma^t\,\mathbb E_{\boldsymbol\pi}[f(h_i^t,t)]
=(1-\gamma)^{-1}\,\mathbb E_{(h_i^t,t)\sim \mathsf H_{i,\mathrm{disc}}^{\boldsymbol\pi}}[f(h_i^t,t)].
\]

\subsection{ELBO decomposition and best-response gaps (Prop.~\ref{prop:elbo_equiv})}
\label{app:elbo_equiv_proof}

\begin{proposition}[ELBO decomposition and best-response-gap equivalence]
\label{prop:elbo_equiv}
Fix a player $i$ and opponents $\boldsymbol\pi_{-i}$, and assume $\alpha_{i,t}(h_i^t)>0$ for all $(h_i^t,t)$.
Let $\pi_i^\star$ denote an optimal policy for the \emph{regularized} best-response problem
\[
\pi_i^\star \in \arg\max_{\pi_i\in\Pi_i} J_i^{\boldsymbol\alpha}(\pi_i,\boldsymbol\pi_{-i}),
\]
where $J_i^{\boldsymbol\alpha}$ is defined in Eq.~\eqref{eq:J_i_def}. Then $\pi_i^\star$ is unique and satisfies the logit form
\[
\pi_i^\star(\cdot\mid h_i^t,t)\ \propto\ \exp\!\big(Q_i^{(\pi_i^\star,\boldsymbol\pi_{-i})}(h_i^t,\cdot,t)/\alpha_{i,t}(h_i^t)\big),
\]
which is consistent with Eq.~\eqref{eq:logit_def} specialized to the fixed-opponent problem.
Moreover, for any $\pi_i\in\Pi_i$,
\begin{equation}
\label{eq:elbo_global_correct}
J_i^{\boldsymbol\alpha}(\pi_i^\star,\boldsymbol\pi_{-i})
- J_i^{\boldsymbol\alpha}(\pi_i,\boldsymbol\pi_{-i})
= \frac{1}{1-\gamma}\,
\mathbb E_{(h_i^t,t)\sim\mathsf H_{i,\mathrm{disc}}^{(\pi_i,\boldsymbol\pi_{-i})}}
\Big[\alpha_{i,t}(h_i^t)\,
\mathrm{KL}\!\big(\pi_i(\cdot\mid h_i^t,t)\,\|\,\pi_i^\star(\cdot\mid h_i^t,t)\big)\Big].
\end{equation}
In particular, $\pi_i^\star$ is the unique maximizer of $J_i^{\boldsymbol\alpha}(\cdot,\boldsymbol\pi_{-i})$, and the joint fixed points of the logit map $\mathcal B^{\boldsymbol\alpha}$ coincide with Nash equilibria of the regularized game with payoffs $\{J_i^{\boldsymbol\alpha}\}_i$.
\end{proposition}

\paragraph{Soft Bellman objects for the playerwise regularized problem.}
Fix opponents $\boldsymbol\pi_{-i}$ and consider player $i$'s regularized objective
$J_i^{\boldsymbol\alpha}(\pi_i,\boldsymbol\pi_{-i})$ in Eq.~\eqref{eq:J_i_def}.
For any $\pi_i$, define the (regularized) value and action-value functions over information states $(h_i^t,t)$:
\begin{align*}
V_i^{\pi_i}(h_i^t,t)
&:=\mathbb E\Big[\sum_{u=t}^{\infty}\gamma^{\,u-t}
\big(R_i(s^u,\mathbf a^u)+\alpha_{i,u}(h_i^u)\,\mathcal H(\pi_i(\cdot\mid h_i^u,u))\big)\,\Big|\,h_i^t,\,\pi_i,\,\boldsymbol\pi_{-i}\Big],\\
Q_i^{\pi_i}(h_i^t,a_i,t)
&:=\mathbb E\Big[R_i(s^t,\mathbf a^t)+\gamma\,V_i^{\pi_i}(h_i^{t+1},t+1)\,\Big|\,h_i^t,\,a_i^t=a_i,\,\pi_i,\,\boldsymbol\pi_{-i}\Big].
\end{align*}
With these definitions,
\begin{equation}
\label{eq:soft_bellman_eval}
V_i^{\pi_i}(h_i^t,t)
=\mathbb E_{a_i\sim\pi_i(\cdot\mid h_i^t,t)}\!\big[Q_i^{\pi_i}(h_i^t,a_i,t)\big]
+\alpha_{i,t}(h_i^t)\,\mathcal H\big(\pi_i(\cdot\mid h_i^t,t)\big).
\end{equation}

\begin{lemma}[One-step ELBO identity]
\label{lem:elbo_one_step}
Fix an information state $(h_i^t,t)$ and a temperature $\alpha:=\alpha_{i,t}(h_i^t)>0$.
Let $q(a)=Q_i^{\pi_i^\star}(h_i^t,a,t)$ and let
$p^\star(\cdot\mid h_i^t,t)=\mathrm{softmax}(q(\cdot)/\alpha)$, i.e.,
\[
p^\star(a_i\mid h_i^t,t)
\propto \exp\!\big(Q_i^{\pi_i^\star}(h_i^t,a_i,t)/\alpha_{i,t}(h_i^t)\big).
\]
Then for any $p\in\Delta(\mathcal A_i)$,
\[
\sum_{a_i} p(a_i)\,q(a_i)+\alpha\,\mathcal H(p)
=
\alpha\log\sum_{a_i}\exp\!\big(q(a_i)/\alpha\big)
-\alpha\,\mathrm{KL}(p\|p^\star).
\]
\end{lemma}

\begin{proof}
This is the standard log-sum-exp / convex conjugacy identity.
Let $\alpha=\alpha_{i,t}(h_i^t)$, $q(a)=Q_i^{\pi_i^\star}(h_i^t,a,t)$, and
$Z=\sum_{a}\exp(q(a)/\alpha)$.
Define $p^\star(a)=\exp(q(a)/\alpha)/Z$. Then
\[
\sum_a p(a)q(a)+\alpha\sum_a p(a)\log p(a)
=\alpha\log Z-\alpha\,\mathrm{KL}(p\|p^\star).
\]
\end{proof}

\begin{proof}[Proof of Prop.~\ref{prop:elbo_equiv}]
\emph{Uniqueness and logit form.}
For fixed $\boldsymbol\pi_{-i}$ and $\alpha_{i,t}(h)>0$, the entropy-regularized control problem for player $i$
admits a unique optimal soft value function (the corresponding soft Bellman optimality operator is a $\gamma$-contraction
in the finite discounted setting), hence the induced optimal policy is unique.
The optimal policy satisfies the pointwise logit form with the optimal continuation values
$Q_i^{(\pi_i^\star,\boldsymbol\pi_{-i})}$, as stated.

\emph{pointwise identity at each information state.}
Fix $(h_i^t,t)$. By Eq.~\eqref{eq:soft_bellman_eval} evaluated at $\pi_i^\star$,
\[
V_i^{\pi_i^\star}(h_i^t,t)
=\mathbb E_{a_i\sim\pi_i^\star(\cdot\mid h_i^t,t)}\!\big[Q_i^{\pi_i^\star}(h_i^t,a_i,t)\big]
+\alpha_{i,t}(h_i^t)\,\mathcal H\big(\pi_i^\star(\cdot\mid h_i^t,t)\big).
\]
Applying Lem.~\ref{lem:elbo_one_step} with $p=\pi_i(\cdot\mid h_i^t,t)$ and $p^\star=\pi_i^\star(\cdot\mid h_i^t,t)$ gives
\begin{equation}
\label{eq:elbo_pointwise}
\begin{split}
V_i^{\pi_i^\star}(h_i^t,t) 
&= \mathbb E_{a_i\sim\pi_i(\cdot\mid h_i^t,t)}\!\big[Q_i^{\pi_i^\star}(h_i^t,a_i,t)\big] \\
&\quad +\alpha_{i,t}(h_i^t)\,\mathcal H\big(\pi_i(\cdot\mid h_i^t,t)\big) \\
&\quad +\alpha_{i,t}(h_i^t)\,\mathrm{KL}\!\big(\pi_i(\cdot\mid h_i^t,t)\,\|\,\pi_i^\star(\cdot\mid h_i^t,t)\big).
\end{split}
\end{equation}

\emph{telescope along trajectories under $(\pi_i,\boldsymbol\pi_{-i})$.}
Under fixed opponents $\boldsymbol\pi_{-i}$, $Q_i^{\pi_i^\star}$ satisfies the (unregularized) one-step recursion
\[
Q_i^{\pi_i^\star}(h_i^t,a_i,t)
=\mathbb E\big[R_i(s^t,\mathbf a^t)+\gamma\,V_i^{\pi_i^\star}(h_i^{t+1},t+1)\,\big|\,h_i^t,\,a_i^t=a_i,\,\pi_i^\star,\,\boldsymbol\pi_{-i}\big],
\]
which depends on $a_i^t$ and the environment/opponents at time $t$ but not on the \emph{behavior} policy used to sample $a_i^t$.
Therefore, taking $a_i^t\sim\pi_i(\cdot\mid h_i^t,t)$ and then unconditional expectation under $(\pi_i,\boldsymbol\pi_{-i})$ yields
\[
\mathbb E\Big[\mathbb E_{a_i\sim\pi_i(\cdot\mid h_i^t,t)}Q_i^{\pi_i^\star}(h_i^t,a_i,t)\Big]
=
\mathbb E\big[R_i(s^t,\mathbf a^t)+\gamma\,V_i^{\pi_i^\star}(h_i^{t+1},t+1)\big],
\]
where expectations are under $(\pi_i,\boldsymbol\pi_{-i})$.

Substitute this into Eq.~\eqref{eq:elbo_pointwise}, multiply by $\gamma^t$, and sum over $t\ge 0$.
The value terms telescope:
\[
\sum_{t\ge 0}\gamma^t\mathbb E\big[V_i^{\pi_i^\star}(h_i^t,t)\big]
-\sum_{t\ge 0}\gamma^{t+1}\mathbb E\big[V_i^{\pi_i^\star}(h_i^{t+1},t+1)\big]
=\mathbb E\big[V_i^{\pi_i^\star}(h_i^0,0)\big].
\]
Hence
\begin{align*}
\mathbb E\big[V_i^{\pi_i^\star}(h_i^0,0)\big]
&=
\sum_{t\ge 0}\gamma^t
\mathbb E\big[R_i(s^t,\mathbf a^t)+\alpha_{i,t}(h_i^t)\mathcal H(\pi_i(\cdot\mid h_i^t,t))\big]\\
&\qquad
+\sum_{t\ge 0}\gamma^t
\mathbb E\big[\alpha_{i,t}(h_i^t)\mathrm{KL}(\pi_i(\cdot\mid h_i^t,t)\|\pi_i^\star(\cdot\mid h_i^t,t))\big].
\end{align*}
The first sum equals $J_i^{\boldsymbol\alpha}(\pi_i,\boldsymbol\pi_{-i})$ by definition.
Also $\mathbb E[V_i^{\pi_i^\star}(h_i^0,0)]=J_i^{\boldsymbol\alpha}(\pi_i^\star,\boldsymbol\pi_{-i})$.
Rearranging and rewriting the discounted sum as an expectation under
$\mathsf H_{i,\mathrm{disc}}^{(\pi_i,\boldsymbol\pi_{-i})}$ yields Eq.~\eqref{eq:elbo_global_correct}.

\emph{Consequences.}
The RHS in Eq.~\eqref{eq:elbo_global_correct} is nonnegative and equals $0$ iff
$\pi_i(\cdot\mid h_i^t,t)=\pi_i^\star(\cdot\mid h_i^t,t)$ for
$\mathsf H_{i,\mathrm{disc}}^{(\pi_i,\boldsymbol\pi_{-i})}$-a.e.\ information state,
so $\pi_i^\star$ is the unique maximizer.
Finally, a joint profile is a Nash equilibrium of the regularized game iff each player is at its unique
regularized best response, which is equivalent to being a joint fixed point of the logit map $\mathcal B^{\boldsymbol\alpha}$.
\end{proof}

\subsection{Soft-value factorization and weighted advantages}
\label{app:soft_values_proof}

\begin{proposition}[Soft-value factorization]
\label{prop:soft_value_factorization}
Assume a linear monotone mixer of the form
\[
Q_{\mathrm{tot}}(s^t,\mathbf a^t,t)
= \sum_{i=1}^N w_i(s^t)\,Q_i(b_i^t,a_i^t,t)+b(s^t),
\qquad w_i(s^t)\ge 0,
\]
and a scalar temperature $\alpha_t>0$. Define the exact total soft value
\[
V_{\mathrm{tot}}^{\mathrm{exact}}(s^t,t)
:= \alpha_t \log \sum_{\mathbf a^t}\exp\big(Q_{\mathrm{tot}}(s^t,\mathbf a^t,t)/\alpha_t\big).
\]
Then it factorizes as
\begin{equation}
\label{eq:diceft_soft_values}
V_{\mathrm{tot}}^{\mathrm{exact}}(s^t,t)
= b(s^t) + \alpha_t \sum_{i=1}^N \log \sum_{a_i\in\mathcal A_i}
\exp\Big(\tfrac{w_i(s^t)}{\alpha_t}\,Q_i(b_i^t,a_i,t)\Big).
\end{equation}
This is the algebraic identity instantiated as Eq.~\eqref{eq:diceft_soft_values} in Sec.~\ref{subsec:dice_ft}
(with $a_i$ replaced by vocabulary tokens and $t$ replaced by token positions $\ell$).
\end{proposition}

\begin{proof}
By definition,
\[
\frac{Q_{\mathrm{tot}}(s^t,\mathbf a^t,t)}{\alpha_t}
= \frac{b(s^t)}{\alpha_t} + \sum_{i=1}^N \frac{w_i(s^t)}{\alpha_t}\,Q_i(b_i^t,a_i^t,t).
\]
Hence
\begin{align*}
\sum_{\mathbf a^t}\exp\!\Big(\frac{Q_{\mathrm{tot}}(s^t,\mathbf a^t,t)}{\alpha_t}\Big)
&= \exp\!\big(b(s^t)/\alpha_t\big)\,
\sum_{a_1}\cdots\sum_{a_N}
\exp\!\Big(\sum_{i=1}^N \tfrac{w_i(s^t)}{\alpha_t}\,Q_i(b_i^t,a_i^t,t)\Big)\\
&= \exp\!\big(b(s^t)/\alpha_t\big)\,
\prod_{i=1}^N\sum_{a_i}\exp\!\Big(\tfrac{w_i(s^t)}{\alpha_t}\,Q_i(b_i^t,a_i,t)\Big),
\end{align*}
since the exponent separates across $i$. Taking $\alpha_t\log(\cdot)$ yields the claim.
\end{proof}

\begin{lemma}[Weighted local soft values and soft advantages]
\label{lem:weighted_advantage}
Assume $\alpha_t>0$ and $w_i(s^t)>0$ for the indices $i$ under consideration.\footnote{
If $w_i(s^t)=0$ then the $i$-th term does not contribute to $Q_{\mathrm{tot}}$; one can define $w_i(s^t)V_i^{\mathrm{soft}}:=\alpha_t\log\sum_{a_i}\exp((w_i/\alpha_t)Q_i)$ and all identities remain valid.}
Define the local soft value
\[
V_i^{\mathrm{soft}}(s^t,b_i^t,t)
:=\frac{\alpha_t}{w_i(s^t)}\log\sum_{a_i}\exp\!\Big(\tfrac{w_i(s^t)}{\alpha_t}\,Q_i(b_i^t,a_i,t)\Big),
\]
and the corresponding weighted soft advantage
\[
A_i^{(w)}(s^t,b_i^t,a_i^t,t)
:=w_i(s^t)\Big(Q_i(b_i^t,a_i^t,t)-V_i^{\mathrm{soft}}(s^t,b_i^t,t)\Big).
\]
Then:
\begin{enumerate}[leftmargin=1.5em,itemsep=2pt]
\item[(i)] $V_{\mathrm{tot}}^{\mathrm{exact}}(s^t,t)
=b(s^t)+\sum_{i=1}^N w_i(s^t)V_i^{\mathrm{soft}}(s^t,b_i^t,t)$.
\item[(ii)] For the softmax policy
$\pi_i^{\mathrm{soft}}(a_i\mid b_i^t,t)\propto\exp(\tfrac{w_i(s^t)}{\alpha_t} Q_i(b_i^t,a_i,t))$,
the advantage admits the pointwise identity
\[
A_i^{(w)}(s^t,b_i^t,a_i^t,t)
=\alpha_t\,\log \pi_i^{\mathrm{soft}}(a_i^t\mid b_i^t,t),
\]
and therefore
\[
\mathbb E_{a_i\sim\pi_i^{\mathrm{soft}}}\!\big[A_i^{(w)}(s^t,b_i^t,a_i,t)\big]
=-\alpha_t\,\mathcal H\big(\pi_i^{\mathrm{soft}}(\cdot\mid b_i^t,t)\big).
\]
\end{enumerate}
\end{lemma}

\begin{proof}
(i) follows by substituting the definition of $V_i^{\mathrm{soft}}$ into Prop.~\ref{prop:soft_value_factorization}.

(ii) Let $Z_i:=\sum_{a_i}\exp(\tfrac{w_i}{\alpha_t} Q_i(b_i^t,a_i,t))$. Then
\[
\log \pi_i^{\mathrm{soft}}(a_i\mid b_i^t,t)=\tfrac{w_i(s^t)}{\alpha_t} Q_i(b_i^t,a_i,t)-\log Z_i.
\]
Multiplying by $\alpha_t$ gives
\[
\alpha_t\log \pi_i^{\mathrm{soft}}(a_i\mid b_i^t,t)
= w_i(s^t)Q_i(b_i^t,a_i,t)-\alpha_t\log Z_i
= w_i(s^t)\big(Q_i(b_i^t,a_i,t)-V_i^{\mathrm{soft}}(s^t,b_i^t,t)\big),
\]
which is exactly $A_i^{(w)}$. Taking expectation under $\pi_i^{\mathrm{soft}}$ yields
$\mathbb E[\log\pi_i^{\mathrm{soft}}]=-\mathcal H(\pi_i^{\mathrm{soft}})$ and hence the last identity.
\end{proof}

\subsection{Monotone progress for KL--mirror steps (Thm.~\ref{thm:mono_improve})}
\label{app:mono_improve_proof}

We provide a clean mirror-ascent progress bound in KL geometry. 

\begin{theorem}[Monotone progress of KL--mirror ascent]
\label{thm:mono_improve}
Let $J:\Pi\to\mathbb R$ be differentiable and $L_J$-smooth in the $\|\cdot\|_{1,2;w}$ norm, i.e., for all
$\boldsymbol\pi,\boldsymbol\pi'\in\Pi$,
\[
J(\boldsymbol\pi')\ \ge\
J(\boldsymbol\pi)+\langle \nabla J(\boldsymbol\pi),\boldsymbol\pi'-\boldsymbol\pi\rangle
-\frac{L_J}{2}\|\boldsymbol\pi'-\boldsymbol\pi\|_{1,2;w}^2.
\]
Assume the KL--Bregman divergence $D_\Psi$ (Eq.~\eqref{eq:weighted_bregman}) satisfies the weighted Pinsker lower bound
$D_\Psi(\boldsymbol\pi,\boldsymbol\pi')\ge \frac{\alpha_{\min}}{2}\|\boldsymbol\pi-\boldsymbol\pi'\|_{1,2;w}^2$
(as in Lem.~\ref{lem:weighted_pinsker}).
Consider the KL--mirror ascent step
\begin{equation}
\label{eq:mirror_update_appC}
\boldsymbol\pi^{k+1}
=\arg\max_{\tilde{\boldsymbol\pi}\in\Pi}
\Big\{\langle \nabla J(\boldsymbol\pi^k),\tilde{\boldsymbol\pi}-\boldsymbol\pi^k\rangle
-\tfrac{1}{\eta_k}D_\Psi(\tilde{\boldsymbol\pi},\boldsymbol\pi^k)\Big\}.
\end{equation}
If $\eta_k\le \alpha_{\min}/(2L_J)$, then
\[
J(\boldsymbol\pi^{k+1})-J(\boldsymbol\pi^k)\ \ge\ \frac{1}{2\eta_k}\,D_\Psi(\boldsymbol\pi^{k+1},\boldsymbol\pi^k)\ \ge\ 0.
\]
\end{theorem}

\begin{lemma}[Mirror optimality inequality]
\label{lem:mirror_optimality}
For the update Eq.~\eqref{eq:mirror_update_appC},
\[
\langle \nabla J(\boldsymbol\pi^k),\boldsymbol\pi^{k+1}-\boldsymbol\pi^k\rangle
\ge \frac{1}{\eta_k}D_\Psi(\boldsymbol\pi^{k+1},\boldsymbol\pi^k).
\]
\end{lemma}
\begin{proof}
Since $\tilde{\boldsymbol\pi}=\boldsymbol\pi^k$ is feasible, optimality of $\boldsymbol\pi^{k+1}$ implies
\[
\langle \nabla J(\boldsymbol\pi^k),\boldsymbol\pi^{k+1}-\boldsymbol\pi^k\rangle
-\tfrac{1}{\eta_k}D_\Psi(\boldsymbol\pi^{k+1},\boldsymbol\pi^k)\ \ge\ 0,
\]
which yields the claim.
\end{proof}

\begin{proof}[Proof of Thm.~\ref{thm:mono_improve}]
By $L_J$-smoothness applied with $\boldsymbol\pi=\boldsymbol\pi^k$ and $\boldsymbol\pi'=\boldsymbol\pi^{k+1}$,
\[
J(\boldsymbol\pi^{k+1})-J(\boldsymbol\pi^k)
\ge \langle \nabla J(\boldsymbol\pi^k),\boldsymbol\pi^{k+1}-\boldsymbol\pi^k\rangle
-\frac{L_J}{2}\|\boldsymbol\pi^{k+1}-\boldsymbol\pi^k\|_{1,2;w}^2.
\]
Use Lem.~\ref{lem:mirror_optimality} and the Pinsker lower bound to obtain
\[
J(\boldsymbol\pi^{k+1})-J(\boldsymbol\pi^k)
\ge \frac{1}{\eta_k}D_\Psi(\boldsymbol\pi^{k+1},\boldsymbol\pi^k)
-\frac{L_J}{\alpha_{\min}}D_\Psi(\boldsymbol\pi^{k+1},\boldsymbol\pi^k)
=\Big(\frac{1}{\eta_k}-\frac{L_J}{\alpha_{\min}}\Big)D_\Psi(\boldsymbol\pi^{k+1},\boldsymbol\pi^k).
\]
If $\eta_k\le \alpha_{\min}/(2L_J)$, then
$\frac{1}{\eta_k}-\frac{L_J}{\alpha_{\min}}\ge \frac{1}{2\eta_k}$,
hence
\[
J(\boldsymbol\pi^{k+1})-J(\boldsymbol\pi^k)\ \ge\ \frac{1}{2\eta_k}\,D_\Psi(\boldsymbol\pi^{k+1},\boldsymbol\pi^k)\ \ge\ 0.
\]
\end{proof}

\subsection{Linear convergence in KL geometry (Thm.~\ref{thm:linear_rate})}
\label{app:linear_rate_proof}

We now give a self-contained linear-rate bound for KL--mirror dynamics in the strongly monotone regime.
The argument is standard but we spell out the key inequalities and constants needed for a publishable statement.

\begin{theorem}[Linear convergence of KL--mirror dynamics]
\label{thm:linear_rate}
Let $F:\Pi\to\mathcal X$ be a field and suppose the variational inequality $\mathrm{VI}(F,\Pi)$ has a unique solution
$\boldsymbol\pi^\star$ (e.g., when $F$ is strongly monotone).
Assume:
\begin{enumerate}[leftmargin=1.25em,itemsep=2pt]
\item[(i)] (\emph{Strong monotonicity}) There exists $\mu>0$ such that for all $\boldsymbol\pi,\boldsymbol\pi'\in\Pi$,
\[
\langle F(\boldsymbol\pi)-F(\boldsymbol\pi'),\boldsymbol\pi-\boldsymbol\pi'\rangle
\ge \mu\|\boldsymbol\pi-\boldsymbol\pi'\|_{1,2;w}^2.
\]
\item[(ii)] (\emph{Lipschitzness}) There exists $L_F<\infty$ such that
\[
\|F(\boldsymbol\pi)-F(\boldsymbol\pi')\|_{\ast}\le L_F\|\boldsymbol\pi-\boldsymbol\pi'\|_{1,2;w},
\]
where $\|\cdot\|_\ast$ is the dual norm of $\|\cdot\|_{1,2;w}$.
\item[(iii)] (\emph{Local quadratic KL geometry}) There exists
$L_\Psi<\infty$ such that for all
$\boldsymbol\pi,\boldsymbol\pi'$ in a neighborhood containing
$\boldsymbol\pi^\star$ and the iterates,
\[
\frac{\alpha_{\min}}{2}\|\boldsymbol\pi-\boldsymbol\pi'\|_{1,2;w}^2
\ \le\ D_\Psi(\boldsymbol\pi,\boldsymbol\pi')\ \le\ \frac{L_\Psi}{2}\|\boldsymbol\pi-\boldsymbol\pi'\|_{1,2;w}^2.
\]
A sufficient condition is iterate-interiority:
if every local action probability in that region is bounded below by
$\nu_{\min}>0$, then one may take
$L_\Psi = \alpha_{\max}/\nu_{\min}$.
\end{enumerate}
Consider the mirror step for $\mathrm{VI}(F,\Pi)$:
\begin{equation}
\label{eq:vi_mirror_step}
\boldsymbol\pi^{k+1}
=\arg\min_{\tilde{\boldsymbol\pi}\in\Pi}
\Big\{\langle F(\boldsymbol\pi^k),\tilde{\boldsymbol\pi}-\boldsymbol\pi^k\rangle
+\tfrac{1}{\eta}D_\Psi(\tilde{\boldsymbol\pi},\boldsymbol\pi^k)\Big\}.
\end{equation}
If the stepsize satisfies
\[
0<\eta \ \le\ \min\Big\{\frac{\mu\,\alpha_{\min}^2}{L_F^2\,L_\Psi},\ \frac{\alpha_{\min}}{L_F}\Big\},
\]
then the iterates converge linearly in KL geometry:
\[
D_\Psi(\boldsymbol\pi^\star,\boldsymbol\pi^{k+1})
\le \rho\,D_\Psi(\boldsymbol\pi^\star,\boldsymbol\pi^{k}),
\qquad
\rho:=1-\eta\,\frac{\mu}{L_\Psi}\ \in(0,1).
\]
\end{theorem}

\begin{remark}[Why the KL upper quadratic bound is local]
\label{rem:kl_quadratic}
For entropy / KL mirror maps, the lower quadratic bound is global but
the upper quadratic bound is only local on the simplex.
If all policies in the iterate region satisfy
$\pi_i(a\mid h,t)\ge \nu_{\min}>0$ for every relevant coordinate,
then the Hessian of the negative-entropy mirror map is bounded above by
$1/\nu_{\min}$ coordinatewise, yielding
\[
D_\Psi(\boldsymbol\pi,\boldsymbol\pi')
\le \frac{\alpha_{\max}}{2\nu_{\min}}
\|\boldsymbol\pi-\boldsymbol\pi'\|_{1,2;w}^2.
\]
Practically, DICE-FT promotes this interior regime via positive
entropies, KL trust regions, and step rejection.
\end{remark}

\begin{lemma}[Three-point inequality for VI mirror steps]
\label{lem:vi_three_point}
Let $\boldsymbol\pi^{k+1}$ be defined by Eq.~\eqref{eq:vi_mirror_step}. Then for any $\boldsymbol\pi\in\Pi$,
\[
\langle F(\boldsymbol\pi^k),\boldsymbol\pi^{k+1}-\boldsymbol\pi\rangle
\le \tfrac{1}{\eta}\Big(D_\Psi(\boldsymbol\pi,\boldsymbol\pi^k)-D_\Psi(\boldsymbol\pi,\boldsymbol\pi^{k+1})-D_\Psi(\boldsymbol\pi^{k+1},\boldsymbol\pi^k)\Big).
\]
\end{lemma}

\begin{proof}
This is the standard Bregman three-point inequality for proximal/mirror steps.
It follows from the first-order optimality conditions of Eq.~\eqref{eq:vi_mirror_step} and the Bregman identity associated with $\Psi$.
\end{proof}

\begin{proof}[Proof of Thm.~\ref{thm:linear_rate}]
Apply Lem.~\ref{lem:vi_three_point} with $\boldsymbol\pi=\boldsymbol\pi^\star$:
\begin{equation}
\label{eq:rate_start}
\eta\,\langle F(\boldsymbol\pi^k),\boldsymbol\pi^{k+1}-\boldsymbol\pi^\star\rangle
\le D_\Psi(\boldsymbol\pi^\star,\boldsymbol\pi^k)-D_\Psi(\boldsymbol\pi^\star,\boldsymbol\pi^{k+1})-D_\Psi(\boldsymbol\pi^{k+1},\boldsymbol\pi^k).
\end{equation}
Since $\boldsymbol\pi^\star$ solves $\mathrm{VI}(F,\Pi)$, we have $\langle F(\boldsymbol\pi^\star),\boldsymbol\pi^{k+1}-\boldsymbol\pi^\star\rangle\ge 0$.
Therefore,
\[
\langle F(\boldsymbol\pi^k),\boldsymbol\pi^{k+1}-\boldsymbol\pi^\star\rangle
\ge \langle F(\boldsymbol\pi^k)-F(\boldsymbol\pi^\star),\boldsymbol\pi^{k+1}-\boldsymbol\pi^\star\rangle.
\]
Decompose:
\begin{align*}
\langle F(\boldsymbol\pi^k)-F(\boldsymbol\pi^\star),\boldsymbol\pi^{k+1}-\boldsymbol\pi^\star\rangle
&=\langle F(\boldsymbol\pi^k)-F(\boldsymbol\pi^\star),\boldsymbol\pi^k-\boldsymbol\pi^\star\rangle
+\langle F(\boldsymbol\pi^k)-F(\boldsymbol\pi^\star),\boldsymbol\pi^{k+1}-\boldsymbol\pi^k\rangle\\
&\ge \mu\|\boldsymbol\pi^k-\boldsymbol\pi^\star\|_{1,2;w}^2
-\|F(\boldsymbol\pi^k)-F(\boldsymbol\pi^\star)\|_\ast\,\|\boldsymbol\pi^{k+1}-\boldsymbol\pi^k\|_{1,2;w}\\
&\ge \mu\|\boldsymbol\pi^k-\boldsymbol\pi^\star\|_{1,2;w}^2
- L_F\|\boldsymbol\pi^k-\boldsymbol\pi^\star\|_{1,2;w}\,\|\boldsymbol\pi^{k+1}-\boldsymbol\pi^k\|_{1,2;w},
\end{align*}
using strong monotonicity and Lipschitzness plus Hölder's inequality.

Insert this lower bound into Eq.~\eqref{eq:rate_start} and use the lower quadratic bound on $D_\Psi(\boldsymbol\pi^{k+1},\boldsymbol\pi^k)$:
\[
D_\Psi(\boldsymbol\pi^{k+1},\boldsymbol\pi^k)\ \ge\ \frac{\alpha_{\min}}{2}\|\boldsymbol\pi^{k+1}-\boldsymbol\pi^k\|_{1,2;w}^2.
\]
We obtain
\begin{align*}
D_\Psi(\boldsymbol\pi^\star,\boldsymbol\pi^{k+1})
&\le D_\Psi(\boldsymbol\pi^\star,\boldsymbol\pi^k)
-\eta\mu\|\boldsymbol\pi^k-\boldsymbol\pi^\star\|_{1,2;w}^2 \\
&\quad +\eta L_F\|\boldsymbol\pi^k-\boldsymbol\pi^\star\|_{1,2;w}\,\|\boldsymbol\pi^{k+1}-\boldsymbol\pi^k\|_{1,2;w} \\
&\quad -\frac{\alpha_{\min}}{2}\|\boldsymbol\pi^{k+1}-\boldsymbol\pi^k\|_{1,2;w}^2.
\end{align*}
For fixed $a:=\|\boldsymbol\pi^k-\boldsymbol\pi^\star\|_{1,2;w}$, the expression
$\eta L_F a\,b-\frac{\alpha_{\min}}{2}b^2$ in $b:=\|\boldsymbol\pi^{k+1}-\boldsymbol\pi^k\|_{1,2;w}$ is maximized at
$b^\star=\eta L_F a/\alpha_{\min}$ with maximum value $\frac{\eta^2L_F^2}{2\alpha_{\min}}a^2$.
Thus,
\[
D_\Psi(\boldsymbol\pi^\star,\boldsymbol\pi^{k+1})
\le D_\Psi(\boldsymbol\pi^\star,\boldsymbol\pi^k)
-\eta\Big(\mu-\frac{\eta L_F^2}{2\alpha_{\min}}\Big)\|\boldsymbol\pi^k-\boldsymbol\pi^\star\|_{1,2;w}^2.
\]
Using the upper quadratic bound $D_\Psi(\boldsymbol\pi^\star,\boldsymbol\pi^k)\le \frac{L_\Psi}{2}\|\boldsymbol\pi^k-\boldsymbol\pi^\star\|_{1,2;w}^2$
implies
\[
\|\boldsymbol\pi^k-\boldsymbol\pi^\star\|_{1,2;w}^2\ \ge\ \frac{2}{L_\Psi}\,D_\Psi(\boldsymbol\pi^\star,\boldsymbol\pi^k).
\]
Therefore,
\[
D_\Psi(\boldsymbol\pi^\star,\boldsymbol\pi^{k+1})
\le \Big(1-\eta\,\frac{2}{L_\Psi}\Big(\mu-\frac{\eta L_F^2}{2\alpha_{\min}}\Big)\Big)\,D_\Psi(\boldsymbol\pi^\star,\boldsymbol\pi^k).
\]
If $\eta\le \mu\alpha_{\min}^2/(L_F^2L_\Psi)$, then $\frac{\eta L_F^2}{2\alpha_{\min}}\le \frac{\mu}{2}\cdot\frac{\alpha_{\min}}{L_\Psi}\le \frac{\mu}{2}$,
and in particular
\[
\eta\,\frac{2}{L_\Psi}\Big(\mu-\frac{\eta L_F^2}{2\alpha_{\min}}\Big)\ \ge\ \eta\,\frac{\mu}{L_\Psi}.
\]
Hence,
\[
D_\Psi(\boldsymbol\pi^\star,\boldsymbol\pi^{k+1})
\le \Big(1-\eta\,\frac{\mu}{L_\Psi}\Big)\,D_\Psi(\boldsymbol\pi^\star,\boldsymbol\pi^k),
\]
which is the claimed linear contraction with $\rho=1-\eta\mu/L_\Psi\in(0,1)$.
\end{proof}

\begin{theorem}[Inexact linear recurrence for mirror steps]
\label{thm:inexact_linear_rate}
Assume the conditions of Theorem~\ref{thm:linear_rate}.
Let the approximate iterate be defined by
\begin{equation}
\label{eq:vi_mirror_step_inexact}
\boldsymbol\pi^{k+1}
=\arg\min_{\tilde{\boldsymbol\pi}\in\Pi}
\Big\{\langle F(\boldsymbol\pi^k)+e^k,
\tilde{\boldsymbol\pi}-\boldsymbol\pi^k\rangle
+\tfrac{1}{\eta}D_\Psi(\tilde{\boldsymbol\pi},\boldsymbol\pi^k)\Big\}.
\end{equation}
If
\[
0<\eta\le
\min\Big\{\frac{\mu\,\alpha_{\min}^2}{L_F^2L_\Psi},\,
\frac{\alpha_{\min}}{L_F},\,
\frac{\alpha_{\min}\mu}{4L_F^2}\Big\},
\]
then
\[
D_\Psi(\boldsymbol\pi^\star,\boldsymbol\pi^{k+1})
\le \rho\,D_\Psi(\boldsymbol\pi^\star,\boldsymbol\pi^k)
+ C_{\mathrm{err}}\,\|e^k\|_\ast^2,
\]
where
\[
\rho:=1-\eta\,\frac{\mu}{L_\Psi},
\qquad
C_{\mathrm{err}}:=\frac{\eta}{\mu}+\frac{\eta^2}{\alpha_{\min}}.
\]
\end{theorem}

\begin{proof}
Apply Lem.~\ref{lem:vi_three_point} to the inexact step
Eq.~\eqref{eq:vi_mirror_step_inexact} with
$\boldsymbol\pi=\boldsymbol\pi^\star$:
\[
\eta\,\langle F(\boldsymbol\pi^k)+e^k,
\boldsymbol\pi^{k+1}-\boldsymbol\pi^\star\rangle
\le
D_\Psi(\boldsymbol\pi^\star,\boldsymbol\pi^k)
-D_\Psi(\boldsymbol\pi^\star,\boldsymbol\pi^{k+1})
-D_\Psi(\boldsymbol\pi^{k+1},\boldsymbol\pi^k).
\]
Since $\boldsymbol\pi^\star$ solves $\mathrm{VI}(F,\Pi)$,
\[
\langle F(\boldsymbol\pi^\star),
\boldsymbol\pi^{k+1}-\boldsymbol\pi^\star\rangle\ge 0.
\]
Proceeding as in the proof of Theorem~\ref{thm:linear_rate} yields
\[
D_\Psi(\boldsymbol\pi^\star,\boldsymbol\pi^{k+1})
\le D_\Psi(\boldsymbol\pi^\star,\boldsymbol\pi^k)
-\eta\mu a^2
+\eta L_Fab
+\eta\|e^k\|_\ast(a+b)
-\frac{\alpha_{\min}}{2}b^2,
\]
where
$a:=\|\boldsymbol\pi^k-\boldsymbol\pi^\star\|_{1,2;w}$ and
$b:=\|\boldsymbol\pi^{k+1}-\boldsymbol\pi^k\|_{1,2;w}$.
Use Young's inequality in the form
\[
\eta L_Fab \le \frac{\eta\mu}{4}a^2 + \frac{\alpha_{\min}}{4}b^2,
\]
which is valid under
$\eta\le \alpha_{\min}\mu/(4L_F^2)$,
and
\[
\eta\|e^k\|_\ast a
\le \frac{\eta\mu}{4}a^2 + \frac{\eta}{\mu}\|e^k\|_\ast^2,
\qquad
\eta\|e^k\|_\ast b
\le \frac{\alpha_{\min}}{4}b^2 + \frac{\eta^2}{\alpha_{\min}}\|e^k\|_\ast^2.
\]
Substituting gives
\[
D_\Psi(\boldsymbol\pi^\star,\boldsymbol\pi^{k+1})
\le D_\Psi(\boldsymbol\pi^\star,\boldsymbol\pi^k)
-\frac{\eta\mu}{2}a^2
+\Big(\frac{\eta}{\mu}+\frac{\eta^2}{\alpha_{\min}}\Big)
\|e^k\|_\ast^2.
\]
Finally, the upper quadratic bound gives
$a^2\ge 2D_\Psi(\boldsymbol\pi^\star,\boldsymbol\pi^k)/L_\Psi$,
so
\[
D_\Psi(\boldsymbol\pi^\star,\boldsymbol\pi^{k+1})
\le \Big(1-\eta\,\frac{\mu}{L_\Psi}\Big)
D_\Psi(\boldsymbol\pi^\star,\boldsymbol\pi^k)
+\Big(\frac{\eta}{\mu}+\frac{\eta^2}{\alpha_{\min}}\Big)
\|e^k\|_\ast^2,
\]
which is the claim.
\end{proof}

\subsection{Hierarchical mirror dynamics (Thm.~\ref{thm:hier_linear})}
\label{app:hier_linear_proof}

\begin{theorem}[Hierarchical linear convergence]
\label{thm:hier_linear}
Let $\mathbf z=(\boldsymbol\pi,g)$ be the hierarchical variable with composite divergence
$D_\Psi^{\mathrm{hier}}$ (as defined in Appendix~\ref{app:hier}).
Assume the hierarchical field $F_{\mathrm{hier}}$ is $\mu_{\mathrm{hier}}$-strongly monotone and $L_{\mathrm{hier}}$-Lipschitz
in a compatible block norm, and that $D_\Psi^{\mathrm{hier}}$ is locally quadratic on the iterates
(upper/lower quadratic bounds analogous to Thm.~\ref{thm:linear_rate} with constants $\alpha_{\min}^{\mathrm{hier}}$ and $L_\Psi^{\mathrm{hier}}$).
Then the blockwise KL--mirror step
\[
\mathbf z^{k+1}
=\arg\min_{\tilde{\mathbf z}\in\mathcal Z}
\Big\{\langle F_{\mathrm{hier}}(\mathbf z^k),\tilde{\mathbf z}-\mathbf z^k\rangle
+\tfrac{1}{\eta}D_\Psi^{\mathrm{hier}}(\tilde{\mathbf z},\mathbf z^k)\Big\}
\]
converges linearly:
\[
D_\Psi^{\mathrm{hier}}(\mathbf z^\star,\mathbf z^{k+1})
\le \Big(1-\eta\,\frac{\mu_{\mathrm{hier}}}{L_\Psi^{\mathrm{hier}}}\Big)\,D_\Psi^{\mathrm{hier}}(\mathbf z^\star,\mathbf z^{k}),
\]
for stepsizes $\eta$ chosen analogously to Thm.~\ref{thm:linear_rate}.
\end{theorem}

\begin{proof}
The proof is identical to Thm.~\ref{thm:linear_rate} after replacing $(F,\Pi,D_\Psi,\|\cdot\|_{1,2;w})$
with the hierarchical counterparts $(F_{\mathrm{hier}},\mathcal Z,D_\Psi^{\mathrm{hier}},\|\cdot\|_{\mathrm{hier}})$.
Strong monotonicity uses the margin $\mu_{\mathrm{hier}}$ (Sec.~\ref{subsec:hier_theory}, Eq.~\eqref{eq:mu_hier}),
while Lipschitzness uses the structural coupling bounds (Prop.~\ref{prop:hier_coupling_bound} and Appendix~\ref{app:hier_coupling}).
Local quadratic bounds for the composite KL divergence hold on interior compact subsets, similarly to Remark~\ref{rem:kl_quadratic}.
\end{proof}

\subsection{Additional experiment results}
\label{additional_experiment}
\begin{table*}[htbp]
\centering
\caption{TravelPlanner full results with selected leaderboard baselines.
Delivery Rate measures format compliance; Commonsense and Hard Constraint pass rates (PR) are reported at both micro (Mi, per-constraint) and macro (Ma, per-plan) granularity; Final PR requires all constraints satisfied simultaneously.
Best result per column within each setting is \textbf{bold}.
DICE rows are highlighted.}
\label{tab:travel-planner-full}
\vspace{2pt}
\footnotesize
\setlength\tabcolsep{3.5pt}
\resizebox{\textwidth}{!}{%
\begin{tabular}{@{}l ccccc c ccccc c@{}}
\toprule
& \multicolumn{6}{c}{\textbf{Validation (\#180)}} & \multicolumn{6}{c}{\textbf{Test (\#1,000)}} \\
\cmidrule(lr){2-7}\cmidrule(lr){8-13}
& \multirow{2}{*}{\begin{tabular}[c]{@{}c@{}}Delivery\\Rate\end{tabular}}
& \multicolumn{2}{c}{\begin{tabular}[c]{@{}c@{}}Commonsense\\PR\end{tabular}}
& \multicolumn{2}{c}{\begin{tabular}[c]{@{}c@{}}Hard Constr.\\PR\end{tabular}}
& \multirow{2}{*}{\begin{tabular}[c]{@{}c@{}}Final\\PR\end{tabular}}
& \multirow{2}{*}{\begin{tabular}[c]{@{}c@{}}Delivery\\Rate\end{tabular}}
& \multicolumn{2}{c}{\begin{tabular}[c]{@{}c@{}}Commonsense\\PR\end{tabular}}
& \multicolumn{2}{c}{\begin{tabular}[c]{@{}c@{}}Hard Constr.\\PR\end{tabular}}
& \multirow{2}{*}{\begin{tabular}[c]{@{}c@{}}Final\\PR\end{tabular}} \\
\cmidrule(lr){3-4}\cmidrule(lr){5-6}\cmidrule(lr){9-10}\cmidrule(lr){11-12}
& & {Mi} & {Ma} & {Mi} & {Ma} & & & {Mi} & {Ma} & {Mi} & {Ma} & \\
\midrule
Greedy Search & 100 & 74.4 & 0 & 60.8 & 37.8 & 0 & 100 & 72.0 & 0 & 52.4 & 31.8 & 0 \\
\midrule
\addlinespace[4pt]
\multicolumn{13}{@{}l}{\textit{Two-stage}} \\
\addlinespace[2pt]
Mixtral-8x7B-MoE & 49.4 & 30.0 & 0 & 1.2 & 0.6 & 0 & 51.2 & 32.2 & 0.2 & 0.7 & 0.4 & 0 \\
Gemini Pro & 28.9 & 18.9 & 0 & 0.5 & 0.6 & 0 & 39.1 & 24.9 & 0 & 0.6 & 0.1 & 0 \\
GPT-3.5-Turbo & 86.7 & 54.0 & 0 & 0 & 0 & 0 & 91.8 & 57.9 & 0 & 0.5 & 0.6 & 0 \\
GPT-4-Turbo & 89.4 & 61.1 & 2.8 & 15.2 & 10.6 & 0.6 & 93.1 & 63.3 & 2.0 & 10.5 & 5.5 & 0.6 \\
GPT-o3 & 92.8 & 64.7 & 4.2 & 18.9 & 12.8 & 1.4 & 95.6 & 68.9 & 5.1 & 14.3 & 8.7 & 1.9 \\
\addlinespace[2pt]
Debate (GPT-4) @3round & 95.2 & 67.3 & 6.7 & 22.7 & 13.1 & 2.3 & 97.8 & 72.4 & 11.3 & 17.4 & 12.1 & 3.7 \\
MA-Debate (Qwen3-32B) & 94.4 & 65.8 & 5.1 & 19.8 & 11.3 & 1.7 & 96.7 & 70.3 & 8.2 & 16.2 & 9.8 & 2.8 \\
MA-FT (Qwen3-32B) & 97.2 & 67.9 & 7.3 & 23.4 & 14.6 & 3.1 & 98.6 & 73.1 & 12.7 & 21.8 & 13.2 & 4.3 \\
Qwen3-235B-Inst & 91.1 & 63.4 & 3.6 & 17.3 & 11.9 & 1.1 & 94.3 & 66.2 & 3.8 & 12.7 & 7.1 & 1.3 \\
MA-Debate (Qwen3-235B) & 96.7 & 69.8 & 9.4 & 27.6 & 16.8 & 3.8 & 98.9 & 76.7 & 16.2 & 24.1 & 14.9 & 5.7 \\
\addlinespace[2pt]
\rowcolor{gray!8} DICE (GPT-4)        & \textbf{100} & \textbf{71.4} & \textbf{15.6} & \textbf{32.1} & \textbf{25.7} & \textbf{7.2} & \textbf{100} & \textbf{82.1} & \textbf{26.6} & \textbf{32.4} & \textbf{17.6} & \textbf{9.3} \\
\rowcolor{gray!8} DICE-PC (Qwen3-32B) & 98.3 & 68.2 & 8.9 & 26.3 & 16.7 & 3.9 & 99.1 & 75.8 & 15.1 & 24.7 & 14.3 & 5.4 \\
\rowcolor{gray!8} DICE-FT (Qwen3-32B) & \textbf{100} & 70.1 & 11.7 & 29.4 & 19.8 & 5.6 & \textbf{100} & 78.6 & 18.9 & 28.1 & 15.9 & 7.1 \\
\midrule
\addlinespace[4pt]
\multicolumn{13}{@{}l}{\textit{Sole-planning}} \\
\addlinespace[2pt]
\midrule
Direct$_{\text{GPT-3.5-Turbo}}$    & 100 & 60.2 & 4.4 & 11.0 & 2.8 & 0   & 100 & 59.5 & 2.7 & 9.5 & 4.4 & 0.6 \\
CoT$_{\text{GPT-3.5-Turbo}}$       & 100 & 66.3 & 3.3 & 11.9 & 5.0 & 0   & 100 & 64.4 & 2.3 & 9.8 & 3.8 & 0.4 \\
ReAct$_{\text{GPT-3.5-Turbo}}$     & 82.2& 47.6 & 3.9 & 11.4 & 6.7 & 0.6 & 81.6& 45.9 & 2.5 & 10.7& 3.1 & 0.7 \\
Reflexion$_{\text{GPT-3.5-Turbo}}$ & 93.9& 53.8 & 2.8 & 11.0 & 2.8 & 0   & 92.1& 52.1 & 2.2 & 9.9 & 3.8 & 0.6 \\
Direct$_{\text{Mixtral-8x7B}}$     & 100 & 68.1 & 5.0 & 3.3  & 1.1 & 0   & 99.3& 67.0 & 3.7 & 3.9 & 1.6 & 0.7 \\
Direct$_{\text{Gemini Pro}}$        & 93.9& 65.0 & 8.3 & 9.3  & 4.4 & 0.6 & 93.7& 64.7 & 7.9 & 10.6& 4.7 & 2.1 \\
Direct$_{\text{GPT-4-Turbo}}$      & \textbf{100}& 80.4& 17.2& 47.1& 22.2& 4.4  & \textbf{100}& 80.6& 15.2& 44.3& 23.1& 4.4 \\
Direct$_{\text{GPT-o3}}$           & \textbf{100}& 82.1& 19.7& 49.8& 24.6& 6.1  & \textbf{100}& 82.9& 17.8& 46.7& 25.3& 5.8 \\
\addlinespace[2pt]
Debate (GPT-4) @3round     & 97.7& 78.9& 15.6& 43.3& 20.6& 6.7  & 98.2& 79.5& 18.8& 41.7& 22.9& 7.1 \\
MA-Debate (Qwen3-32B)      & 98.9& 79.6& 16.7& 42.7& 21.3& 7.8  & 99.3& 80.3& 18.1& 41.2& 23.4& 8.4 \\
MA-FT (Qwen3-32B)          & \textbf{100}& 80.8& 17.9& 45.1& 22.8& 9.1  & \textbf{100}& 81.4& 19.3& 43.8& 24.7& 9.7 \\
Qwen3-235B-Inst             & \textbf{100}& 81.3& 18.9& 48.7& 23.9& 6.7  & \textbf{100}& 81.9& 17.1& 45.9& 24.5& 6.3 \\
MA-Debate (Qwen3-235B)     & 99.4& 81.7& 19.4& 48.9& 25.1& 10.1 & 99.8& 82.6& 20.8& 47.3& 26.7& 11.3\\
\addlinespace[2pt]
\rowcolor{gray!8} DICE-PC (GPT-4)       & \textbf{100}& \textbf{83.3}& \textbf{22.2}& \textbf{51.7}& \textbf{27.8}& \textbf{12.9} & \textbf{100}& \textbf{84.2}& \textbf{23.5}& \textbf{49.8}& \textbf{28.7}& \textbf{15.2} \\
\rowcolor{gray!8} DICE-PC (Qwen3-32B)& \textbf{100}& 81.7& 18.3& 45.8& 23.4& 8.9  & \textbf{100}& 82.1& 19.7& 43.6& 24.8& 10.3\\
\rowcolor{gray!8} DICE-FT (Qwen3-32B)& \textbf{100}& 82.5& 20.1& 48.3& 25.6& 10.6 & \textbf{100}& 83.4& 21.4& 46.9& 26.5& 12.7\\
\bottomrule
\end{tabular}%
}
\end{table*}

\begin{figure*}[t!]
\centering

\begin{subfigure}[t]{0.32\textwidth}
\centering
\begin{tikzpicture}
\begin{axis}[
  width=\textwidth, height=4.5cm,
  xlabel={Training iteration},
  ylabel={Clipping hit rate (\%)},
  xmin=0, xmax=100, ymin=0, ymax=40,
  grid=major, grid style={dashed,gray!30},
  tick label style={font=\scriptsize},
  label style={font=\scriptsize},
  legend style={font=\tiny, at={(0.97,0.97)}, anchor=north east},
]
\addplot[only marks, blue!80!black, mark=*, mark size=2.5pt] coordinates {
  (1,35) (50,4.8) (100,3.4)
};
\draw[dashed,red!70!black,thick] (axis cs:50,0) -- (axis cs:50,40)
  node[pos=0.92,right,font=\tiny] {critic warm-up};
\end{axis}
\end{tikzpicture}
\caption{Text-Hanabi: clipping boundary hit rate decays from ${\sim}35\%$ to $<5\%$ after critic warm-up (${\sim}50$ iterations), consistent with $O(L\varepsilon^2)$ bias diminishing as policies stabilize.}
\label{fig:diag_hanabi_clip}
\end{subfigure}\hfill
\begin{subfigure}[t]{0.32\textwidth}
\centering
\begin{tikzpicture}
\begin{axis}[
  width=\textwidth, height=4.5cm,
  xlabel={Training iteration},
  ylabel={Relative gradient std},
  xmin=0, xmax=100, ymin=0, ymax=0.35,
  grid=major, grid style={dashed,gray!30},
  tick label style={font=\scriptsize},
  label style={font=\scriptsize},
  legend style={font=\tiny, at={(0.97,0.97)}, anchor=north east},
]
\addplot[only marks, red!80!black, mark=triangle*, mark size=2.5pt] coordinates {(1,0.30) (50,0.24) (100,0.24)};
\addlegendentry{$G{=}4$}
\addplot[only marks, blue!80!black, mark=square*, mark size=2.2pt] coordinates {(1,0.16) (50,0.12) (100,0.12)};
\addlegendentry{$G{=}8$}
\addplot[only marks, green!60!black, mark=o, mark size=2.2pt] coordinates {(1,0.08) (50,0.06) (100,0.06)};
\addlegendentry{$G{=}16$}
\end{axis}
\end{tikzpicture}
\caption{Text-Hanabi: within-group gradient relative std stabilizes at 0.24 ($G{=}4$), 0.12 ($G{=}8$), 0.06 ($G{=}16$), confirming $O(1/G)$ scaling.}
\label{fig:diag_hanabi_var}
\end{subfigure}\hfill
\begin{subfigure}[t]{0.32\textwidth}
\centering
\begin{tikzpicture}
\begin{axis}[
  width=\textwidth, height=4.5cm,
  xlabel={Training iteration},
  ylabel={Reconstruction loss},
  xmin=0, xmax=100, ymin=0, ymax=0.12,
  grid=major, grid style={dashed,gray!30},
  tick label style={font=\scriptsize},
  label style={font=\scriptsize},
  legend style={font=\tiny, at={(0.97,0.97)}, anchor=north east},
]
\addplot[only marks, blue!80!black, mark=*, mark size=2.5pt] coordinates {(1,0.10) (50,0.030) (100,0.030)};
\addlegendentry{recon.\ loss}
\addplot[thick, red!70!black, dashed, forget plot] coordinates {
  (0,0.07) (100,0.07)
};
\node[font=\tiny, red!70!black] at (axis cs:75,0.078) {$\mu \approx 0.07$};
\end{axis}
\end{tikzpicture}
\caption{Text-Hanabi: reconstruction loss plateaus at ${\sim}0.03$ after 50 iterations, well below the uniqueness margin $\mu \approx 0.07$ (dashed line).}
\label{fig:diag_hanabi_recon}
\end{subfigure}

\vspace{1em}

\begin{subfigure}[t]{0.32\textwidth}
\centering
\begin{tikzpicture}
\begin{axis}[
  width=\textwidth, height=4.5cm,
  xlabel={Training iteration},
  ylabel={Clipping hit rate (\%)},
  xmin=0, xmax=120, ymin=0, ymax=35,
  grid=major, grid style={dashed,gray!30},
  tick label style={font=\scriptsize},
  label style={font=\scriptsize},
  legend style={font=\tiny, at={(0.97,0.97)}, anchor=north east},
]
\addplot[only marks, blue!80!black, mark=*, mark size=2.5pt] coordinates {(1,28) (50,12) (100,7.2) (120,6.8)};
\end{axis}
\end{tikzpicture}
\caption{MATH-500: clipping boundary hit rate decays from ${\sim}28\%$ to $<8\%$ after ${\sim}80$ iterations, mirroring the Text-Hanabi trajectory on the full LLM backbone.}
\label{fig:diag_math_clip}
\end{subfigure}\hfill
\begin{subfigure}[t]{0.32\textwidth}
\centering
\begin{tikzpicture}
\begin{axis}[
  width=\textwidth, height=4.5cm,
  xlabel={Training iteration},
  ylabel={Relative gradient std},
  xmin=0, xmax=120, ymin=0, ymax=0.35,
  grid=major, grid style={dashed,gray!30},
  tick label style={font=\scriptsize},
  label style={font=\scriptsize},
  legend style={font=\tiny, at={(0.97,0.97)}, anchor=north east},
]
\addplot[only marks, red!80!black, mark=triangle*, mark size=2.5pt] coordinates {(1,0.29) (50,0.23) (100,0.23) (120,0.23)};
\addlegendentry{$G{=}4$}
\addplot[only marks, blue!80!black, mark=square*, mark size=2.2pt] coordinates {(1,0.15) (50,0.11) (100,0.11) (120,0.11)};
\addlegendentry{$G{=}8$}
\end{axis}
\end{tikzpicture}
\caption{MATH-500: within-group gradient relative std at $G{=}4$ (0.23) and $G{=}8$ (0.11), consistent with $O(1/G)$ scaling and matching Text-Hanabi values.}
\label{fig:diag_math_var}
\end{subfigure}\hfill
\begin{subfigure}[t]{0.32\textwidth}
\centering
\begin{tikzpicture}
\begin{axis}[
  width=\textwidth, height=4.5cm,
  xlabel={Training iteration},
  ylabel={KL divergence},
  xmin=0, xmax=120, ymin=0, ymax=0.35,
  grid=major, grid style={dashed,gray!30},
  tick label style={font=\scriptsize},
  label style={font=\scriptsize},
  legend style={font=\tiny, at={(0.97,0.97)}, anchor=north east},
]
\addplot[only marks, blue!80!black, mark=*, mark size=2.5pt] coordinates {(1,0.015) (50,0.014) (100,0.011) (120,0.011)};
\addlegendentry{$\widehat{\mathrm{KL}}_{\mathrm{old}}$}
\addplot[only marks, orange!80!black, mark=square*, mark size=2.2pt] coordinates {(1,0.02) (50,0.10) (100,0.14) (120,0.14)};
\addlegendentry{$\widehat{\mathrm{KL}}_{\mathrm{ref}}$}
\addplot[thick, blue!50!black, dashed, forget plot] coordinates {
  (0,0.05) (120,0.05)
};
\node[font=\tiny, blue!50!black] at (axis cs:95,0.065) {alarm: 0.05};
\addplot[thick, orange!50!black, dashed, forget plot] coordinates {
  (0,0.3) (120,0.3)
};
\node[font=\tiny, orange!50!black] at (axis cs:95,0.315) {alarm: 0.3};
\end{axis}
\end{tikzpicture}
\caption{MATH-500: $\widehat{\mathrm{KL}}_{\mathrm{old}}$ remains below 0.02 and $\widehat{\mathrm{KL}}_{\mathrm{ref}}$ below 0.15 throughout training, both well under their respective alarm thresholds (dashed lines).}
\label{fig:diag_math_kl}
\end{subfigure}

\caption{Empirical error diagnostic time series across training iterations.
Row~1: Text-Hanabi (lightweight decision layer, all three error components directly measurable).
Row~2: MATH-500 (full LLM backbone, two of three components directly measurable; KL diagnostics serve as indirect evidence for representation error).
These time series correspond to the summary statistics reported in Section~\ref{sec:algorithm} and provide visual confirmation that the measurable components of the DICE-FT field-error budget remain controlled during training.
All plotted points are exact checkpoint measurements, taken at iterations $\{1, 50, 100\}$ (Text-Hanabi) and $\{1, 50, 100, 120\}$ (MATH-500).
The qualitative trends (decreasing clip-overshoot activity, decreasing gradient variance, stable critic / representation proxy) are consistent with the inexact-mirror regime assumed by Corollary*~\ref{cor:approx_convergence}.}
\label{fig:error_diagnostics}
\end{figure*}
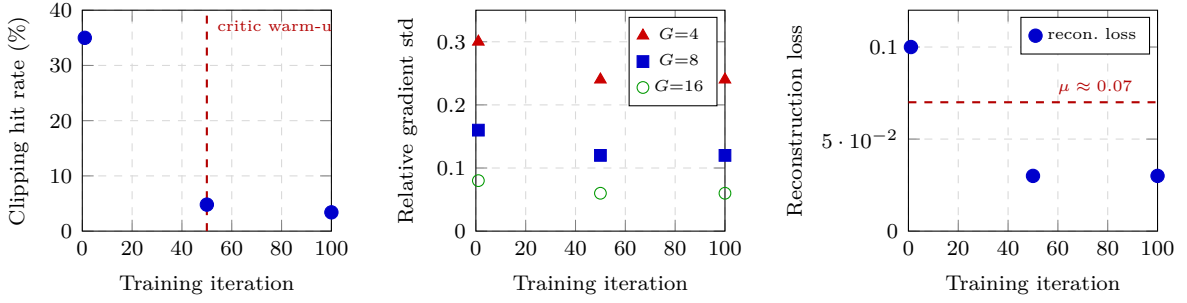
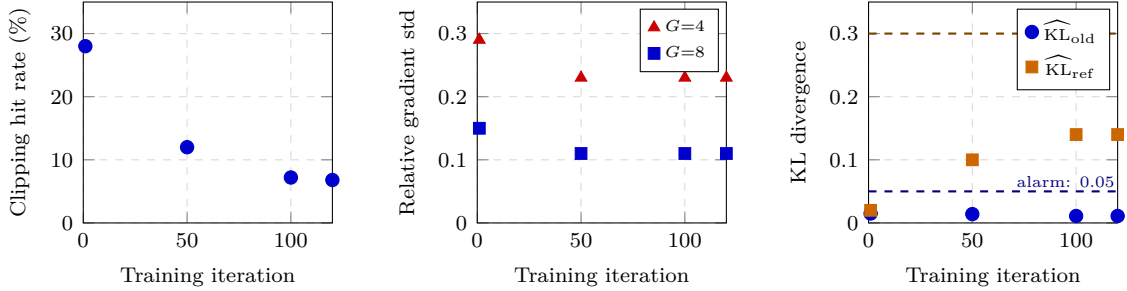

The revised field-error budget isolates the explicit $G$-dependence in
the sampling term $\sigma_g^2/G$.
Empirically, the within-group gradient relative standard deviation on
Text-Hanabi is approximately $0.24$ at $G{=}4$, $0.12$ at $G{=}8$, and
$0.06$ at $G{=}16$, consistent with the expected $1/\sqrt{G}$ trend.
This explains the diminishing returns beyond $G{=}8$: once the sampling
term falls below the representation / optimization residuals, further
increases in $G$ produce limited gains while linearly increasing rollout
cost.

\subsection{Field-error budget: seed-level analysis}
\label{app:neighborhood_detail}

Under the revised theory, the relevant measurable quantities are the
components of the field-error budget in
Eq.~\eqref{eq:prop2_field_budget} rather than a single scalar
$\delta$.
Across seeds on Text-Hanabi, the clip-overshoot term decreases during
training, the sampling term exhibits the predicted $1/G$ behavior, and
the critic / representation proxy remains stable after warm-up.
We therefore use these curves as diagnostics of the inexact-mirror
regime rather than claiming a universally calibrated neighborhood radius
in weighted KL units.

\subsection{Algorithmic details deferred from Section~\ref{sec:algorithm}}
\label{app:error_diagnostics_detail}

\paragraph{Heterogeneous temperature specification.}
Let $H^{\mathrm{ref}}_{i,t,\ell}$ denote the entropy of the frozen reference model $\pi_{\mathrm{ref}}(\cdot\mid b_i^t,\ell)$ and fix a threshold $h_0$; then
\[
\alpha_{i,t,\ell}=\alpha_{\mathrm{hi}}
\ \text{if}\ H^{\mathrm{ref}}_{i,t,\ell}\ge h_0,
\qquad
\alpha_{i,t,\ell}=\alpha_{\mathrm{lo}}\ \text{otherwise},
\]
with $0<\alpha_{\mathrm{lo}}\le \alpha_{\mathrm{hi}}$.
This schedule is driven only by information available at the time of generation and requires no additional modeling assumptions.

\paragraph{Group-relative control variate derivation.}
For each prompt, we sample a group of $G$ joint rollouts under $\pi^{\mathrm{old}}$ and define the group return $R^{(g)} := \frac{1}{N}\sum_{i} r_i^{(g)}$.
We compute a leave-one-out standardized return
\[
\bar R_{-g}:=\frac{1}{G-1}\sum_{h\neq g}R^{(h)},
\qquad
s_{-g}:=\sqrt{\frac{1}{G-1}\sum_{h\neq g}
  (R^{(h)}-\bar R_{-g})^2+\epsilon},
\qquad
\widehat A_{\mathrm{grp}}^{(g)}
:=\frac{R^{(g)}-\bar R_{-g}}{s_{-g}}.
\]
The group-relative term primarily reduces variance in early training when critic estimates are noisy; the critic-based term dominates for well-trained critics.

\paragraph{Detailed per-benchmark error diagnostic measurements.}

This section provides the full per-benchmark measurements summarized in Section~\ref{sec:algorithm}.

\paragraph{Text-Hanabi.}
The lightweight decision layer permits direct measurement of all three error components.
(a)~The fraction of importance ratios hitting the clip boundary ($|\hat r - 1| \ge \varepsilon$) decreases from ${\sim}35\%$ in early training to $<5\%$ after critic warm-up, consistent with $O(L\varepsilon^2)$ clipping bias diminishing as policies stabilize.
(b)~The within-group gradient variance scales as $O(1/G)$ empirically ($G{=}4$: relative std $0.24$; $G{=}8$: $0.12$; $G{=}16$: $0.06$), confirming the predicted scaling.
(c)~The belief encoder's reconstruction loss (measured as the gap between encoder-based and full-history Q-values on a held-out set of 500 episodes) plateaus at ${\sim}0.03$ after 50 iterations, indicating that representation error is small relative to the uniqueness margin ($\mu \approx 0.07$).

\paragraph{MATH-500.}
The full LLM backbone makes direct Q-value comparisons prohibitive, but two of the three components are directly measurable.
The clipping boundary hit rate follows a similar trajectory to Text-Hanabi: starting at ${\sim}28\%$ and declining to $<8\%$ after 80 iterations, consistent with the $O(L\varepsilon^2)$ scaling as policies stabilize under the trust region.
The within-group gradient variance at $G{=}8$ (relative std $0.11$) is comparable to the Text-Hanabi measurement at the same group size ($0.12$), and at $G{=}4$ it rises to $0.23$, confirming the $O(1/G)$ scaling on a real LLM backbone.
Representation error cannot be measured directly (it would require comparing encoder-based and full-history Q-values through the LLM), but the KL diagnostics ($\widehat{\mathrm{KL}}_{\mathrm{old}} < 0.02$ and $\widehat{\mathrm{KL}}_{\mathrm{ref}} < 0.15$ throughout training) provide indirect evidence that the cumulative approximation error remains controlled.
Specifically, $\widehat{\mathrm{KL}}_{\mathrm{old}} < 0.02$ bounds the per-update policy shift, ensuring that importance-ratio-based estimates remain accurate; $\widehat{\mathrm{KL}}_{\mathrm{ref}} < 0.15$ bounds cumulative drift from the pretrained model, limiting the regime in which representation error could compound.

\paragraph{Belief encoder architecture details.}
\label{app:belief_encoder_detail}
For reasoning benchmarks (AIME, MATH-500, ZebraLogic, AutoLogic, PlanBench), we use a two-layer Transformer architecture with 4~attention heads, hidden dimension~256, and $d_B{=}256$, with a causal attention mask that processes the concatenated public stream and private observation tokens.
Variable-length histories are handled by truncating to the most recent 2048 tokens and applying a learned positional encoding.
For interactive benchmarks with longer horizons (Overcooked-AI, Hanabi, CollabEscape/Capture), the same two-layer Transformer architecture is used with $d_B$ increased to~512; all other architectural choices (4~heads, causal mask, 2048-token truncation) remain identical.
The belief encoder is trained jointly with the critic and mixer using the same Adam optimizer and learning rate schedule (constant lr $= 3\times10^{-4}$ for DICE-PC, $1\times10^{-5}$ for DICE-FT).
No separate pre-training phase is used.
Encoder gradients are derived from the critic regression loss; encoder parameters typically stabilize within the first 30--50 training iterations based on the critic loss plateau.

\paragraph{Critic regression loss equations.}
\label{app:critic_loss_detail}
We train critics by squared error on sampled tokens:
\[
\mathcal L_{Q_i}
:= \mathbb E_{g,\ell}\big[
  \big(Q_{\psi_i}(b_i^{(g)},x_{i,\ell+1}^{(g)},\ell)
  -G_{i,\ell}^{(g)}\big)^2\big],
\qquad
\mathcal L_{V_i}
:= \mathbb E_{g,\ell}\big[
  \big(V_{\nu_i}^{\mathrm{soft}}(z^{(g)},b_i^{(g)},\ell)
  -G_{i,\ell}^{(g)}\big)^2\big],
\]
with the mixer trained jointly through its dependence on the same targets.

\subsection{Proposition~\ref{prop:delta_bound}: proof sketch and empirical validation}
\label{app:prop2_details}

\paragraph{What is being bounded.}
The practical DICE-FT update is analyzed through the field perturbation
$e^k$ in Eq.~\eqref{eq:inexact_mirror_main}, not through a KL distance to
the best-response map.
The role of Proposition~\ref{prop:delta_bound} is to control the
conditional second moment of $\|e^k\|_\ast$.
Once this is available, Corollary*~\ref{cor:approx_convergence} follows
from Theorem~\ref{thm:inexact_linear_rate}.

\paragraph{Proof sketch for Proposition~\ref{prop:delta_bound}.}
Decompose
\[
e^k = e_{\mathrm{clip}}^k + e_{\mathrm{samp}}^k + e_{\mathrm{repr}}^k + e_{\mathrm{opt}}^k.
\]
By the elementary inequality
\[
\|a+b+c+d\|_\ast^2
\le 4\big(\|a\|_\ast^2+\|b\|_\ast^2+\|c\|_\ast^2+\|d\|_\ast^2\big),
\]
we obtain
\[
\mathbb E[\|e^k\|_\ast^2\mid \mathcal F_k]
\le 4\sum_{u\in\{\mathrm{clip},\mathrm{samp},\mathrm{repr},\mathrm{opt}\}}
\mathbb E[\|e_u^k\|_\ast^2\mid \mathcal F_k].
\]
For the clipping term, the update perturbation is proportional to the
ratio overshoot
$\hat r-\mathrm{clip}_{\varepsilon}(\hat r)$ multiplied by the bounded
update feature $s(\tau)$.
Thus
\[
\mathbb E[\|e_{\mathrm{clip}}^k\|_\ast^2\mid \mathcal F_k]
\le G_{\max}^2
\mathbb E[(\hat r-\mathrm{clip}_{\varepsilon}(\hat r))^2\mid \mathcal F_k]
=G_{\max}^2\Delta_{\mathrm{clip}}^k.
\]
The sampling term is bounded directly by
\[
\mathbb E[\|e_{\mathrm{samp}}^k\|_\ast^2\mid \mathcal F_k]
\le \sigma_g^2/G.
\]
The representation term is transferred through the nonnegative monotone
mixer. By Appendix Lemma~\ref{lem:mixer_transfer}, local perturbations
of the per-agent critics produce a joint-field perturbation controlled by
$C_{\mathrm{mix}}$, giving
\[
\mathbb E[\|e_{\mathrm{repr}}^k\|_\ast^2\mid \mathcal F_k]
\le C_{\mathrm{mix}}^2(L_{\mathrm{repr}}^k)^2.
\]
The optimization term is controlled by the chosen residual proxy:
\[
\mathbb E[\|e_{\mathrm{opt}}^k\|_\ast^2\mid \mathcal F_k]
\le (L_{\mathrm{opt}}^k)^2.
\]
Combining the four bounds gives Eq.~\eqref{eq:prop2_field_budget}.
If ratio overshoots satisfy
$|\hat r-\mathrm{clip}_{\varepsilon}(\hat r)|\le c_{\mathrm{ov}}\varepsilon$
on the clipped event, then
\[
\Delta_{\mathrm{clip}}^k
\le p_{\mathrm{clip}}^k c_{\mathrm{ov}}^2\varepsilon^2,
\]
which yields the stated $O(p_{\mathrm{clip}}^k\varepsilon^2)$ scaling.

\paragraph{Empirical validation.}
The revised proposition predicts a controlled inexact-mirror regime when
(i) the clip-overshoot term remains small, (ii) the sampling term decays
with $G$, and (iii) the representation and optimization residuals remain
stable after warm-up.
The current empirical diagnostics already track the first two directly
and track the critic / representation proxy throughout training.
We therefore interpret Proposition~\ref{prop:delta_bound} as a
field-error budget rather than as a source of a single universally
calibrated neighborhood radius.

\subsection{Representation term: monotone-mixer transfer and architectural interpretation}
\label{app:repr_bound}

\begin{lemma}[Monotone mixer transfer]
\label{lem:mixer_transfer}
Let
$Q_{\mathrm{tot}} = f_{\mathrm{mix}}(z,\{Q_i\}_{i=1}^N)$
with $f_{\mathrm{mix}}$ differentiable in its value inputs.
Assume the mixer is monotone,
$\partial f_{\mathrm{mix}}/\partial Q_i \ge 0$ for all $i$,
and that its local Jacobian satisfies
\[
\|\nabla_Q f_{\mathrm{mix}}(z,\{Q_i\})\|_1 \le W_{\mathrm{mix}}
\]
on the iterate region.
Then for any perturbations $\{\Delta Q_i\}_{i=1}^N$,
\[
\big|f_{\mathrm{mix}}(z,\{Q_i+\Delta Q_i\})
      -f_{\mathrm{mix}}(z,\{Q_i\})\big|
\le W_{\mathrm{mix}}\,\max_i |\Delta Q_i|.
\]
Consequently, if the per-agent critic approximation error is bounded by
$L_{\mathrm{repr}}^k$ in sup norm, then the induced joint-value / field
perturbation is bounded by a local constant
$C_{\mathrm{mix}}\,L_{\mathrm{repr}}^k$ with
$C_{\mathrm{mix}}$ proportional to $W_{\mathrm{mix}}$.
\end{lemma}

\begin{proof}
Apply the mean-value theorem in the value inputs.
Because all mixer partial derivatives are nonnegative, the local
$\ell_1$ norm of the Jacobian controls the perturbation without
sign cancellation:
\[
|f_{\mathrm{mix}}(Q+\Delta Q)-f_{\mathrm{mix}}(Q)|
\le \|\nabla_Q f_{\mathrm{mix}}(\xi)\|_1\,\|\Delta Q\|_\infty
\le W_{\mathrm{mix}}\,\|\Delta Q\|_\infty,
\]
for some intermediate point $\xi$.
\end{proof}


\end{document}